\documentclass[nohyperref]{article}

\usepackage[utf8]{inputenc} 
\usepackage[T1]{fontenc}    
\usepackage[hyphens]{url}
\usepackage{hyperref}
\usepackage[hyphenbreaks]{breakurl}
\usepackage{url}            
\usepackage{booktabs}       
\usepackage[accepted]{icml2022}

\usepackage{amsmath}
\usepackage{dsfont}
\usepackage{relsize}
\usepackage{amsfonts}       
\usepackage{nicefrac}       
\usepackage{microtype}      
\usepackage{amsthm}
\usepackage{thmtools, thm-restate}
\usepackage{array,booktabs,tabularx, graphicx}
\usepackage{multicol}
\usepackage{multirow}
\newtheorem{definition}{Definition}

\newtheorem*{corollary}{Corollary}
\newtheorem{lemma}{Lemma}

\usepackage{xcolor,colortbl}
\usepackage{enumitem}

\usepackage{caption} 
\definecolor{Blue}{RGB}{128, 128, 255}
\definecolor{Gray}{RGB}{200, 200, 200}
\definecolor{Grayblue}{RGB}{180, 180, 222}
\definecolor{Whiteblue}{RGB}{230, 230, 255}
\definecolor{Blue}{RGB}{128, 128, 255}
\newcolumntype{G}{>{\columncolor{Gray}}X}
\newcolumntype{B}{>{\columncolor{Blue}}X}
\newcolumntype{P}{>{\columncolor{Grayblue}}X}
\newcolumntype{Q}{>{\columncolor{Whiteblue}}X}

\newif\ifpaper
\paperfalse

\newcommand{\R}{\mathbb{R}}

\newcommand{\Lce}{\mathcal{L}_{\text{CE}}}

\newcommand{\E}{\mathop{{}\mathbb{E}}}
\newcommand{\D}{\mathcal{D}}

\newcolumntype{C}{>{\centering\arraybackslash}X}

\def\R{\mathbb{R}}

\def\D{\mathcal{D}}

\newcommand{\pin}{p(x|i)}
\newcommand{\ptrout}{p(x|o)}

\definecolor{lightgrey}{rgb}{0.9, 0.9, 0.9}
\definecolor{update}{rgb}{0., 0., 0.}

\definecolor{columnbest}{rgb}{0., 0.4, 0.}
\newcommand{\best}[1]{\textcolor{columnbest}{#1}}

\newcommand{\subtab}{4mm}

\newcommand{\combiStwo}{Combi $s_2$}
\newcommand{\combiSthree}{Combi $s_3$}

\newcommand{\backgroundSone}{BGC $s_1$}
\newcommand{\backgroundStwo}{BGC $s_2$}
\newcommand{\backgroundSthree}{BGC $s_3$}
\newcommand{\Kpo}{K\!+\!1}

\icmltitlerunning{Breaking Down Out-of-Distribution Detection}

\begin{document}

\setlength{\abovedisplayskip}{2.2mm}
\setlength{\belowdisplayskip}{2.2mm}
\setlength{\textfloatsep}{3mm}

\twocolumn[
\icmltitle{Breaking Down Out-of-Distribution Detection: Many Methods Based on OOD Training Data Estimate a Combination of the Same Core Quantities
}

\icmlsetsymbol{equal}{*}

\author{%
  Julian Bitterwolf\\
  Department of Science\\
  University of Tübingen \\
  \texttt{julian.bitterwolf@uni-tuebingen.de \ } \\
   \And
   Alexander Meinke \\
  Department of Science\\
  University of Tübingen \\
   \texttt{alexander.meinke@uni-tuebingen.de} \\
   \AND
   Maximilian Augustin \\
  Department of Science\\
  University of Tübingen \\
   \texttt{maximilian.augustin@uni-tuebingen.de} \\
   \And
   Matthias Hein \\
  Department of Science\\
  University of Tübingen \\
   \texttt{matthias.hein@uni-tuebingen.de} \\
}

\begin{icmlauthorlist}
\icmlauthor{Julian Bitterwolf}{yyy}
\icmlauthor{Alexander Meinke}{yyy}
\icmlauthor{Maximilian Augustin}{yyy}
\icmlauthor{Matthias Hein}{yyy}
\end{icmlauthorlist}

\icmlaffiliation{yyy}{University of Tübingen}
\icmlcorrespondingauthor{}{julian.bitterwolf@uni-tuebingen.de}
\icmlkeywords{Machine Learning, ICML}

\vskip 0.3in
]

\printAffiliationsAndNotice

\begin{abstract}
It is an important problem in trustworthy machine learning to recognize out-of-distribution (OOD) inputs which are inputs unrelated to the in-distribution task. Many out-of-distribution detection methods have been suggested in recent years. The goal of this paper is to recognize common objectives as well as to identify the implicit scoring functions of different OOD detection methods. We focus on the sub-class of methods that use surrogate OOD data during training in order to learn an OOD detection score that generalizes to new unseen out-distributions at test time. We show that binary discrimination between in- and (different) out-distributions is equivalent to several distinct formulations of the OOD detection problem. When trained in a shared fashion with a standard classifier, this binary discriminator reaches an OOD detection performance similar to that of Outlier Exposure. Moreover, we show that the confidence loss which is used by Outlier Exposure has an implicit scoring function which differs in a non-trivial fashion from the theoretically optimal scoring function in the case where training and test out-distribution are the same, which again is similar to the one used when training an Energy-Based OOD detector or when adding a background class. In practice, when trained in exactly the same way, all these methods perform similarly.
\end{abstract}

\section{Introduction}

While deep learning has significantly improved performance in many application domains, there are serious concerns for using deep neural networks in applications which are of safety-critical nature. With one major problem being adversarial samples \citep{SzeEtAl2014,MadEtAl2018}, which are small imperceptible modifications of the image that change the decision of the classifier, another major problem are overconfident predictions \citep{NguYosClu2015,HenGim2017,HeiAndBit2019}
for images not belonging to the classes of the actual task. Here, one distinguishes between far out-of-distribution data, e.g. different forms of noise or completely unrelated tasks like CIFAR-10 vs. SVHN, and close out-of-distribution data which can for example occur in related image classification tasks where the semantic structure is very similar e.g. CIFAR-10 vs. CIFAR-100. Both are important to be distinguished from the in-distribution, but it is conceivable that close out-of-distribution data is the more difficult problem with potentially fatal consequences: in an automated diagnosis system we want that the system recognizes that it ``does not know'' when a new unseen disease comes in rather than assigning high confidence into a known class leading to fatal treatment decisions. Thus out-of-distribution awareness is a key property of trustworthy AI systems.

In this paper, we focus on the setting of OOD detection  where during training time, there is no information available on the distribution of OOD inputs that might appear when the model is used for inference.
It is often reasonable to assume access to a surrogate out-distribution during training.
One can however not assume that these surrogate is related to the OOD inputs that will be encountered at test-time.
A large number of different approaches to OOD detection based on combinations of density estimation, classifier confidence, logit space energy, feature space geometry, behaviour on auxiliary tasks, and other principles has been proposed to tackle this problem. We give a detailed overview of existing OOD detection methods 
in Appendix~\ref{sec:related_work}. However, most OOD detection papers are focused on establishing superior empirical detection performance and provide little theoretical background on differences but also similarities to existing methods. In this paper we want to take a different path as we believe that a solid theoretical basis is needed to make further progress in this field. Our goal is to identify, at least for a particular subclass of techniques, whether the differences are indeed due to a different underlying theoretical principle or whether they are due to the efficiency of different \emph{estimation techniques} for the same underlying detection criterion, called ``scoring function''. In some cases, we will see that one can even disentangle the estimation procedure from the scoring function, so that one can simulate several different scoring functions from one model's estimated quantities.\\
A simple approach to OOD detection is to treat it as a binary discrimination problem between in- and out-of-distribution, or more generally predicting a score how likely the input is OOD.
In this paper, we show that from the perspective of Bayesian decision theory, several established methods are indeed equivalent to this binary discriminator. Differences arise mainly from i) the choice of the training out-distribution, e.g. the popular Outlier Exposure of \citet{HenMazDie2019} has advocated the use of a rich and large set of natural images
as a proxy for the distribution of natural images, and ii) differences in the estimation procedure.
Concretely, the main contributions of this paper are:
\begin{itemize}[leftmargin=3.1mm, topsep=0pt]
    \item We show that several OOD detection approaches
    which optimize an objective that includes predictions on surrogate OOD data
    are equivalent to the binary discriminator between in- and out-distribution when analyzing the rankings induced by the Bayes optimal classifier/density.
    \item We derive the implicit scoring functions for the confidence loss \citep{LeeEtAl2018} used by Outlier Exposure \citep{HenMazDie2019}, for Energy-Based OOD Detection~\citep{liu2020energy}, and for an extra background class for the out-distribution \citep{thulasidasan2021effective}. The confidence scoring function turns out not to be equivalent to the ``optimal'' scoring function of the binary discriminator when training and test out-distributions are the same. 
    \item We show that the combination of a binary discriminator between in- and out-distribution with a standard classifier on the in-distribution, when trained in a shared fashion, yields OOD detection performance competitive with state-of-the-art methods based on surrogate OOD data. 
    \item We show that density estimation is equivalent to discrimination between the in-distribution and uniform noise which indicates why standard density estimates are not suitable for OOD detection, as has frequently been observed.
\end{itemize}
Even though we identify that a simple baseline is competitive with the state-of-the-art, the main aim of this paper is a better understanding of the key components of different OOD detection methods and to identify the key properties which lead to SOTA OOD detection performance. All of our findings are supported by extensive experiments on CIFAR-10 and CIFAR-100 with evaluation on various  challenging out-of-distribution test datasets.

\section{Models for OOD Data and Equivalence of OOD Detection Scores}\label{section:scores}
We first characterize the set of transformations of a scoring function which leaves the OOD detection criteria like AUC or FPR invariant. This is important for the analysis later on, since the scoring functions of different methods are in many cases not identical as functions but yield the same OOD detection performance by those criteria.
Like most work in the literature we consider OOD detection on a compact input domain $X$ with the most important example being image classification where $X = [0,1]^D$.
The most popular approach to OOD detection is the construction of an in-distribution-scoring function $f:X \rightarrow \mathbb{R} \cup \{\pm \infty \}$ such that $f(x)$ tends to be smaller if $x$ is drawn from an out-distribution $p(x|\D\!=\!o)$, short $p(x|o)$, than if it is drawn from the in-distribution $p(x|\D\!=\!i)$, short $p(x|i)$. There is a variety of different performance metrics for this task, with a very common one being the \textit{area under the receiver-operator characteristic curve} (AUC). The AUC for a scoring function $f$ distinguishing between an in-distribution $p(x|i)$ and an out-distribution $p(x|o)$ is given by
\begin{align*}
    \mathrm{AUC}_f& \big(p(x|i), p(x|o)\big) \\
    &= \underset{ \substack{
x \sim p(x|i) \\
y \sim p(z|o) 
} } { \mathlarger{\mathbb{E}}}\left[ \mathds{1}_{f(x)>f(z)} + \frac{1}{2} \mathds{1}_{f(x)=f(z)}\right].
\end{align*}
We define an equivalence of scoring functions based on their AUCs and will show that this equivalence implies equality of other employed performance metrics as well.
\begin{definition}\label{Def:AUCequivalence}
Two scoring functions $f$ and $g$ are equivalent and we write $f\cong g$ if 
\begin{align*}
\mathrm{AUC}_f \big(p(x|i), p(x|o)\big) = \mathrm{AUC}_g \big(p(x|i), p(x|o)\big)    
\end{align*}
for all potential distributions $p(x|i)$ and $p(x|o)$.
\end{definition}
As the AUC is not dependent on the actual values of $f$ but just on the ranking induced by $f$ one obtains the following characterization of the equivalence of two scoring functions.
\begin{restatable}{theorem}{scoreequiv}
\label{thm:score_equiv}
Two scoring functions $f,g$ are equivalent $f\cong g$ if and only if there exists a strictly
monotonously increasing $\phi :\mathrm{range}(g) \rightarrow\mathrm{range}(f)$ such that $f=\phi(g)$. 
\end{restatable}
\begin{corollary}
The equivalence between scoring functions in Def.~\ref{Def:AUCequivalence} is an equivalence relation.
\end{corollary}
Another metric is the \textit{false positive rate at a fixed true positive rate q}, denoted as FPR@qTPR. A commonly used value for the TPR is 95\%.  The smaller the FPR@qTPR, the better the OOD discrimination performance.
\begin{restatable}{lemma}{equivfpr}
    Two equivalent scoring functions $f\cong g$ have the same FPR@qTPR for any pair of in- and out-distributions $p(x|i), p(x|o)$ and for any chosen TPR q.
\end{restatable}
In the next section, we use the previous results to show that the Bayes optimal scoring functions of several proposed methods for out-of-distribution detection are equivalent to those of simple binary discriminators.
\section{Bayes-optimal Behaviour of Binary Discriminators and Common OOD Detection Methods}
In the following we will show that the Bayes optimal function of several existing approaches to OOD detection for unlabeled data are equivalent to a binary discriminator between in- and a (training) out-distribution, whereas
different solutions arise for methods that involve labeled data.
As the equivalences are based on the Bayes optimal solution, these are asymptotic statements and thus it has 
to be noted that convergence to the Bayes optimal solution can be infinitely slow and that the methods can have implicit inductive biases. This is why we additionally support our findings with extensive experiments.
\subsection{OOD detection with methods using unlabeled data}\label{sec:bayes_unlabeled}
We first provide a formal definition of OOD detection before we show the equivalence of density estimators resp. likelihood to a binary discriminator. \\
\textbf{The OOD problem}
In order to make rigorous statements about the OOD detection problem we first have to provide the mathematical basis for doing so. We assume that we are given an in-distribution $p(x|i)$ and potentially also a \emph{training} out-distribution $p(x|o)$. At this particular point no labeled data is involved, so both of them are just distributions over $X$. For simplicity we assume in the following that they both have a density wrt. the Lebesgue measure on $X=[0,1]^d$. We assume that in practice we get samples from the mixture distribution
\begin{equation*}
    p(x)
    =p(x|i)p(i)+p(x|o)(1-p(i))
\end{equation*} 
where $p(i)$ is the probability that we expect to see in-distribution samples in total. In order to make the decision between in-and out-distribution for a given point $x$, it is thereby optimal to estimate
\begin{equation*} 
p(i|x)=\frac{p(x|i)p(i)}{p(x)}=\frac{p(x|i)p(i)}{p(x|i)p(i)+p(x|o)p(o)},
\end{equation*} 
which is defined for all $x \in [0,1]^d$ with $p(x)>0$ (assuming $p(x|i)$ and $p(x|o)$ can be written as densities).
If the training out-distribution is also the test out-distribution then this is already optimal but we would like that the approach generalizes to other unseen test out-distributions and thus an important choice is the training out-distribution $p(x|o)$. Note that as $p(i|x)$ is only well-defined for all $x$ with $p(x)>0$, it is thus reasonable to choose for $p(x|o)$ a distribution with support in $[0,1]^d$.
In this case we ensure that the criterion with which we perform OOD detection is defined for any possible input $x$. This is desirable, as OOD detection should work for any possible input $x \in X$. \\
\textbf{Optimal prediction of a binary discriminator between in- and out-distribution}\label{section:binary}
 We consider a binary discriminator with model parameters $\theta$ between in- and (training) out-distribution, where  $\hat{p}_\theta(i|x)$ is the predicted probability for the in-distribution. Under the assumption that $p(i)$ is the probability for in-distribution samples and using cross-entropy (which in this case is the logistic loss up to a constant global factor of $\log(2)$) the expected loss becomes: 
 
\begin{align*}\label{eq:binary-exp}
\begin{split}
    \min_\theta \ \ \ \ \ &p(i) \E_{x\sim p(x|i)}\left[ -\log \hat{p}_\theta(i|x)\right]\\
    + &p(o) \E_{x\sim p(x|o)}\left[ -\log (1 - \hat{p}_\theta(i|x))\right] \ .
\end{split}
\end{align*}
One can derive that the Bayes optimal classifier minimizing the expected loss has the predictive distribution:
\begin{equation*}
\hat{p}_{\theta^*}(i|x)=\frac{p(x|i)p(i)}{p(x|i)p(i)+p(x|o)p(o)}=p(i|x).
\end{equation*} 
Thus at least for the training out-distribution, a binary classifier based on samples from in- and (training) out-distribution would suffice to solve the OOD detection problem perfectly.
\\
\textbf{Equivalence of density estimation and binary discrimination for OOD detection}\label{section:generative}
In this section we further analyze the relationship of common OOD detection approaches with the binary discriminator between in-and out-distribution. We start with density estimators sourced from generative models. A basic approach that is known to yield relatively weak OOD performance~\citep{NalEtAl2018, ren2019likelihood, XiaoLikelihoodRegret} is directly utilizing a model's estimate for the density $p(x|i)$ at a sample input $x$.
An improved density based approach which uses perturbed in-distribution samples as a surrogate training out-distribution is the Likelihood Ratios method \citep{ren2019likelihood}, which proposes to fit a generative model for both the in- and out-distribution and to use the ratio between the likelihoods output by the two models as a discriminative feature.

We show that with respect to the scoring function, the correct density $p(x|i)$ is equivalent to the Bayes optimal prediction of a binary discriminator between the in-distribution and uniform noise. Furthermore, the density ratio $\frac{p(x|i)}{p(x|o)}$ is equivalent to the prediction of a binary discriminator between the two distributions on which the respective models used for density estimation have been trained.
Because of this equivalence, we argue that the use of binary discriminators is a simple alternative to these methods because of its easier training procedure. While this equivalence is  an asymptotic statement, the experimental comparisons in the appendix show that the methods perform similarly poorly compared to the methods using labeled data.

We first prove the more general case of arbitrary likelihood ratios. In the following we use the abbreviation $\lambda=\frac{p(o)}{p(i)}$ to save space and make the statements more concise.
\begin{restatable}{lemma}{lhratios}
\label{lem:lh_ratios}
Assume $\pin$ and $\ptrout$ can be represented by densities and the support of $\ptrout$  covers the whole input domain $X$. Then $\frac{p(x|i)}{p(x|o)} \cong \frac{p(x|i)}{p(x|i)+\lambda p(x|o)}$ for any $\lambda>0$.
\end{restatable}
This means that the likelihood ratio score of two optimal density estimators is equivalent to the in-distribution probability $\hat{p}_{\theta^*}(i|x)$ predicted by a binary discriminator and this is true for any possible ratio of $p(i)$ to $p(o)$. In the experiments below, we show that using such a discriminator has similar performance as the likelihood ratios of the different trained generative models.

For the approaches that try to directly use the likelihood of a generative model as a discriminative feature, this means that their objective is equivalent to training a binary discriminator against uniform noise, whose density is $p_\mathrm{Uniform}(x) = p(x|o)= 1$ at any $x$.

\begin{restatable}{lemma}{lhuni}
Assume that $\pin$ can be represented by a density. Then $p(x|i) \cong \frac{p(x|i)}{p(x|i)+\lambda}$ for any $\lambda>0$.
\end{restatable}
This provides additional evidence why a purely density based approach for many applications proves to be insufficient as an OOD detection score on the complex image domain: it is not reasonable to assume that a binary discriminator between certain classes of natural images on the one hand and uniform noise on the other hand provides much useful information about images from other classes or even about other nonsensical inputs.

\subsection{OOD detection for methods using labeled data}\label{section:OE}
We first discuss how one can formulate the OOD problem when one has access to labeled data for the in-distribution and we identify the target distribution of OOD detection using a background/reject class. Then we derive the Bayes optimal classifier of the confidence loss~\citep{LeeEtAl2018} as used by the most successful variant of Outlier Exposure~\citep{HenMazDie2019} and discuss the implicit scoring function.
In most cases the scoring functions turn out not to be non-equivalent to $p(i|x)$ (which is optimal if training and test out-distribution agree) as they integrate additional information from the classification task.
Given a joint in-distribution $p(y,x|i)$ (where $y \in \{1,\ldots,K\}$ given that we have $K$ labels) for the labeled in-distribution, there are different ways how to come up with a joint distribution for in- and out-distribution. Interestingly, the different encodings used e.g. in training with a background class \citep{thulasidasan2021effective} vs. training a classifier with confidence loss \citep{LeeEtAl2018} together with variants of the employed scoring function lead to methods which unexpectedly can have quite different behavior.

\textbf{Background class:} In this case we just put all out-of-distribution samples into a $K+1$-class which is typically called background/reject class \citep{thulasidasan2021effective}. The joint distribution then becomes
\begin{align*}
p(y,x) = \begin{cases} p(y,x|i)p(i) & \textrm{ if } y \in \{1,\ldots,K\},\\
                       p(x|o)p(o) & \textrm{ if } y=\Kpo.\end{cases}
\end{align*}
We denote by $p(x|i)=\sum_{y=1}^K p(y,x|i)$ the marginal in-distribution and note that the marginal distribution of the joint distribution of in- and out-distribution is again 
\[ p(x)  = p(x|i)p(i)+p(x|o)p(o).\]
Thus we get the conditional distribution
\[ p(y|x) = \begin{cases} p(y|x,i)p(i|x) & \textrm{ if } y \in \{1,\ldots,K\},\\ p(o|x)=1-p(i|x) & \textrm{ if } y=\Kpo.\end{cases}\]
The Bayes optimal solution of training with a background class using any calibrated loss function $L(y,f(x))$, e.g. the cross-entropy loss \citep{LapEtAl2016}, then yields a Bayes optimal classifier $f^*$ which has a predictive distribution $p_{f^*}(y|x)=p(y|x)$. There are two potential scoring function that come to mind:
\begin{align*}
    s_1(x)=1-p_{f^*}(\Kpo|x)\text{ and } 
    s_2(x)=\max_{k=1,\ldots,K} p_{f^*}(k|x)
\end{align*}
The first one, used in \citet{ chen2020informative-outlier-matters, thulasidasan2021effective}, is motivated by the fact that $p_{f^*}(\Kpo|x)$ is directly the predicted probability that the point is from the out-distribution as indeed it holds: $s_1(x)=p(i|x)$ which is the optimal scoring function if training and test out-distribution are equal. On the other hand the maximal predicted probability  $\max_{k=1,\ldots,K} p_{f^*}(k|x)$, which is often employed as a scoring function~\cite{HenGim2017}, becomes for the Bayes optimal classifier
\[ s_2(x)= p(i|x) \max_{k=1,\ldots,K} p(k|x,i),\]
which is a product of $p(i|x)$ and the maximal conditional probability of some class of the in-distribution; note that $s_2$ is well defined as $p(i|x)$ is defined if $p(x|o)$ has support everywhere in $X$ and if $p(i|x)>0$ then also $p(x|i)>0$. Thus the scoring function $s_2(x)$ integrates additionally to $p(i|x)$ also class-specific information and is thus less dependent on the chosen training out-distribution. In fact, one can see that $s_2$ only ranks points high if both the binary discriminator \emph{and} the classifier rank the corresponding point high. However, in the case where training and test out-distribution are identical, this scoring function is not equivalent to $p(i|x)$ and thus introduces  a bias in the estimation. 

\textbf{Outlier Exposure \cite{HenMazDie2019} with confidence loss~\citep{LeeEtAl2018}:} We analyze the Bayes optimal solution for the confidence loss~\citep{LeeEtAl2018} that is used by Outlier Exposure (OE) and show that the associated scoring function can be written,
similarly to the scoring function $s_2(x)$ for training with a background class, as a function of $p(i|x)$ and $p(y|x,i)$.

The training objective with the confidence loss is in expectation given by
\begin{align*}
    \min_\theta \E_{\substack{(x,y)\sim \\p(x,y|i)}}\! \left[\Lce(f_{\theta}(x), y)\right] + \lambda \!\! \E_{x\sim p(x|o)}\left[ \Lce(f_{\theta}(x), u^K)\right] \ ,
\end{align*}
where $\theta$ are the model parameters and $f_{\theta}(x) \in \R^K$ is the model output as logits, and $u^K = (\frac{1}{K}, \ldots, \frac{1}{K})^T$ is the uniform distribution over the $K$ classes of the in-distribution classification task.

In the following theorem we derive the Bayes optimal predictive distribution for this training objective.
\begin{restatable}{theorem}{OEoptim}\label{thm:OE_optim}
The predictive distribution $p_{f^*}(y|x)$ of the Bayes optimal classifier $f^*$ minimizing the expected confidence loss 
is given for $y \in \{1,\ldots,K\}$ as
\begin{align*}
    p_{f^*}(y|x) = p(i|x) p(y|x,i) + \frac{1}{K} \big(1-p(i|x)\big) \ .
\end{align*}
\end{restatable}
Thus the effective scoring function of using the probability of the predicted class as suggested in \citet{HenGim2017, LeeEtAl2018, HenMazDie2019} is
\begin{align*}
\begin{split}
s_3(x)
&= p(i|x) \max_{y=1,\ldots,K} p(y|x,i) + \frac{1}{K} \big(1-p(i|x)\big) \\
&= p(i|x) \Big[ \max_{y=1,\ldots,K} p(y|x,i)-\frac{1}{K}\Big]  +\frac{1}{K}.    
\end{split}
\end{align*}
Please note that the term inside the brackets is positive as $\max_{k=1,\ldots,K} p(k|x,i)\geq \frac{1}{K}$. Interestingly, the scoring functions $s_2$ and $s_3$ are not equivalent even though they look quite similar. In particular, due to the subtraction of $\frac{1}{K}$ the scoring function $s_3$ puts more emphasis on the classifier than $s_2$.
In Appendix~\ref{sec:energy} we additionally analyze Energy-Based OOD Detection~\citep{liu2020energy}, and show that the Bayes optimal decision is equivalent to using the scoring function $s_1=p(i|x)$.

\textbf{Energy-Based OOD Detection~\citep{liu2020energy}:} This method uses the energy of a model's logit prediction which is defined as $E_\theta(x) = - \log \sum_{y=1}^K e^{f_\theta(y|x)}$ as an OOD detection score. The model's energy is fine-tuned to be low on the in-distribution and high on surrogate OOD data, within certain margins $m_{\text{out}}, m_{\text{in}}$.
We analyze the method in detail in Appendix \ref{sec:energy}, where we prove the following:
\begin{restatable}{theorem}{EnergyOptim}\label{thm:E_optim}
The Bayes optimal logit output $f^*_\theta(x)$ of the Energy-Based OOD detection model minimizing the expected loss on an input $x$ yields class probabilities $p^*_\theta(k|x) = p(k|x,i)$ that are optimal for a standard classifier with cross-entropy loss and simultaneously fulfills
\begin{align*}
    -E^*_\theta(x) =  p(i|x) \cdot (m_{\text{out}} - m_{\text{in}}) - m_{\text{out}} \ . 
\end{align*}
\end{restatable}
\begin{restatable}{corollary}{EnergyOptimCor}\label{thm:E_optim_cor}
The Bayes optimal solution of the Energy-based OOD detection criterion is equivalent to the Bayes optimal solution of a binary discriminator between the training in- and out-distributions.
\end{restatable}

\subsection{Separate vs shared estimation of $p(i|x)$ 
and $p(y|x,i)$}
So far we have derived that at least from the point of view of the ranking induced by the Bayes optimal solution, OOD detection based on  generative methods, likelihood ratios, logit energy, and the background class formulation with the scoring function $s_1$ is equivalent to a binary classification problem between in- and out-distribution in order to estimate $p(i|x)$. The differences arise mainly in the choice of the training out-distribution $p(x|o)$: i) uniform for generative resp. density based methods, ii) a quite specific out-distribution for likelihood ratios \citep{ren2019likelihood} and iii) a proxy of the distribution of all natural images \citep{HenMazDie2019,thulasidasan2021effective}. 
On the other hand when labeled data is involved we can additionally train a classifier on the in-distribution in order to estimate $p(y|x,i)$. We will then combine the estimates of $p(i|x)$ and $p(y|x,i)$ according to the three scoring functions derived in the previous section and check if the novel OOD detection methods constructed in this way perform similar to the OOD methods from which we derived the corresponding scoring function i) OOD detection with a background class \citep{thulasidasan2021effective} or ii) using Outlier Exposure \cite{HenMazDie2019}. This will allow us to differentiate between differences of the employed scoring functions for OOD detection and the estimators for the involved quantities. In this way we foster a more systematic approach to OOD detection.

In the unlabeled case we train simply the binary classifier $p_\theta:[0,1]^d \rightarrow \R$ using logistic/cross entropy loss in a class balanced fashion
\begin{align*}
  \min_{\theta}   \Big( &-\frac{1}{N} \sum_{i = 1}^{N}  \log\left(\hat{p}_\theta(i|x_i^{\text{IN}})\right) \\
  &- \frac{\lambda}{M}\sum_{j=1}^{M} \log \left(1\!-\!\hat{p}_\theta(i|x_{j}^{\text{OUT}}) \right) \Big)
    \ ,
\end{align*}
where $(x_i^{\text{IN}})_{i=1}^N$ and $(x_{j}^{\text{OUT}})_{j=1}^M$ are samples from the in-distribution and the out-distribution. 

In the case where we have labeled data we can additionally solve the classification problem.
The obvious approach is to train the binary classifier for estimating $p(i|x)$ and the classifier to estimate $p(y|x,i)$
completely independently. Not surprisingly, we show in Section \ref{section:experiments} that this approach works less well. In fact both tasks benefit from each other. Moreover, in training a neural network using a background class or with Outlier Exposure \citep{HenMazDie2019} we are implicitly using a shared representation for both tasks which improves the results.

Thus we propose to train the binary discriminator of in-versus out-distribution together with the classifier on the in-distribution jointly.
Concretely, we use a neural network with $K+1$ outputs where the first $K$ outputs represent the classifier and the last output is the logit of the binary discriminator. The resulting shared problem can then be written as 
\begin{align*}
    \min_{\theta} \Big( &-\frac{1}{N_b} \sum_{r = 1}^{N_b}  \log \hat{p}_\theta(i|x_r^{\text{IN}}) \\
    &- \frac{\lambda}{M}\sum_{s=1}^{M} \log\left(1\!-\!\hat{p}_\theta(i|x_{s}^{\text{OUT}}) \right) \\
   & - \frac{1}{N_c} \sum_{t = 1}^{N_c} \log\hat{p}_\theta(y_t^{\text{IN}}|x_t^{\text{IN}}) \Big)
\end{align*}
where $\lambda=\frac{p(o)}{p(i)}$ which is typically set to $1$ during training in order to get a class-balanced problem. Note that the in-distribution samples
$(x_r^{\text{IN}})_{r=1}^{N_b}$ used to estimate $p(i|x)$ can be a super-set of the labeled examples $(x_t^{\text{IN}},y_t^{\text{IN}})_{t=1}^{N_c}$ used to train the classifier so that one can potentially integrate unlabeled data - this is an advantage compared to OOD detection with a background class or Outlier Exposure where this is not directly possible. An example of such a situation is given in Appendix~\ref{section:partially_labelled}, where we observe that shared training of a classifier and a binary discriminator works better than OE when only 10\% of the in-distribution training samples have labels. We stress that the loss functions of the classifier and the discriminator act on independent outputs; the functions modelling the two tasks only interact with each other due to the shared network weights up to the final layer. Nevertheless, we see in the next Section \ref{section:experiments} that training with a shared representation boosts both the classifier and the binary discriminator.

\begin{table*}[!h]
\caption{Accuracy on the in-distribution (CIFAR-10/CIFAR-100) and FPR@95\%TPR for various test out-distributions of different OOD methods with OpenImages  as training out-distribution (results for the test set of OpenImages are not used in the mean FPR). Lower false positive rate is better. 
All methods except Mahalanobis have been trained using the same architecture, training parameters, schedule and augmentation.
$s_1,s_2,s_3$ are the scoring functions introduced in Section \ref{section:OE}, implicit scoring functions in parentheses. Our binary discriminator (\textsc{BinDisc}) resp. the combination with the shared classifier (\textsc{Shared Combi}) and the models with background class (\textsc{BGC}) with scoring functions $s_2$ or $s_3$ outperform the Mahalanobis detector~\citep{LeeEtAl2018b} and are similar to Outlier Exposure~\citep{HenMazDie2019}. CelebA makes no sense as test out-distribution for CIFAR-100 as man/woman are classes in CIFAR-100.
}\label{table:OI_AA_FPR}
\setlength\tabcolsep{.5pt} 
\vskip 0.15in
\begin{center}
\begin{small}
\begin{sc}
\makebox[\textwidth][c]{
\begin{tabularx}{1.0\textwidth}{lC|C|CCCCCCC|C}
\multicolumn{10}{c}{in-distribution: CIFAR-10} \\
\midrule
       & &  Mean &     SVHN    &   LSUN    &   Uni &   Smooth     &   C-100 &  80M &   CelA  &    OpenIm     \\ 
Model    & Acc. &  FPR &    FPR    &   FPR    &   FPR &  FPR    &   FPR & FPR &   FPR &   FPR      \\
\midrule
Plain Classi
 &  95.16
 &  53.01    
 &  47.87  
 &  50.00  
 &  17.51  
 &  65.81  
 &  60.43  
 &  53.44  
 &  76.00  
 &  63.71   \\
\midrule
Mahalanobis
 &
 & 36.68
 & 20.97
 & 49.00
 & \best{\phantom{0}0.00}
 & \best{\phantom{0}0.00}
 & 57.21
 & 48.85
 & 80.71
 & 55.13
\\
\midrule
Energy ($s_1$)
 & 94.13
 & 18.59
 & 10.39
 & \best{\phantom{0}0.00}
 & \best{\phantom{0}0.00}
 & \best{\phantom{0}0.00}
 & 65.60
 & 53.98
 & 0.15
 & 1.44
\\
\midrule
OE  ($s_3$)
 &  95.06
 &  \textbf{15.20}    
 &  \phantom{0}9.58  
 &  \best{\phantom{0}0.00}  
 &  \best{\phantom{0}0.00}  
 &  \best{\phantom{0}0.00}  
 &  \best{54.05}  
 &  \best{42.33}  
 &  \best{\phantom{0}0.45}  
 &  \phantom{0}3.46   \\
\midrule
  \hspace{\subtab} \backgroundSone
 &  
 &  18.83    
 &  \best{\phantom{0}2.36}  
 &  \best{\phantom{0}0.00}  
 &  \best{\phantom{0}0.00}  
 &  \best{\phantom{0}0.00}  
 &  72.00  
 &  56.41  
 &  \phantom{0}1.04  
 &  \phantom{0}0.05   \\
  \hspace{\subtab} \backgroundStwo
 &  95.21
 &  16.52    
 &  \phantom{0}7.51  
 &  \best{\phantom{0}0.00}  
 &  \phantom{0}0.05  
 &  \phantom{0}2.10  
 &  55.16  
 &  44.57  
 &  \phantom{0}6.26  
 &  \phantom{0}1.65   \\
\backgroundSthree
 &  95.21
 &  16.63    
 &  \phantom{0}7.69  
 &  \best{\phantom{0}0.00}  
 &  \phantom{0}0.07  
 &  \phantom{0}2.36  
 &  55.19  
 &  44.67  
 &  \phantom{0}6.41  
 &  \phantom{0}1.74   \\
\midrule
\hspace{\subtab} Shared BinDisc  ($s_1$)
 &  
 &  19.56    
 &  \phantom{0}4.65  
 &  \best{\phantom{0}0.00}  
 &  \best{\phantom{0}0.00}  
 &  \best{\phantom{0}0.00}  
 &  77.50  
 &  53.93  
 &  \phantom{0}0.87  
 &  \best{\phantom{0}0.04}   \\
\hspace{\subtab} Shared Classi 
 &  \textbf{95.28}
 &  29.34    
 &  28.00  
 &  \phantom{0}7.00  
 &  \phantom{0}2.33  
 &  33.04  
 &  58.61  
 &  47.90  
 &  28.54  
 &  35.94   \\
Shared Combi $s_2$
 &  \textbf{95.28}
 &  16.00    
 &  \phantom{0}8.56  
 &  \best{\phantom{0}0.00}  
 &  \best{\phantom{0}0.00}  
 &  \best{\phantom{0}0.00}  
 &  58.80  
 &  42.79  
 &  \phantom{0}1.83  
 &  \phantom{0}0.61   \\
Shared Combi $s_3$
 &  \textbf{95.28}
 &  16.06    
 &  \phantom{0}9.00  
 &  \best{\phantom{0}0.00}  
 &  \best{\phantom{0}0.00}  
 &  \best{\phantom{0}0.00}  
 &  58.68  
 &  42.85  
 &  \phantom{0}1.91  
 &  \phantom{0}0.66   \\ \toprule

\multicolumn{10}{c}{in-distribution: CIFAR-100} \\
\midrule
       & &  Mean &     SVHN    &   LSUN    &   Uni &   Smooth     &   C-10 & 80M  &    &   OpenIm      \\ 
Model    & Acc. &  FPR &    FPR    &   FPR    &   FPR &  FPR    &   FPR & FPR &    &   FPR      \\ \midrule
Plain Classi
 &  77.16
 &  67.57    
 &  75.50  
 &  78.33  
 &  22.60  
 &  70.98  
 &  \best{80.55}  
 &  77.43  
 &  \phantom{0000}  
 &  80.80   \\
\midrule
Mahalanobis
 &
 & 53.88
 & 54.36
 & 66.00
 & 46.43
 & 0.06
 & 85.39
 & \best{71.01}
 & 
 & 74.69
 \\ \midrule
Energy ($s_1$)
 & 73.47
 & 35.91
 & 39.48
 & \best{\phantom{0}0.00}  
 & \best{\phantom{0}0.00} 
 & 0.46
 & 92.76
 & 82.80
 & 
 & 1.41
 \\ \midrule
OE  ($s_3$)
 &  77.19
 &  35.03    
 &  47.36  
 &  \best{\phantom{0}0.00}  
 &  \phantom{0}0.67  
 &  \phantom{0}0.08  
 &  84.64  
 &  77.42  
 &  \phantom{0000}  
 &  \phantom{0}1.28   \\
\midrule
  \hspace{\subtab} \backgroundSone
 &  
 &  \textbf{31.14}    
 &  11.58  
 &  \best{\phantom{0}0.00}  
 &  \best{\phantom{0}0.00}  
 &  \best{\phantom{0}0.00}  
 &  93.94  
 &  81.29  
 &  \phantom{0000}  
 &  \best{\phantom{0}0.07}   \\
  \hspace{\subtab} \backgroundStwo
 &  \textbf{77.61}
 &  33.32    
 &  37.06  
 &  \best{\phantom{0}0.00}  
 &  \best{\phantom{0}0.00}  
 &  \phantom{0}0.20  
 &  84.50  
 &  78.17  
 &  \phantom{0000}  
 &  \phantom{0}1.26   \\
\backgroundSthree
 &  \textbf{77.61}
 &  33.36    
 &  37.27  
 &  \best{\phantom{0}0.00}  
 &  \best{\phantom{0}0.00}  
 &  \phantom{0}0.20  
 &  84.51  
 &  78.19  
 &  \phantom{0000}  
 &  \phantom{0}1.27   \\
 \midrule
\hspace{\subtab} Shared BinDisc  ($s_1$)
 &  
 &  31.86    
 &  \best{10.77}  
 &  \best{\phantom{0}0.00}  
 &  \best{\phantom{0}0.00}  
 &  \best{\phantom{0}0.00}  
 &  95.25  
 &  85.11  
 &  \phantom{0000}  
 &  \phantom{0}0.08   \\
\hspace{\subtab} Shared Classi
 &  77.35
 &  67.23    
 &  71.05  
 &  \phantom{0}5.00  
 &  97.70  
 &  69.68  
 &  82.05  
 &  77.89  
 &  \phantom{0000}  
 &  28.38   \\
Shared Combi $s_2$
 &  77.35
 &  33.01    
 &  37.30  
 &  \best{\phantom{0}0.00}  
 &  \best{\phantom{0}0.00}  
 &  \phantom{0}1.06  
 &  82.71  
 &  77.01  
 &  \phantom{0000}  
 &  \phantom{0}1.80   \\
Shared Combi $s_3$
 &  77.35
 &  33.06    
 &  37.57  
 &  \best{\phantom{0}0.00}  
 &  \best{\phantom{0}0.00}  
 &  \phantom{0}1.13  
 &  82.68  
 &  77.01  
 &  \phantom{0000}  
 &  \phantom{0}1.85   \\
 \bottomrule
\end{tabularx}
}
\end{sc}
\end{small}
\end{center}
\vskip -0.1in
\end{table*}

\section{Experiments}\label{section:experiments}
We use CIFAR-10 and CIFAR-100~\citep{krizhevsky2009learning} datasets as in-distribution and OpenImages dataset~\citep{OpenImages2} as training out-distribution.
The 80 Million Tiny Images (80M) dataset~\citep{torralba200880} is the de facto standard for training out-distribution aware models that has been adopted by most prior works, but this dataset has been withdrawn by the authors as~\cite{Birhane_2021_WACV} pointed out the presence of offensive images.
To be able to compare with other state-of-the-art methods without introducing a potential bias due to dataset selection, we include the evaluation with 80M as training out-distribution in Appendix~\ref{section:80M}. Moreover, we show in the appendix results for the binary discriminator trained with different training out-distributions vs. likelihoods resp. likelihood ratios \citep{ren2019likelihood} as OOD method.

We use as OOD detection metric the false positive rate at 95\% true positive rate, FPR@95\%TPR; evaluations with  AUC are in Appendix~\ref{sec:auc}.
We evaluate the OOD detection performance on the following datasets:
SVHN ~\citep{SVHN}, resized LSUN Classroom~\citep{lsun}, Uniform Noise, Smooth Noise generated as described by~\citep{HeiAndBit2019}, the respective other CIFAR dataset, 80M, and CelebA~\citep{CelebA}. We highlight that none of the listed methods has access to those test distributions during training or for fine-tuning as we try to assess the ability of an OOD aware model to generalize to unseen distributions. The FPR for the OpenImages test set is not part of the Mean AUC, since OpenImages \textit{has} been used during training.

The binary discriminators (\textsc{BinDisc}) as well as the classifiers with background class (\textsc{BGC}) and the shared binary discriminator+classifier (\textsc{Shared}) of $p(i|x)$ and $p(y|x,i)$ are trained on the 40-2 Wide Residual Network \citep{ZagKom2016} architecture with the same training schedule as used in~\cite{HenMazDie2019} for training their Outlier Exposure(\textsc{OE}) models. This includes averaging the loss over batches that are twice as large for the out-distribution. This way we ensure that the differences do not arise due to differences in the training schedules or other important details but only on the employed objectives. 
In addition to their standard augmentation and normalization, we apply AutoAugment~\citep{cubuk2019autoaugment} without Cutout, and we use $\lambda=1$ where applicable,
which is a sound choice as we observe in an ablation on $\lambda$ in Appendix~\ref{sec:lambda}.
For the \textsc{Energy} OOD detector, we fine-tune the plain model for 10 epochs with OpenImages as out-distribution.
For the Mahalanobis OOD detector~\citep{LeeEtAl2018b}, we use the models and code published by the authors and use OpenImages for the fine tuning of input noise and layer weighting regression.
Our code is available at \url{https://github.com/j-cb/Breaking_Down_OOD_Detection} and we describe the exact details of the training settings and the used dataset splits in Appendix \ref{sec:exp_details}.

\subsection{Out-distribution aware training with labeled in-distribution data}
In Table \ref{table:OI_AA_FPR} we compare multiple OOD methods trained with training out-distribution OpenImages and CIFAR-10/100 as in-distribution: confidence of standard training (\textsc{Plain}) and OE,  \textsc{Mahalanobis} detection, \textsc{Energy}-Based OOD detection, classifier with background class (\textsc{BGC}) and the combination of a plain classifier and a binary in-vs-out-distribution classifier with shared representation (\textsc{Shared Combi}). As described in Section \ref{section:scores}, both  \textsc{BGC}  and \textsc{Shared Combi} can be used in combination with different scoring functions. For  \textsc{BGC}, we evaluate all three scoring functions $s_1$, $s_2$ and $s_3$ and for \textsc{Shared Combi} we only use $s_2$ and $s_3$ as $s_1$ is equivalent to $p(i|x)$ which is the output of \textsc{Shared BinDisc}. Additionally, 
we evaluate OOD detection based on the confidence of the shared classifier (\textsc{Shared Classi}) trained together with \textsc{Shared BinDisc}.

\begin{table}[!h]
\caption{Mean FPR@95\%TPR over all means for the dataset combinations shown in Tables~\ref{table:OI_AA_FPR}, \ref{table:fpr_80M}, and \ref{table:rImgNet}.
}\label{table:mean_fpr}
\setlength\tabcolsep{.5pt} 
\vskip 0.15in
\begin{center}
\begin{small}
\begin{sc}
\begin{tabularx}{1.0\columnwidth}{lC|lC}
Model & FPR & \hspace{2mm} Model & FPR \\ \midrule
Plain Classi & 60.83 & \hspace{4mm} Shared BinDisc ($s_1$) & 18.89 \\
OE  ($s_3$) & 19.53 & \hspace{4mm} Shared Classi  & 34.93 \\
\hspace{2mm} \backgroundSone & 19.29 & \hspace{2mm} Shared Combi $s_2$ & 18.35 \\
\hspace{2mm} \backgroundStwo & 19.58 & \hspace{2mm} Shared Combi $s_3$ & 18.38 \\
 \backgroundSthree & 19.62 \\
\end{tabularx}
\end{sc}
\end{small}
\end{center}
\vskip -0.1in
\end{table}

\begin{table*}[!htbp]
\caption{
Evaluation (same metrics as in Table~\ref{table:OI_AA_FPR}) of models trained with shared and separate representations. Shared training benefits both the classifier and the binary discriminators. 
\label{table:separate_OI_FPR}
}
\setlength\tabcolsep{.5pt} 
\vskip 0.15in
\begin{center}
\begin{small}
\begin{sc}
\makebox[\textwidth][c]{
\begin{tabularx}{1.0\textwidth}{lC|C|CCCCCCC|C}
\multicolumn{10}{c}{in-distribution: CIFAR-10} \\
\midrule
      & &  Mean &     SVHN    &   LSUN    &   Uni &   Smooth     &   C-100 &  80M &   CelA  &    OpenIm     \\ 
Model    & Acc. &  FPR &    FPR    &   FPR    &   FPR &  FPR    &   FPR & FPR &   FPR &   FPR      \\ \midrule
  \hspace{\subtab} Separate BinDisc  ($s_1$)
 &
 & 23.49
 & 6.21
 & \best{\phantom{0}0.00} 
 & \best{\phantom{0}0.00} 
 & \best{\phantom{0}0.00} 
 & 83.79
 & 65.77
 & 8.68
 & \best{\phantom{0}0.00} 
 \\
 \hspace{\subtab} Plain Classi
 &  95.16
 &  53.01    
 &  47.87  
 &  50.00  
 &  17.51  
 &  65.81  
 &  60.43  
 &  53.44  
 &  76.00  
 &  63.71   \\
  Separate Combi $s_3$
 & 95.16
 & 21.40
 & 13.15
 & \best{\phantom{0}0.00} 
 & \best{\phantom{0}0.00} 
 & \best{\phantom{0}0.00} 
 & 59.96
 & 49.78
 & 26.93
 & \phantom{0}0.45
  \\ \midrule
\hspace{\subtab} Shared BinDisc ($s_1$)
 &  
 &  19.56    
 &  \best{\phantom{0}4.65}
 &  \best{\phantom{0}0.00}  
 &  \best{\phantom{0}0.00}  
 &  \best{\phantom{0}0.00}  
 &  77.50  
 &  53.93  
 &  \best{\phantom{0}0.87}  
 &  \phantom{0}0.04   \\
\hspace{\subtab} Shared Classi 
 &  \textbf{95.28}
 &  29.34    
 &  28.00  
 &  \phantom{0}7.00  
 &  \phantom{0}2.33  
 &  33.04  
 &  58.61  
 &  47.90  
 &  28.54  
 &  35.94   \\
Shared Combi $s_3$
 &  \textbf{95.28}
 &  16.06    
 &  \phantom{0}9.00  
 &  \best{\phantom{0}0.00}  
 &  \best{\phantom{0}0.00}  
 &  \best{\phantom{0}0.00}  
 &  58.68  
 &  42.85  
 &  \phantom{0}1.91  
 &  \phantom{0}0.66   \\
\midrule
  Plain $\otimes$ Sha Disc $s_3$
 & 95.16
 & \textbf{15.96}
 & 8.10
 & \best{\phantom{0}0.00} 
 & \best{\phantom{0}0.00} 
 & \best{\phantom{0}0.00} 
 & \best{58.48}
 & \best{42.60}
 & 2.53
 & \phantom{0}0.70
  \\
\toprule
\multicolumn{10}{c}{in-distribution: CIFAR-100} \\
\midrule
& &  Mean &     SVHN    &   LSUN    &   Uni &   Smooth     &   C-10 &  80M &     &    OpenIm     \\ 
Model    & Acc. &  FPR &    FPR    &   FPR    &   FPR &  FPR    &   FPR & FPR &    &   FPR  \\ \midrule
  \hspace{\subtab} Separate BinDisc ($s_1$)
 &
 & 32.50
 & 14.28
 & \best{\phantom{0}0.00} 
 & \best{\phantom{0}0.00} 
 & \best{\phantom{0}0.00} 
 & 96.50
 & 84.22
 &  \phantom{0000}  
 & \phantom{0}\best{0.02}
 \\
 \hspace{\subtab} Plain Classi
 &  77.16
 &  67.57    
 &  75.50  
 &  78.33  
 &  22.60  
 &  70.98  
 &  \best{80.55}  
 &  77.43  
 &  \phantom{0000}  
 &  80.80   \\
  Separate Combi $s_3$
 & 77.16
 & 41.94
 & 69.44
 & \best{\phantom{0}0.00} 
 & \best{\phantom{0}0.00} 
 & 24.39
 & 81.15
 & 76.67
 &  \phantom{0000}  
 & \phantom{0}0.89
 \\ \midrule
\hspace{\subtab} Shared BinDisc ($s_1$)
 &  
 &  \textbf{31.86}    
 &  \best{10.77}  
 &  \best{\phantom{0}0.00}  
 &  \best{\phantom{0}0.00}  
 &  \best{\phantom{0}0.00}  
 &  95.25  
 &  85.11  
 &  \phantom{0000}  
 &  \phantom{0}0.08
\\
\hspace{\subtab} Shared Classi 
 &  \textbf{77.35}
 &  67.23    
 &  71.05  
 &  \phantom{0}5.00  
 &  97.70  
 &  69.68  
 &  82.05  
 &  77.89  
 &  \phantom{0000}  
 &  28.38
\\
Shared Combi $s_3$
 &  \textbf{77.35}
 &  33.06    
 &  37.57  
 &  \best{\phantom{0}0.00}  
 &  \best{\phantom{0}0.00}  
 &  \phantom{0}1.13  
 &  82.68  
 &  77.01  
 &  \phantom{0000}  
 &  \phantom{0}1.85   
\\ \midrule
  Plain $\otimes$ Sha Disc $s_3$
 & 77.16
 & 33.38
 & 37.00
 & \best{\phantom{0}0.00} 
 & \best{\phantom{0}0.00} 
 & 5.42
 & 81.33
 & \best{76.50}
 &  \phantom{0000}  
 & \phantom{0}2.23
\\
\bottomrule
\end{tabularx}
}
\end{sc}
\end{small}
\end{center}
\vskip -0.1in
\end{table*}

For CIFAR-10, a first interesting observation is that \textsc{Shared Classi} has remarkably good OOD performance; significantly better than a normal classifier (plain) even though it is just trained using normal cross-entropy loss and so the OOD performance is only due to the regularization enforced by the shared representation with \textsc{Shared BinDisc}. In fact \textsc{Shared BinDisc} has already good OOD performance with a mean FPR@95\%TPR of 19.56, which is improved by considering scoring function $s_2/s_3$ in the combination of  \textsc{Shared BinDisc} and \textsc{Shared Classi}
which yields very good classification accuracy and mean FPR/AUC.
Moreover, interesting are the results of the classifier with background class (\textsc{BGC}) which is the method recently advocated in \cite{thulasidasan2021effective}. It works very well but the performance depends on the chosen scoring function. Whereas $s_1$ (output of the background class) is a usable scoring function (mean FPR: 18.83), the maximum probability over the other classes $s_2$ (mean FPR: 16.52) or the combination in terms of $s_3$ (mean FPR: 16.63) performs better.
In total with the scoring function $s_2/s_3$ integrating classifier and discriminative information, \textsc{BGC} reaches similar performance to OE (which implicitly also uses $s_3$ as scoring function).
In general, the differences of the methods are relatively minor both in terms of OOD detection and classification accuracy, where the latter is better for all OOD methods compared to the plain classifier; this is most likely explained by better learned representations, see also \citet{HenMazDie2019,RATIO} 
for similar observations.
The results for CIFAR-100 are similar to CIFAR-10, with some reversals of the overall rankings of the compared methods. OE achieves comparable OOD results to \textsc{BGC} $s_2/s_3$ and \textsc{Shared Combi} $s_2/s_3$. For this in-distribution our \textsc{BGC} $s_1$ and \textsc{Shared BinDisc} perform best in terms of OOD performance. Classification test accuracy is slightly higher for \textsc{BGC} and \textsc{Shared}, but the differences are minor. 
The above observations are confirmed by further experiment further experiments with 80M as training out-distribution in Appendix~\ref{section:80M} as well as with Restricted ImageNet \citep{tsipras2018robustness} as in-distribution and the remaining ImageNet classes as training out-distribution in Appendix~\ref{sec:RImageNet}.
For the reader's convenience, we summarize the mean FPR over the 5 models for each method and over all datasets in Table~\ref{table:mean_fpr}. As discussed, the differences between the different methods using surrogate OOD data are relatively minor, with the none of the methods being strictly better or worse than an other over all 5 settings.

Overall, as suggested by the theoretical results on the equivalence of the Bayes optimal classifier of \textsc{OE} with the $s_3$ scoring function of \textsc{BGC} and \textsc{Shared Combi}, we observe that even though these methods are derived and in particular trained with quite different objectives, they behave very similar in our experiments.
In total we think that this provides a much better understanding where differences of OOD methods are coming from. Regarding the question of which method and scoring function should be used for a given application, the experimental results across datasets and different out-distributions, see Appendix \ref{section:80M}, suggest that their difference is minor and there is no clear best choice.
However, in Appendix~\ref{section:coins}, we describe a potential situation where the $s_3$ score and in consequence OE is not powerful enough to distinguish in- and out-of-distribution inputs. On the other hand, in cases where the $s_1$ score is not very informative as training and test out-distributions largely differ, combining it with the 
classifier confidences is beneficial; this can be observed in experiments with SVHN as training out-distribution which we show in Appendix~\ref{section:SVHN}.
This is why for an unknown situation, we recommend \textsc{BGC} or \textsc{Shared Combi} with the $s_2$ scoring function as the safest option. However, it is an open question if there are also situations where $s_2$ is fundamentally inferior to $s_3$.

\subsection{Shared representation learning for the binary discriminator}
As highlighted above the shared training of \textsc{Shared Classi} and \textsc{Shared BinDisc} and their combination \textsc{Shared Combi} with $s_2/s_3$ as scoring functions yields strong OOD detection and test accuracy among all methods. Here, we evaluate the importance of training the binary discriminator and the plain classifier with a shared representation in comparison to training two entirely separate models \textsc{Plain Classi} and \textsc{Separate BinDisc} and their combination \textsc{Separate Combi} with scoring function $s_3$.
The results for CIFAR-10 and CIFAR-100 can be found in Table \ref{table:separate_OI_FPR}. In total, we see that separate training in particular for CIFAR-100 leads to worse results compared to shared training as expected as the binary discriminator and the classifier cannot benefit from each other. An interesting curiosity is that the combination of the separate classifier with
the binary discriminator trained in a shared fashion (\textsc{Plain $\otimes$ Sha Disc}) yields almost the same OOD results as \textsc{Shared Combi} even though the classifier is significantly worse. Overall, \textsc{Shared Combi} performs significantly better when also considering the better classification accuracy which it inherits  from \textsc{Shared Classi}.

\section{Conclusion}
In this paper we have analyzed different OOD detection methods and have shown that the simple baseline of a binary discriminator between in-and out-distribution is a powerful OOD detection method if trained in a shared fashion with a classifier. Moreover, we have revealed the inner mechanism of Outlier Exposure and training with a background class which unexpectedly use a scoring function which integrates information from $p(i|x)$ \textit{and} $p(y|x,i)$.
We think that these findings will 
allow to build novel OOD methods in a more principled fashion.

\subsubsection*{Acknowledgments}
The authors acknowledge support from the German Federal Ministry of Education and Research (BMBF) through the Tübingen AI Center (FKZ: 01IS18039A) and from the Deutsche Forschungsgemeinschaft (DFG, German Research Foundation) under Germany’s Excellence Strategy (EXC number 2064/1, Project number 390727645), as well as from the DFG TRR 248 (Project number 389792660).
The authors thank the International Max Planck Research School for Intelligent Systems (IMPRS-IS) for supporting Alexander Meinke.


\medskip
\bibliography{main.bib}

\begin{thebibliography}{48}
\providecommand{\natexlab}[1]{#1}
\providecommand{\url}[1]{\texttt{#1}}
\expandafter\ifx\csname urlstyle\endcsname\relax
  \providecommand{\doi}[1]{doi: #1}\else
  \providecommand{\doi}{doi: \begingroup \urlstyle{rm}\Url}\fi

\bibitem[Augustin et~al.(2020)Augustin, Meinke, and Hein]{RATIO}
Augustin, M., Meinke, A., and Hein, M.
\newblock Adversarial robustness on in-and out-distribution improves
  explainability.
\newblock In \emph{ECCV}, 2020.

\bibitem[Birhane \& Prabhu(2021)Birhane and Prabhu]{Birhane_2021_WACV}
Birhane, A. and Prabhu, V.~U.
\newblock Large image datasets: A pyrrhic win for computer vision?
\newblock In \emph{WACV}, pp.\  1537--1547, 2021.

\bibitem[Bishop(1994)]{bishop1994}
Bishop, C.~M.
\newblock Novelty detection and neural network validation.
\newblock \emph{IEEE Proceedings-Vision, Image and Signal processing},
  141:\penalty0 217--222, 1994.

\bibitem[Boyd et~al.(2004)Boyd, Boyd, and Vandenberghe]{BoydVandenberghe}
Boyd, S., Boyd, S.~P., and Vandenberghe, L.
\newblock \emph{Convex optimization}.
\newblock Cambridge University Press, 2004.

\bibitem[Chen et~al.(2021)Chen, Li, Wu, Liang, and
  Jha]{chen2020informative-outlier-matters}
Chen, J., Li, Y., Wu, X., Liang, Y., and Jha, S.
\newblock Informative outlier matters: Robustifying out-of-distribution
  detection using outlier mining.
\newblock In \emph{ECML}, 2021.

\bibitem[Cubuk et~al.(2019)Cubuk, Zoph, Mane, Vasudevan, and
  Le]{cubuk2019autoaugment}
Cubuk, E.~D., Zoph, B., Mane, D., Vasudevan, V., and Le, Q.~V.
\newblock Autoaugment: Learning augmentation strategies from data.
\newblock In \emph{CVPR}, 2019.

\bibitem[Deng et~al.(2009)Deng, Dong, Socher, Li, Li, and
  Fei-Fei]{imagenet_cvpr09}
Deng, J., Dong, W., Socher, R., Li, L.-J., Li, K., and Fei-Fei, L.
\newblock Imagenet: A large-scale hierarchical image database.
\newblock In \emph{CVPR}, 2009.

\bibitem[Golan \& El-Yaniv(2018)Golan and El-Yaniv]{golan2018deep}
Golan, I. and El-Yaniv, R.
\newblock Deep anomaly detection using geometric transformations.
\newblock \emph{arXiv preprint arXiv:1805.10917}, 2018.

\bibitem[Grathwohl et~al.(2020)Grathwohl, Wang, Jacobsen, Duvenaud, Norouzi,
  and Swersky]{Grathwohl2020Your}
Grathwohl, W., Wang, K.-C., Jacobsen, J.-H., Duvenaud, D., Norouzi, M., and
  Swersky, K.
\newblock Your classifier is secretly an energy based model and you should
  treat it like one.
\newblock In \emph{International Conference on Learning Representations}, 2020.
\newblock URL \url{https://openreview.net/forum?id=Hkxzx0NtDB}.

\bibitem[Hein et~al.(2019)Hein, Andriushchenko, and Bitterwolf]{HeiAndBit2019}
Hein, M., Andriushchenko, M., and Bitterwolf, J.
\newblock Why {ReLU} networks yield high-confidence predictions far away from
  the training data and how to mitigate the problem.
\newblock In \emph{CVPR}, 2019.

\bibitem[Hendrycks \& Gimpel(2017)Hendrycks and Gimpel]{HenGim2017}
Hendrycks, D. and Gimpel, K.
\newblock A baseline for detecting misclassified and out-of-distribution
  examples in neural networks.
\newblock In \emph{ICLR}, 2017.

\bibitem[Hendrycks et~al.(2019{\natexlab{a}})Hendrycks, Mazeika, and
  Dietterich]{HenMazDie2019}
Hendrycks, D., Mazeika, M., and Dietterich, T.
\newblock Deep anomaly detection with outlier exposure.
\newblock In \emph{ICLR}, 2019{\natexlab{a}}.
\newblock \url{https://github.com/hendrycks/outlier-exposure}.

\bibitem[Hendrycks et~al.(2019{\natexlab{b}})Hendrycks, Mazeika, Kadavath, and
  Song]{hendrycks2019selfsupervised}
Hendrycks, D., Mazeika, M., Kadavath, S., and Song, D.
\newblock Using self-supervised learning can improve model robustness and
  uncertainty.
\newblock \emph{Advances in Neural Information Processing Systems (NeurIPS)},
  2019{\natexlab{b}}.

\bibitem[Krasin et~al.(2017)Krasin, Duerig, Alldrin, Ferrari, Abu-El-Haija,
  Kuznetsova, Rom, Uijlings, Popov, Kamali, Malloci, Pont-Tuset, Veit,
  Belongie, Gomes, Gupta, Sun, Chechik, Cai, Feng, Narayanan, and
  Murphy]{OpenImages2}
Krasin, I., Duerig, T., Alldrin, N., Ferrari, V., Abu-El-Haija, S., Kuznetsova,
  A., Rom, H., Uijlings, J., Popov, S., Kamali, S., Malloci, M., Pont-Tuset,
  J., Veit, A., Belongie, S., Gomes, V., Gupta, A., Sun, C., Chechik, G., Cai,
  D., Feng, Z., Narayanan, D., and Murphy, K.
\newblock Openimages: A public dataset for large-scale multi-label and
  multi-class image classification.
\newblock \emph{Dataset available from
  https://storage.googleapis.com/openimages/web/index.html}, 2017.

\bibitem[Krizhevsky \& Hinton(2009)Krizhevsky and
  Hinton]{krizhevsky2009learning}
Krizhevsky, A. and Hinton, G.
\newblock Learning multiple layers of features from tiny images.
\newblock Technical report, Citeseer, 2009.

\bibitem[Laptev et~al.(2016)Laptev, Savinov, Buhmann, and
  Pollefeys]{LapEtAl2016}
Laptev, D., Savinov, N., Buhmann, J., and Pollefeys, M.
\newblock {TI}-pooling: {T}ransformation-invariant pooling for feature learning
  in convolutional neural networks.
\newblock In \emph{CVPR}, 2016.

\bibitem[LeCun et~al.(2006)LeCun, Chopra, Hadsell, Ranzato, and
  Huang]{lecun2006tutorial}
LeCun, Y., Chopra, S., Hadsell, R., Ranzato, M., and Huang, F.
\newblock A tutorial on energy-based learning.
\newblock 2006.

\bibitem[Lee et~al.(2018{\natexlab{a}})Lee, Lee, Lee, and Shin]{LeeEtAl2018}
Lee, K., Lee, H., Lee, K., and Shin, J.
\newblock Training confidence-calibrated classifiers for detecting
  out-of-distribution samples.
\newblock In \emph{ICLR}, 2018{\natexlab{a}}.

\bibitem[Lee et~al.(2018{\natexlab{b}})Lee, Lee, Lee, and Shin]{LeeEtAl2018b}
Lee, K., Lee, H., Lee, K., and Shin, J.
\newblock A simple unified framework for detecting out-of-distribution samples
  and adversarial attacks.
\newblock In \emph{NeurIPS}, 2018{\natexlab{b}}.
\newblock \url{https://github.com/pokaxpoka/deep_Mahalanobis_detector}.

\bibitem[Lee et~al.(2018{\natexlab{c}})Lee, Lee, Lee, and Shin]{lee2018simple}
Lee, K., Lee, K., Lee, H., and Shin, J.
\newblock A simple unified framework for detecting out-of-distribution samples
  and adversarial attacks.
\newblock In \emph{NeurIPS}, 2018{\natexlab{c}}.

\bibitem[Li \& Vasconcelos(2020)Li and Vasconcelos]{li2020background}
Li, Y. and Vasconcelos, N.
\newblock Background data resampling for outlier-aware classification.
\newblock In \emph{Proceedings of the IEEE/CVF Conference on Computer Vision
  and Pattern Recognition}, pp.\  13218--13227, 2020.

\bibitem[Liang et~al.(2018)Liang, Li, and Srikant]{liang2017enhancing}
Liang, S., Li, Y., and Srikant, R.
\newblock Enhancing the reliability of out-of-distribution image detection in
  neural networks.
\newblock In \emph{ICLR}, 2018.

\bibitem[Liu \& Abbeel(2020)Liu and Abbeel]{liu2020hybrid}
Liu, H. and Abbeel, P.
\newblock Hybrid discriminative-generative training via contrastive learning.
\newblock \emph{arXiv preprint arXiv:2007.09070}, 2020.

\bibitem[Liu et~al.(2020)Liu, Wang, Owens, and Li]{liu2020energy}
Liu, W., Wang, X., Owens, J., and Li, Y.
\newblock Energy-based out-of-distribution detection.
\newblock \emph{Advances in Neural Information Processing Systems}, 2020.

\bibitem[Liu et~al.(2015)Liu, Luo, Wang, and Tang]{CelebA}
Liu, Z., Luo, P., Wang, X., and Tang, X.
\newblock Deep learning face attributes in the wild.
\newblock In \emph{ICCV}, 2015.

\bibitem[Madry et~al.(2018)Madry, Makelov, Schmidt, Tsipras, and
  Valdu]{MadEtAl2018}
Madry, A., Makelov, A., Schmidt, L., Tsipras, D., and Valdu, A.
\newblock Towards deep learning models resistant to adversarial attacks.
\newblock In \emph{ICLR}, 2018.

\bibitem[Meinke \& Hein(2020)Meinke and Hein]{meinke2020towards}
Meinke, A. and Hein, M.
\newblock Towards neural networks that provably know when they don't know.
\newblock In \emph{ICLR}, 2020.

\bibitem[Mohseni et~al.(2020)Mohseni, Pitale, Yadawa, and
  Wang]{Mohseni2020SelfSupervisedLF}
Mohseni, S., Pitale, M., Yadawa, J., and Wang, Z.
\newblock Self-supervised learning for generalizable out-of-distribution
  detection.
\newblock In \emph{AAAI}, 2020.

\bibitem[Nalisnick et~al.(2019)Nalisnick, Matsukawa, Teh, Gorur, and
  Lakshminarayanan]{nalisnick19b}
Nalisnick, E., Matsukawa, A., Teh, Y.~W., Gorur, D., and Lakshminarayanan, B.
\newblock Hybrid models with deep and invertible features.
\newblock In \emph{ICML}, 2019.

\bibitem[{Nalisnick} et~al.(2019){Nalisnick}, {Matsukawa}, {Whye Teh}, {Gorur},
  and {Lakshminarayanan}]{NalEtAl2018}
{Nalisnick}, E., {Matsukawa}, A., {Whye Teh}, Y., {Gorur}, D., and
  {Lakshminarayanan}, B.
\newblock Do deep generative models know what they don't know?
\newblock In \emph{ICLR}, 2019.

\bibitem[Netzer et~al.(2011)Netzer, Wang, Coates, Bissacco, Wu, and Ng]{SVHN}
Netzer, Y., Wang, T., Coates, A., Bissacco, A., Wu, B., and Ng, A.~Y.
\newblock Reading digits in natural images with unsupervised feature learning.
\newblock In \emph{NeurIPS Workshop on Deep Learning and Unsupervised Feature
  Learning}, 2011.

\bibitem[Nguyen et~al.(2015)Nguyen, Yosinski, and Clune]{NguYosClu2015}
Nguyen, A., Yosinski, J., and Clune, J.
\newblock Deep neural networks are easily fooled: {H}igh confidence predictions
  for unrecognizable images.
\newblock In \emph{CVPR}, 2015.

\bibitem[Papadopoulos et~al.(2021)Papadopoulos, Rajati, Shaikh, and Wang]{OECC}
Papadopoulos, A.-A., Rajati, M.~R., Shaikh, N., and Wang, J.
\newblock Outlier exposure with confidence control for out-of-distribution
  detection.
\newblock \emph{Neurocomputing}, 441:\penalty0 138--150, 2021.

\bibitem[Paszke et~al.(2019)Paszke, Gross, Massa, Lerer, Bradbury, Chanan,
  Killeen, Lin, Gimelshein, Antiga, Desmaison, Kopf, Yang, DeVito, Raison,
  Tejani, Chilamkurthy, Steiner, Fang, Bai, and Chintala]{PyTorch}
Paszke, A., Gross, S., Massa, F., Lerer, A., Bradbury, J., Chanan, G., Killeen,
  T., Lin, Z., Gimelshein, N., Antiga, L., Desmaison, A., Kopf, A., Yang, E.,
  DeVito, Z., Raison, M., Tejani, A., Chilamkurthy, S., Steiner, B., Fang, L.,
  Bai, J., and Chintala, S.
\newblock Pytorch: An imperative style, high-performance deep learning library.
\newblock In \emph{NeurIPS}, pp.\  8024--8035, 2019.

\bibitem[Pedregosa et~al.(2011)Pedregosa, Varoquaux, Gramfort, Michel, Thirion,
  Grisel, Blondel, Prettenhofer, Weiss, Dubourg, Vanderplas, Passos,
  Cournapeau, Brucher, Perrot, and Duchesnay]{scikit-learn}
Pedregosa, F., Varoquaux, G., Gramfort, A., Michel, V., Thirion, B., Grisel,
  O., Blondel, M., Prettenhofer, P., Weiss, R., Dubourg, V., Vanderplas, J.,
  Passos, A., Cournapeau, D., Brucher, M., Perrot, M., and Duchesnay, E.
\newblock Scikit-learn: Machine learning in {P}ython.
\newblock \emph{JMLR}, 12:\penalty0 2825--2830, 2011.

\bibitem[Ren et~al.(2019)Ren, Liu, Fertig, Snoek, Poplin, DePristo, Dillon, and
  Lakshminarayanan]{ren2019likelihood}
Ren, J., Liu, P.~J., Fertig, E., Snoek, J., Poplin, R., DePristo, M.~A.,
  Dillon, J.~V., and Lakshminarayanan, B.
\newblock Likelihood ratios for out-of-distribution detection.
\newblock In \emph{NeurIPS}, 2019.

\bibitem[Roy et~al.(2021)Roy, Ren, Azizi, Loh, Natarajan, Mustafa, Pawlowski,
  Freyberg, Liu, Beaver, Vo, Bui, Winter, MacWilliams, Corrado, Telang, Liu,
  Cemgil, Karthikesalingam, Lakshminarayanan, and Winkens]{DermaOOD}
Roy, A.~G., Ren, J.~J., Azizi, S., Loh, A., Natarajan, V., Mustafa, B.,
  Pawlowski, N., Freyberg, J., Liu, Y., Beaver, Z.~W., Vo, N., Bui, P., Winter,
  S., MacWilliams, P., Corrado, G., Telang, U., Liu, Y., Cemgil, T.,
  Karthikesalingam, A., Lakshminarayanan, B., and Winkens, J.
\newblock Does your dermatology classifier know what it doesn't know? detecting
  the long-tail of unseen conditions.
\newblock \emph{Medical Imaging Analysis}, 2021.

\bibitem[Russakovsky et~al.(2015)Russakovsky, Deng, Su, Krause, Satheesh, Ma,
  Huang, Karpathy, Khosla, Bernstein, Berg, and Fei-Fei]{ILSVRC15}
Russakovsky, O., Deng, J., Su, H., Krause, J., Satheesh, S., Ma, S., Huang, Z.,
  Karpathy, A., Khosla, A., Bernstein, M., Berg, A.~C., and Fei-Fei, L.
\newblock {ImageNet Large Scale Visual Recognition Challenge}.
\newblock \emph{International Journal of Computer Vision (IJCV)}, 115\penalty0
  (3):\penalty0 211--252, 2015.
\newblock \doi{10.1007/s11263-015-0816-y}.

\bibitem[Salimans et~al.(2017)Salimans, Karpathy, Chen, and
  Kingma]{Salimans2017PixeCNN}
Salimans, T., Karpathy, A., Chen, X., and Kingma, D.~P.
\newblock Pixelcnn++: A pixelcnn implementation with discretized logistic
  mixture likelihood and other modifications.
\newblock In \emph{ICLR}, 2017.

\bibitem[Szegedy et~al.(2014)Szegedy, Zaremba, Sutskever, Bruna, Erhan,
  Goodfellow, and Fergus]{SzeEtAl2014}
Szegedy, C., Zaremba, W., Sutskever, I., Bruna, J., Erhan, D., Goodfellow, I.,
  and Fergus, R.
\newblock Intriguing properties of neural networks.
\newblock In \emph{ICLR}, 2014.

\bibitem[Tack et~al.(2020)Tack, Mo, Jeong, and Shin]{tack2020csi}
Tack, J., Mo, S., Jeong, J., and Shin, J.
\newblock Csi: Novelty detection via contrastive learning on distributionally
  shifted instances.
\newblock In \emph{Advances in Neural Information Processing Systems}, 2020.

\bibitem[Thulasidasan et~al.(2021)Thulasidasan, Thapa, Dhaubhadel, Chennupati,
  Bhattacharya, and Bilmes]{thulasidasan2021effective}
Thulasidasan, S., Thapa, S., Dhaubhadel, S., Chennupati, G., Bhattacharya, T.,
  and Bilmes, J.
\newblock An effective baseline for robustness to distributional shift.
\newblock \emph{arXiv: 2105.07107}, 2021.

\bibitem[Torralba et~al.(2008)Torralba, Fergus, and Freeman]{torralba200880}
Torralba, A., Fergus, R., and Freeman, W.~T.
\newblock 80 million tiny images: A large data set for nonparametric object and
  scene recognition.
\newblock \emph{IEEE PAMI}, 30\penalty0 (11):\penalty0 1958--1970, 2008.

\bibitem[Tsipras et~al.(2019)Tsipras, Santurkar, Engstrom, Turner, and
  Madry]{tsipras2018robustness}
Tsipras, D., Santurkar, S., Engstrom, L., Turner, A., and Madry, A.
\newblock Robustness may be at odds with accuracy.
\newblock In \emph{International Conference on Learning Representations}, 2019.
\newblock URL \url{https://openreview.net/forum?id=SyxAb30cY7}.

\bibitem[Winkens et~al.(2020)Winkens, Bunel, Roy, Stanforth, Natarajan, Ledsam,
  MacWilliams, Kohli, Karthikesalingam, Kohl, et~al.]{winkens2020contrastive}
Winkens, J., Bunel, R., Roy, A.~G., Stanforth, R., Natarajan, V., Ledsam,
  J.~R., MacWilliams, P., Kohli, P., Karthikesalingam, A., Kohl, S., et~al.
\newblock Contrastive training for improved out-of-distribution detection.
\newblock \emph{arXiv preprint arXiv:2007.05566}, 2020.

\bibitem[Xiao et~al.(2020)Xiao, Yan, and Amit]{XiaoLikelihoodRegret}
Xiao, Z., Yan, Q., and Amit, Y.
\newblock Likelihood regret: An out-of-distribution detection score for
  variational auto-encoder.
\newblock In \emph{NeurIPS}, 2020.

\bibitem[Yu et~al.(2015)Yu, Zhang, Song, Seff, and Xiao]{lsun}
Yu, F., Zhang, Y., Song, S., Seff, A., and Xiao, J.
\newblock Lsun: Construction of a large-scale image dataset using deep learning
  with humans in the loop.
\newblock \emph{CoRR}, abs/1506.03365, 2015.

\bibitem[Zagoruyko \& Komodakis(2016)Zagoruyko and Komodakis]{ZagKom2016}
Zagoruyko, S. and Komodakis, N.
\newblock Wide residual networks.
\newblock In \emph{BMVC}, pp.\  87.1--87.12, 2016.

\end{thebibliography}
\bibliographystyle{icml2022}

\onecolumn

\appendix

\section{Proofs}
\scoreequiv*
\begin{proof}\noindent
 Assume that such a function $\phi$ exists. Then for any pair $x,y$ we have the logical equivalences $g(x) > g(y) \Leftrightarrow f(x) = \phi(g(x)) > \phi(g(y)) = f(y)$ and $g(x) = g(y) \Leftrightarrow f(x) = \phi(g(x)) = \phi(g(y)) = f(y)$. This directly implies that the AUCs are the same, regardless of the distributions.

 Assume $f\cong g$. For each $a \in \mathrm{range}(g)$, choose some $\hat{a} \in g^{-1}(a)$.
    For any pair $x,y \in X$, by regarding the Dirac distributions $p(x|i) = \delta_x$ and $p(x|o) = \delta_y$ that are each concentrated on one of the points, we can infer that $f(x) > f(y) \Leftrightarrow \mathrm{AUC}_f (p(x|i), p(x|o)) = 1 \Leftrightarrow \mathrm{AUC}_g (p(x|i), p(x|o)) = 1 \Leftrightarrow g(x) > g(y)$ and similarly $f(x) = f(y) \Leftrightarrow g(x) = g(y)$.
    The latter ensures that the function defined as
    \begin{align}
    \begin{split}
          \phi: \mathrm{range}(g) &\rightarrow\mathrm{range}(f) \\
        a &\mapsto f(\hat{a})      
    \end{split}
    \end{align}
    is independent of the choice of $\hat{a}$ and that $f = \phi \circ g$, and the former confirms that $\phi$ is strictly monotonously increasing. 
\end{proof}

\equivfpr*
\begin{proof}
    We know that a function $\phi$ as in Theorem \ref{thm:score_equiv} exists. Then for any pair $x,y$, we have the logical equivalences \begin{align}
    g(x) > g(y) \Leftrightarrow f(x) = \phi(g(x)) > \phi(g(y)) = f(y)
    \end{align}
    and
    \begin{align}
    g(x) = g(y) \Leftrightarrow f(x) = \phi(g(x)) = \phi(g(y)) = f(y) \ .
    \end{align}
    This directly implies that the FPR@qTPR-values are the same, for any $p(x|i), p(x|o)$ and q.
\end{proof}

\lhratios*
\begin{proof}
    The function $\phi: [0, \infty] \rightarrow [0,1]$ defined by $\phi(x) = \frac{x}{x + \lambda}$ (setting  $\phi(\infty) = 1$) fulfills the criterion from Theorem \ref{thm:score_equiv} of being strictly monotonously increasing.
    With 
    \begin{align}
        \phi\left(\frac{p(x|i)}{p(x|o)}\right)
        &= \frac{\frac{p(x|i)}{p(x|o)}}{\frac{p(x|i)}{p(x|o)}  + \lambda \frac{p(x|o)}{p(x|o)}} = \frac{p(x|i)}{p(x|i)+\lambda p(x|o)}
    \end{align}
    for $p(x|o) \neq 0$ and $\phi\left(\frac{p(x|i)}{0}\right) = \phi(\infty) = 1 = \frac{p(x|i)}{p(x|i)+\lambda \cdot 0}$, the equivalence follows.
\end{proof}

\lhuni*
\begin{proof}
    This is a special case of Lemma \ref{lem:lh_ratios}, by setting $p(x|o) = 1 = p_\mathrm{Uniform}(x)$.
\end{proof}

\OEoptim*
\begin{proof}
Minimizing the loss of Outlier Exposure
\begin{align}\label{eq:loss_OE}
    \min_\theta \E_{\substack{(x,y)\sim \\p(x,y|i)}}\! \left[\Lce(f_{\theta}(x), y)\right] + \lambda \!\! \E_{x\sim p(x|o)}\left[ \Lce(f_{\theta}(x), u^K)\right] \ ,
\end{align}
means solving the optimization problem
\begin{align}\begin{split}
    \min_{p_f(\cdot|x)} \ \  & -p(i|x) \cdot \sum_{k=1}^K p(k|x,i) \cdot \log p_f(k|x) - (1 - p(i|x)) \cdot \sum_{k=1}^K \frac{1}{K} \cdot \log p_f(k|x) \\[2mm]
    \text{subject to} & \hspace{0.6cm} p_f(k|x) \geq 0 \ \text{ for each $k \in \{1, \ldots, K\}$} \\
    & \hspace{0.6cm} \sum_{k=1}^K p_f(k|x) =1 \ ,
\end{split}
\end{align}
where $p_f(\cdot|x)$ is the model's $K$-dimensional prediction.
For $p(i|x) = 0$ or $p(k|x,i) = 0$, the optimalities of the respective terms are easy to show (applying the common conventions for $0\log 0$), so we assume that tose are non-zero.
The Lagrange function of the optimization problem is 
\begin{align}\begin{split}
    L(p_f(\cdot|x), \alpha, \beta) = -&p(i|x) \cdot \sum_{k=1}^K p(k|x,i) \cdot \log p_f(k|x) - (1 - p(i|x)) \cdot \sum_{k=1}^K \frac{1}{K} \cdot \log p_f(k|x) \\
    &- \sum_{k=1}^K \alpha_k p_f(k|x) + \beta \left(-1+\sum_{k=1}^K p_f(k|x)\right) \ ,
\end{split}
\end{align} with $\beta \in \R$ and $\alpha \in \R_+^K$.
Its first derivative with respect to $p_f(k|x)$ for any $k$ is
\begin{align}\label{first_derivative_L}
    \begin{split}
         \frac{\partial L}{p_f(k|x)} 
    &= -p(i|x) \cdot p(k|x,i) \frac{1}{p_f(k|x)} - (1 - p(i|x)) \cdot \frac{1}{K} \frac{1}{p_f(k|x)} - \alpha_k + \beta \\
    &= -\frac{s^K(k|x)}{p_f(k|x)} - \alpha_k + \beta \ ,
    \end{split}
\end{align}
where we set $s^K(k|x) := p(i|x) p(k|x,i) + \frac{1}{K} \big(1-p(i|x)\big)$.
The second derivative is a positive diagonal matrix for any point of its domain, therefore we find the unique minimum by setting (\ref{first_derivative_L}) to zero, i.e. at
\begin{align}
    p_f(k|x) = \frac{s^K(k|x)}{\beta - \alpha_k} \ .
\end{align}
The dual problem is hence the maximization (with $\alpha_k \geq 0$) of 
\begin{align*}
    q(\alpha, \beta)
    &= -p(i|x) \cdot \sum_{k=1}^K p(k|x,i) \cdot \log \frac{s^K(k|x)}{\beta - \alpha_k} - (1 - p(i|x)) \cdot \sum_{k=1}^K \frac{1}{K} \cdot \log \frac{s^K(k|x)}{\beta - \alpha_k} \\
    &\hspace{2cm} - \sum_{k=1}^K \alpha_k \frac{s^K(k|x)}{\beta - \alpha_k} + \beta \left(-1+\sum_{k=1}^K \frac{s^K(k|x)}{\beta - \alpha_k}\right) \\
    &= \sum_{k=1}^K s^K(k|x) \left( - \log s^K(k|x) + \log(\beta - \alpha_k) + \frac{\beta}{\beta - \alpha_k} - \frac{\alpha_k}{\beta - \alpha_k} \right) - \beta \ ;
\end{align*}
here, $\alpha$ only appears in $\log(\beta - \alpha_k)$ which has a positive factor $s^K(k|x)$, so $\alpha = 0$ maximizes the expression.
Noting $\sum_{k=1}^K s^K(k|x) = 1$, what remains is $q^0(\beta) = 1 + \log(\beta) - \sum_{k=1}^K s^K(k|x) \log s^K(k|x) - \beta$, which is maximized by $\beta = 1$.
This means that the dual optimal pair is $p_f(k|x) = s^K(k|x), (\beta = 1, \alpha = 0)$. Slater's condition \citep{BoydVandenberghe} holds since the feasible set of the original problem is the probability simplex. Thus, $p_{f*}(\cdot|x) = s^K(x)$ is indeed primal optimal.
\end{proof}

\section{Confidence loss models and models with a background class or a binary discriminator are only equivalent after $s_3$ is applied}\label{section:coins}
Seeing that for both in-distribution accuracy and OOD detection, OE models trained with confidence loss, models with background class and shared classifier/discriminator combinations behave very similarly, the question arises if the training methods themselves lead to equivalent models.
One idea might be that and effect of the confidence loss on the degree of freedom from logit translation invariance could be unfolded to obtain a $K+1$-dimensional output that contains the same information as a classifier with background class or with and additional binary discriminator output (the latter two are indeed equivalent).
This is not the case, as the following example that in certain situations, the $s_1$ and $s_2$ scores of background class models and classifier/discriminator combinations are able to separate in- and out-distribution, while the $s_3$ score and the equivalent confidence loss/OE models cannot.

\subsection{Euro coin classifier}
As an example where the mentioned non-equivalence would occur, we hypothetically regard the task of classifying photos of 1-Euro coins by the issuing country.
Each €1 coin features a common side that is the same for each country and a national side that pictures a unique motive per country.
We assume that one side is visible on each photo, and that the training dataset of size $2mK$ is balanced, consisting of $2m$ coin photos with label $c$ for each country $c \in \{1, \ldots, K\}$, where $m$ photos show the common side and the other $m$ photos show the informative national side for each country $c$.

It is easy to see that the Bayes optimal classifier trained with cross-entropy loss on this dataset predicts $100\%$ for the respective country $c$ when shown a photo of the national side of a €1 coin, and predicts $1/K$ for each country when shown the common side of a €1 coin.

Now we compare the behaviour of the different methods given a training out-distribution of poker chips images which are clearly recognizable as not being €1 coins.

A $K$-class model trained with confidence loss~\citep{LeeEtAl2018} will not make a difference between common side coin images and poker chip images, and in the Bayes optimal case, it will predict the uniform class distribution in both cases. 
This does not only hold for the prediction of a hypothetical Bayes optimal model: assuming full batch gradient descent and identical sets of $m$ common side training photos for each class,  the loss for a common side input is the same as the loss for a poker chip.

On the other hand, a binary discriminator will easily distinguish between poker chips and €1 coins, no matter which side of the coin is shown. The same holds for a model with background class: the score of the class $K+1$ will be close to $1$ for chips and close to $0$ for €1 coins.

We conclude that in the described situation, models trained with confidence loss/outlier exposure are not able to sufficiently distinguish in- and out-distribution, while the $s_1$ scoring function of a classifier with background class or a binary discriminator is suitable for this task.

With the $s_2$ scoring function, the background class model gives us  $s_2(x)=\max_{k=1,\ldots,K} p_{f}(k|x)$, and thus $\frac{(1-p_{f}(K+1|x))}{K} \leq s_2(x) \leq 1-p_{f}(K+1|x)$, which means that if $p_{f}(K+1|x)$ is sufficiently large for in-distribution inputs and sufficiently small for out-of-distribution inputs,  $s_2$ is able to distinguish them independent of inconclusiveness in the first $K$ classes. Similarly, $s_2$ applied to a binary discriminator with a classifier (shared trained or not) will be able to distinguish common sides of coins and poker chips.

With $s_3$, on the other hand, common sides of coins and poker chips can no longer be separated.
For a classifier/discriminator pair, as defined above, $\displaystyle
s_3(x) = p_d(i|x) \Big[ \max_{y=1,\ldots,K} p_c(y|x,i)-\frac{1}{K}\Big]  +\frac{1}{K}$.
If on the common side of a coin the classifier predicts uniform $\frac{1}{K}$, we have $s_3(x) = \frac{1}{K}$ no matter what the discriminator $p_d(i|x)$ predicts.
On poker chips with discriminator prediction $p_d(i|x) = 1$, we also get $s_3(x) = 1$.
For background class models, $\displaystyle s_3(x)=\max_{k=1,\ldots,K} p_{f}(k|x) + \frac{1}{K}p_{f}(K+1|x)$ also yields $\frac{1}{K}$ for a common side where the prediction over the $K$ in-distribution classes is uniform and for a poker chips, where $p_{f}(K+1|x) = 1$.
The fact that in this coin scenario when scored with $s_3$, background class and classifier/discriminator combinations have the same problem as confidence loss/OE is not surprising considering their equivalence shown in Theorem~\ref {thm:OE_optim}.

\section{Experimental Details}\label{sec:exp_details}
\subsection{Training}
For training our models, we build upon the code of~\citet{HenMazDie2019} which they have available at \url{https://github.com/hendrycks/outlier-exposure} and borrow their general architecture and training settings.
Concretely, we use 40-2 Wide Residual Network~\citep{ZagKom2016} models with normalization based on the CIFAR datasets and a dropout rate of 0.3.
They are trained for 100 epochs with an initial learning rate of $0.1$ that decreases following a cosine annealing schedule.
Unless mentioned otherwise, each training step uses a batch of size 128 for the in-distribution and a batch of size 256 for the training out-distribution. The optimizer uses stochastic gradient descent with a Nesterov momentum of 0.9.
Weight decay is set to $5\cdot10^{-4}$.
The deep learning framework we use is PyTorch~\citep{PyTorch}, and for evaluating we use the scikit-learn~\citep{scikit-learn} implementation of the AUC.
Our code is available at \url{https://github.com/j-cb/Breaking_Down_OOD_Detection}.

For evaluating the Mahalanobis detector, we use the code by the authors of \citet{LeeEtAl2018b} provided at \url{https://github.com/pokaxpoka/deep_Mahalanobis_detector}. The input noise levels and regression parameters are chosen on the available out-distribution OpenImages and are 0.0014 for CIFAR-10 and 0.002 for CIFAR-100. 

All experiments were run on Nvidia V100 GPUs of an internal cluster of our institution, using up to 4 GB GPU memory (batch sizes in:128/out:256), with no noticeable difference between ours and the compared OE~\cite{HenMazDie2019} runs.

\subsection{Datasets}
We train our models with the train splits of CIFAR-10 and CIFAR-100~\cite{krizhevsky2009learning} (MIT license) which each consist of 50,000 labeled images, and evaluate on their test splits of 10,000 samples.
As training out-distribution we use OpenImages~v4~\citep{OpenImages2} (images have a CC BY 2.0 license); the training split that we employ here consists of 8,945,291 images of different sizes, which get resized to $32 \times 32$ pixels, and we test on 10,000 from the official validation split.
For training with 80 Million Tiny Images (80M)~\citep{torralba200880} in Appendix~\ref{section:80M} (no license, see links in Appendix~\ref{section:80M}), we use data from the beginning of the sequentialized dataset, and evaluate on a test set of 30,080 images starting at index 50,000,000. A subset of CIFAR images contained in 80M is excluded for training and evaluation.
Further image datasets used for evaluation are SVHN~\citep{SVHN} (free for non-commercial use) with 26,032 samples, LSUN~\citep{lsun}  Classroom (no license) with 300 samples, and CelebA~\citep{CelebA} (available for non-commercial research purposes only) with 19,962 test samples.
Uniform and Smooth Noise~\citep{HeiAndBit2019} are sampled, the latter by generating uniform noise and smoothing it using a Gaussian filter with a width that is drawn uniformly at random in $[1, 2.5]$. Each datapoint is then shifted and scaled linearly such that the minimal and maximal pixel values are 0 and 1, respectively. For both noises, we evaluate 30,080 inputs.

\section{Related Work on OOD Detection}\label{sec:related_work}
Out-of-distribution detection has been an important research area in recent years, and several approaches that are fitted towards different training and inference scenarios have been proposed.

One seemingly obvious line of thought is to use generative models for density estimation to differentiate between in- and out-distribution~\citep{bishop1994,NalEtAl2018,ren2019likelihood,nalisnick19b,XiaoLikelihoodRegret}. Recent methods to a certain extent overcome the problem mentioned in \citet{NalEtAl2018} that generative models can assign higher likelihood to distributions on which they have not been trained. Another line of work are score-based methods using an underlying classifier or the internal features of such a classifier, potentially combined with a generative model \citep{HenGim2017,liang2017enhancing, lee2018simple,HenMazDie2019,HeiAndBit2019}. One of the most effective methods up to now is Outlier Exposure ~\citep{HenMazDie2019} and work building upon it~\citep{chen2020informative-outlier-matters,meinke2020towards,Mohseni2020SelfSupervisedLF,RATIO,OECC,thulasidasan2021effective} where a classifier is trained on the in-distribution task and one enforces low confidence as proposed by \citet{LeeEtAl2018} during training on a large and diverse set of out-of-distribution images \citep{HenMazDie2019} which can be seen as a proxy of all natural images. This approach generalizes well to other out-distributions.
Recently, NTOM~\citep{chen2020informative-outlier-matters} has achieved excellent results for detecting far out-of-distribution data by adding a background class to the classifier which is trained on samples from the surrogate out-distribution that are mined such that they show a desired hardness for the model. 
At test time, the output probability for that class is used to decide if an input is to be flagged as OOD. Their ATOM method does the same while also adding adversarial perturbations to the OOD inputs during training.
Even though it has been claimed that new approaches outperform \citet{HenMazDie2019}, up to our knowledge this has not been shown consistently across different and challenging test out-of-distribution datasets (including close and far out-of-distribution datasets).
Below, we discuss some other recently proposed approaches that build upon different premises on the data available during training.

\citet{hendrycks2019selfsupervised} do not use any OOD data during training and instead teach the model to predict whether an input has been rotated by 0°, 90°, 180° or 270°. For inference, they use the loss of this predictor as an OOD score, and add this score to classifier output entropy, which behaves very similar to classifier confidence.
Similar to our methods, they also use shared representations and the combination of the in-distribution classifier with a dedicated OOD detection score. If one interprets their rotation predictor loss as being an estimator of $\log p(o|x)$ for some implicit out-distribution, their scoring function coincides with our $s_2$ scores.

\citet{golan2018deep} learn a similar transformation detector (with Dirichlet statistics collected on the in-distribution replacing ground truth labels) and use it directly to detect OOD samples without using in-distribution class information. 

\citet{winkens2020contrastive} fit for each class a normalized Mahalanobis detector on the activations of a model trained with SimCLR ands a classification head on only the in-distribution with smoothed labels. They describe their method as applying class-wise density estimation in the feature space, where the normalized Mahalanobis distance is equivalent to a Gaussian density for each class.

\citet{DermaOOD} treat an interesting application of flagging unseen skin diseases, making use of class labels that are also available for their training OOD data, which contains diseases that are different from both the in-distribution diseases and the unseen diseases.
This allows them to do fine-grained OOD detection by regarding the sum over all OOD classes which for their dataset shows large improvement over methods that treat the training out-distributuion as one class.
They gain additional slight improvements by combining this with a coarse grained binary loss that treats the sum over all in-distribution class probabilities as $p(i|x)$ and the sum over all OOD classes as $p(o|x)$.
They show that this method can be combined with various representation learning approaches in order to improve their detection of unknown diseases.

\citet{tack2020csi} introduce distribution shifting transformations into SimCLR training. Those are transformations that are so strong that the resulting samples can be considered as OOD and as negatives w.r.t. the original in the SimCLR loss.
Similarly to \citet{hendrycks2019selfsupervised} and \citet{golan2018deep}, they also train a head that classifies the applied transformation. 
In a version extended to using in-distribution labels, they consider samples from the same class and with the same transformation as positives, and samples where either is different as negatives.
With this method, they obtain OOD detection results that significantly improve over standard classifier confidence, without using any training OOD dataset.

\cite{liu2020hybrid} derive a contrastive loss from Joint Energy-Based Model training \citep{Grathwohl2020Your} and train it together with cross-entropy on the in-distribution in order to obtain classifiers whose logit values can be transformed into energies that are equivalent as OOD scoring functions to in-distribution density. They show that using these energies yields some improvements over previous density estimation approaches, and that also the classifier confidences show moderately improved OOD detection when compared to standard training.

\citet{liu2020energy} propose another Energy-Based method which incorporates surrogate OOD data during training. We analyse this method in detail in Appendix~\ref{sec:energy}, where we show that their Bayes optimal OOD detector is equivalent to the binary discriminator between in- and out-distribution.

\cite{li2020background} use the same training method as OE and show that careful resampling of the training out-distribution resembling hard negative mining can reduce its size and therefore lead to a more resource effective training OOD dataset, while the resulting models reach similarly good OOD detection performance.

\clearpage
\section{Experiments Comparing Density Estimation and Binary Discriminators}
\subsection{Unlabeled in-distribution data available for training}
We want to answer the following questions: How does estimating the in-distribution density compare to simply employing a binary discriminator between the in-distribution and uniform noise with respect to the task of out-of-distribution detection? Can other density based models be substituted for potentially easier to handle binary discriminators against a suitable (semi-)synthetic out-distribution? As generative models, we use a standard likelihood VAE, a likelihood PixelCNN++ and additionally compare with a Likelihood Regret VAE~\citep{XiaoLikelihoodRegret}. The binary classifier is trained to separate real data from uniform noise, thus none of the methods presented in this section make use of 80M or any other surrogate distribution. The results for OOD detection in terms of AUC for all methods are presented in Table \ref{table:uni}.

Comparing the OOD detection performance of the binary discriminator trained against uniform noise with both VAE models, we asses that neither model is suitable for reliably detecting inputs from unknown out-distributions. 

Following the theoretical analysis from the previous sections, the likelihood models and our binary classifier are able to perfectly separate the in-distribution data from uniform noise. This is expected as those methods are trained on that particular task of separating CIFAR-10 from uniform noise, whereas the LH Regret VAE with modified train objective has worse performance on uniform noise. It appears as if the training objective of the binary classifier seems to be too easy as the training and validation loss converge to almost zero in the first few epochs of training. However, the ability to separate uniform noise from real images does not generalize to other test distributions as both methods fail to achieve good out-of-distribution detection performance on the other test distributions. We note that while the score features from the likelihood models and the binary classifier are in expectation equivalent, both methods behave quite different on the test datasets (except for uniform noise). This is not surprising, as the probability of drawing real images from the uniform distribution is so small that neither training method properly regularizes the model's behaviour on those particular image manifolds. Thus the results are artifacts of random fluctuation and no method clearly outperforms the other one, for example the binary classifier is better at separating CIFAR-10 from SVHN whereas the likelihood VAE significantly outperforms the binary classifier on LSUN. Similar fluctuations also exist between the two variational auto encoders and PixelCNN++, but in conclusion none of those methods is able to generalize to more specialized unseen distributions.

\begin{table*}[!htbp] 
\caption{AUC for CIFAR-10 vs. various out-distributions of different methods that have access to only unlabelled CIFAR-10 data during training. 
Shown are the scores obtained from the likelihoods of the PixelCNN++ from \citet{ren2019likelihood}, 
\\
}\label{table:uni}
\setlength\tabcolsep{.5pt} 
\vskip 0.15in
\begin{center}
\begin{small}
\begin{sc}
\makebox[\textwidth][c]{
\begin{tabularx}{1.0\textwidth}{l|C|CCCCCCC|C}
       &   Mean  &    SVHN    &   LSUN    &   CelebA  &   Smooth  &      C-100 & Open  &   80M &   Uni      \\ 
Model & AUC & AUC & AUC & AUC & AUC & AUC & AUC & AUC & AUC \\
\midrule
 Lh PixelCNN++
 &  57.05
 &  07.14
 &  \best{89.02}
 &  57.77
 &  77.93
 &  52.96
 &  \best{70.52}
 &  43.99 
 &  \best{100.00}  \\
 Lh VAE 
 &  \textbf{57.97}
 &  20.98
 &  83.03
 &  48.05
 &  \best{91.67} 
 &  51.32  
 &  55.62
 &  \best{55.09}
 &  \best{100.00}  \\
 Lh Regret VAE 
 &  52.24
 & \best{87.36}
 &  35.73
 &  \best{70.69}
 &  14.84
 &  \best{53.03} 
 &  50.81
 &  53.21
 &  94.11   \\ 
BinDisc Uniform 
 &  45.90
 &  71.22  
 &  34.67  
 &  35.07  
 &  47.13  
  &  47.32  
 &  44.70  
 &  41.17 
 &  \best{100.00}    \\
\end{tabularx}
}
\end{sc}
\end{small}
\end{center}
\vskip -0.1in
\end{table*}

\subsection{Likelihood Ratios as a binary discriminator}
In Section \ref{sec:bayes_unlabeled}, we discussed that for the Bayes optimal solutions of their training objectives, the ratio of the likelihoods of two density estimators for different distributions is as an OOD detection scoring function equivalent to the prediction of a binary discriminator between the two distributions.
In order to find out which role this equivalence plays in practise, we  train a binary discriminator between CIFAR-10 as in-distribution and the background distribution obtained by mutating 10\% of the pixles of in-distribution images as described in~\citet{ren2019likelihood}. In Table~\ref{table:lhratio}, we compare the OOD detection performance of this discriminator with likelihood ratios estimated with PixelCNN++~\citep{Salimans2017PixeCNN} as trained with the code of~\citet{ren2019likelihood} setting $L_2$ regularization as 10, and with the numbers taken from~\citet{XiaoLikelihoodRegret} given for their VAE models. Even though we use the code and hyperparameter settings of \citet{ren2019likelihood}, the AUC we obtain for SVHN as out-distribution differs significantly from their reported 93.1\%.
We observe that all three methods struggle with detecting inputs from several out-distributions and thus we do not consider them as reliable out-of-distribution detection methods.

\begin{table*}[!htbp] 
\caption{AUROC for CIFAR-10 vs. various out-distributions. The Likelihood Ratio~\cite{ren2019likelihood}  PixelCNN++ models were trained with their code. For the VAE, we cite the numbers of \cite{XiaoLikelihoodRegret}, as indicated by~*. 
We also show a binary discriminator between CIFAR-10 and the background distribution of \cite{ren2019likelihood}.
}\label{table:lhratio}
\setlength\tabcolsep{.5pt} 
\vskip 0.15in
\begin{center}
\begin{small}
\begin{sc}
\makebox[\textwidth][c]{
\begin{tabularx}{1.0\textwidth}{l|C|CCCCCCCC}
Model       &   Mean  &    SVHN    &   LSUN    &   Uni &   Smooth  &      C-100  &  Open  &  CelebA  &80M      \\ \midrule
  LhRatio PixelCNN++
 &  \textbf{64.77}
 &  12.57
 &  \best{88.85}
 &  \best{100.00}
 &  \best{89.55}
 &  \best{53.60}
 &  \best{74.50}
 &  \best{58.26}
 &  40.87   \\
 LhRatio VAE* 
 &
 &  26.5*  
 &  63.2* 
 &  \best{100.0*}  
 &   
 &  
 &  
 &  44.7* 
 &     \\
BinDisc
 &  58.81
 &  \best{55.07}  
 &  66.10  
 &  \best{100.00}  
 &  44.43  
 &  47.05  
 &  57.20  
 &  48.48  
 &  \best{52.16}    \\
\end{tabularx}
}
\end{sc}
\end{small}
\end{center}
\vskip -0.1in
\end{table*}


\section{Energy-Based OOD Detection in terms of Binary Discriminators} \label{sec:energy}
With Energy-Based OOD Detection~\citep{liu2020energy}, we exhibit a surprising case of the equivalence of Binary Discriminators to a further OOD detection method which is based on ideas quite different to those of the other extensively discussed methods.
This method is based on the premise that the logits $f_\theta(x)$ of a classifier can be used to define an energy (here we ignore a potential temperature factor which as~\citet{liu2020energy} find can in good conscience be set to one)
\begin{align}
    E_\theta(x) = - \log \sum_{l=1}^K e^{f_\theta(l|x)}
\end{align}
which the model assigns to the input, as has been proposed by \citet{lecun2006tutorial, Grathwohl2020Your}.
Ideally, this energy would be equivalent to the probability density of the in-distribution via
\begin{align}
    \log p(x|i) = -E_\theta(x) - Z \ \ \text{\ \  where } Z = \int_{z \in [0,1]^D} e^{-E_\theta(z)} dz \ .
\end{align}
However, since the integral over the whole image domain is intractible, it is not possible to effectively decrease $-E_\theta$ on the in-distribution directly while also controlling $Z$.
Naïvely training the classifier to have high energy on random inputs, i.e. uniform noise, is of course not a solution, since the model easily distinguishes the noise, and it is very unlikely to encounter any at least vaguely real images within finite training time.
Thus, \citet{liu2020energy} rather use a surrogate training out-distribution of natural images which they increase the energy on during training; for their experiments, they take 80 Million Tiny Images, for which we compare their models with several other methods in Appendix~\ref{section:80M}.
Simultaneously, they minimize the standard classifier cross-entropy on the in-distribution.
In order to avoid infinitely small potential losses, their training objective uses two margin hyper-parameters $m_{\text{in}} < m_{\text{out}}$ and reads
\begin{align}
\begin{split}
    &\E_{(x,y) \sim p(x,y|i)} -\log p_\theta(y|x) \\
    & + \lambda \cdot \left( 
        \E_{(x) \sim p(x|i)} \max\left\{0, E_\theta(x) - m_{\text{in}}\right\}^2
        + \E_{w \sim p(w|o)} \max\left\{0, m_{\text{out}} - E_\theta(w)\right\}^2
    \right)
\end{split}
\end{align}
with
\begin{align}
    \log p_\theta(y|x) 
    = \log\frac{e^{f_\theta(y|x)}}{\sum_{l=1}^K e^{f_\theta(l|x)}}
    =  f_\theta(y|x) + E_\theta(x) \ .
\end{align}
Note that their $\lambda$ does not balance between in and out by depending on a prior on $p(i)$, but it balances the energy regularization compared to the classifier cross-entropy loss.

\EnergyOptim*

\begin{proof}
Our goal is to find for a given input $x$ the model output $f_\theta(x)$ that minimizes the expected loss, assuming that we know the probabilities $p(y|x,i)$ and $p(i|x)$. 
This expected loss is
\begin{align}
    -&p(i|x) \cdot \sum_{k=1}^{K} p(k|x,i) \log p_\theta(k|x) \\
    &+ \lambda \cdot \left( p(i|x) \max\left\{0, E_\theta(x) - m_{\text{in}}\right\}^2
    + (1-p(i|x)) \max\left\{0, m_{\text{out}} - E_\theta(x)\right\}^2 \right) \ .
\end{align}
First, we note that if a certain $E^{*}_\theta(x)$ minimizes the expected energy loss to
\begin{align}
    \lambda \cdot \left( p(i|x) \max\left\{0, E^{*}_\theta(x) - m_{\text{in}}\right\}^2 +  (1-p(i|x)) \max\left\{0, m_{\text{out}} - E^{*}_\theta(x)\right\}^2 \right) \ ,
\end{align}
and if some  $f^{\sharp}_\theta(x)$ minimizes the expected CE loss (note that its minimization is independent of the positive factor $p(i|x)$) to
\begin{align}
    -\sum_{k=1}^{K} p(k|x,i) \log\frac{e^{f^{\sharp}_\theta(k|x)}}{\sum_{l=1}^K e^{f^{\sharp}_\theta(l|x)}} \ ,
\end{align}
with corresponding $E^{\sharp}_\theta(x) = -\log \sum_{i=1}^K e^{f^{\sharp}_\theta(x)[i]}$, then the  logit output with $- E^{*}_\theta(x) + E^{\sharp}_\theta(x)$ added to each component
has the energy
\begin{align}
    E_\theta(x) 
    &= - \log \sum_{i=1}^K e^{f^{\sharp}_\theta(x)[i] - E^{*}_\theta(x) + E^{\sharp}_\theta(x)}
    = - \log \left(e^{-E^{*}_\theta(x) + E^{\sharp}_\theta(x)}\cdot\sum_{i=1}^K e^{f^{\sharp}_\theta(x)[i]}\right) \\
    &= E^{*}_\theta(x) - E^{\sharp}_\theta(x) - \log \sum_{i=1}^K e^{f^{\sharp}_\theta(x)[i]}
    = E^{*}_\theta(x) \ ,
\end{align}
which means that this $f^{*}(x)$ minimizes the expected energy loss, and also fulfills
\begin{align}
    p_\theta(x)[y] 
    &= \frac{e^{f^{\sharp}_\theta(x)[y] - E^{*}_\theta(x) + E^{\sharp}_\theta(x)}}{\sum_{l=1}^K e^{f^{\sharp}_\theta(x)[l] - E^{*}_\theta(x) + E^{\sharp}_\theta(x)}}
    = \frac{e^{f^{\sharp}_\theta(x)[y]}}{\sum_{l=1}^K e^{f^{\sharp}_\theta(x)[l]}} \ ,
\end{align}
i.e. has the same probability predictions as $f^{\sharp}$ and thus also minimizes the expected CE loss.
This means that both can be optimized independently.

As we've seen before, the optimal $p_\theta(k|x)$ for the expected CE loss is $p_\theta(k|x) = p(k|x,i)$.

Thus we independently minimize the expected energy loss 
\begin{align}
    p(i|x) \max\left\{0, E_\theta(x) - m_{\text{in}}\right\}^2 
    +  (1-p(i|x)) \max\left\{0, m_{\text{out}} - E_\theta(x)\right\}^2 \ .
\end{align}
The first derivative with respect to $E_\theta(x)$ is
\begin{align}
    \begin{cases}
        2\cdot(1-p(i|x)) \cdot (E_\theta(x) - m_{\text{out}})  &\text{ for } E_\theta(x) < m_{\text{in}} \ , \\
        2\cdot \left( p(i|x) \cdot(E_\theta(x) - m_{\text{in}}) + (1-p(i|x)) \cdot (E_\theta(x) - m_{\text{out}})\right)  &\text{ for }  m_{\text{in}} < E_\theta(x) < m_{\text{out}} \ , \\
        2\cdot p(i|x) \cdot(E_\theta(x) - m_{\text{in}}) &\text{ for }  m_{\text{out}} < E_\theta(x) \ . \\
    \end{cases}
\end{align}
which simplified is
\begin{align}
    \begin{cases}
        < 0  &\text{ for } E_\theta(x) < m_{\text{in}} \ , \\
       2\cdot\left(E_\theta(x) + p(i|x) \cdot (m_{\text{out}} - m_{\text{in}}) - m_{\text{out}}\right)  &\text{ for }  m_{\text{in}} < E_\theta(x) < m_{\text{out}} \ , \\
       > 0 &\text{ for }  m_{\text{out}} < E_\theta(x) \ . \\
    \end{cases}
\end{align}
Here, for simplicity we again make the reasonable assumption that in- and out-distribution have full support, i.e. $0 < p(i|x) < 1$. (If we do not want to make the full support assumptions, for $p(i|x) = 0$, any energy $\geq m_{\text{out}}$ would be optimal, and for $p(i|x) = 1$, any energy $\leq m_{\text{in}}$ would be optimal.)

A the margin $E_\theta(x) = m_{\text{in}}$, 
the derivatives for $E_\theta(x) < m_{\text{in}}$ and $m_{\text{in}} < E_\theta(x) < m_{\text{out}}$ coincide  as $2\cdot(1-p(i|x)) \cdot (m_{\text{in}} - m_{\text{out}}) < 0$, which means that the minimizer is $> m_{\text{in}}$. Similarly, we see that the minimizer is $< m_{\text{out}}$.

As the second derivative is positive, by solving where the derivative is zero, we find that the optimal negative energy (which is the score which they use for OOD detection) is
\begin{align}
    -E^*_\theta(x) =  p(i|x) \cdot (m_{\text{out}} - m_{\text{in}}) - m_{\text{out}} \ . 
\end{align}
\end{proof}

This is a strictly monotonously increasing function in $p(i|x)$.
By our Theorem 1, we conclude that the negative energy score they obtain from their method is (for the Bayes optimal model that minimizes the loss on the training distributions) equivalent to $p(i|x)$.
\EnergyOptimCor*
Note that while the classifier and energy loss terms \emph{can} be optimized independently, the model is trained for both tasks simultaneously, which means that similar synergies to those we observe in \ref{table:separate_OI_FPR} might contribute to the good performance of this method.
While we observe slightly worse results when transferring the method to using AutoAugment and OpenImages for fine-tuning in \ref{section:experiments}, the empirical evaluation of Energy-Based models in Appendix~\ref{section:80M} confirms that their similar behaviour to the binary discriminator between in- and training-out-distribution indeed holds in practice.

\clearpage
\section{Evaluations with the AUC Detection Metric}\label{sec:auc}
Complementing the evaluations of the FPR@95TPR metric presented in the main paper, Tables~\ref{table:OI_auc} and~\ref{table:sep_OI_AUC} show the AUC (AUROC) values of the same models as in Tables~\ref{table:OI_AA_FPR} and~\ref{table:separate_OI_FPR}, respectively. The observations on the strengths of the different methods and scoring functions discussed in Section~\ref{section:experiments} also hold for the AUC evaluations, which is not surprising since the AUC and FPR@95TPR metrics are closely related.

\begin{table*}[!htbp]
\caption{Accuracy on the in-distribution (CIFAR-10/CIFAR-100) and \textbf{AUC} for various test out-distributions of the different OOD methods with OpenImages as training out-distribution for which the FPRs are shown in Table~\ref{table:OI_AA_FPR}. 
}\label{table:OI_auc}
\setlength\tabcolsep{.5pt} 
\vskip 0.15in
\begin{center}
\begin{small}
\begin{sc}
\makebox[\textwidth][c]{
\begin{tabularx}{1.0\textwidth}{lC|C|CCCCCCC|C}
\multicolumn{10}{c}{in-distribution: CIFAR-10} \\
\midrule
       & &  Mean &     SVHN    &   LSUN    &   Uni &   Smooth     &   C-100 &  80M &   CelA  &    OpenIm     \\ 
Model    & Acc. &  AUC &    AUC    &   AUC    &   AUC &  AUC    &   AUC & AUC &   AUC &   AUC      \\
\midrule
Plain Classi
 &  95.16
 &  91.85    
 &  93.52  
 &  92.94  
 &  97.04  
 &  92.84  
 &  89.61  
 &  91.30  
 &  85.70  
 &  84.81   \\
\midrule
Mahalanobis
 &
 & 91.63
 & 96.34
 & 92.39
 & \best{100.00}
 & 99.82
 & 86.78
 & 89.41
 & 76.67
 & 84.81
\\
\midrule
Energy  ($s_1$)
 & 94.13
 & 93.20
 & 96.85
 & 99.99
 & 99.93
 & 99.59
 & 76.90
 & 79.30
 & 99.87
 & 99.52
\\
\midrule
OE ($s_3$)
 &  95.06
 &  \textbf{97.28}   
 &  98.49  
 &  99.99  
 &  99.99
 &  99.99  
 &  90.03  
 &  92.53  
 &  \best{99.91}  
 &  99.43   \\
\midrule
  \hspace{\subtab} \backgroundSone
 &  
 &  95.02    
 &  \best{99.48}  
 &  \best{100.00}  
 &  99.99 
 &  99.95  
 &  79.64  
 &  86.37  
 &  99.74  
 &  \best{99.97}   \\
  \hspace{\subtab} \backgroundStwo
 &  95.21
 &  97.22    
 &  98.90  
 &  \best{100.00}  
 &  99.99  
 &  99.67  
 &  \best{90.47}  
 &  92.41  
 &  99.11  
 &  99.73   \\
\backgroundSthree
 &  95.21
 &  97.21    
 &  98.87  
 &  \best{100.00}  
 &  99.98  
 &  99.62  
 &  \best{90.47}  
 &  92.41  
 &  99.08  
 &  99.71   \\
\midrule
\hspace{\subtab} Shared BinDisc ($s_1$)
 &  
 &  92.51    
 &  98.77  
 &  \best{100.00}  
 &  99.89  
 &  99.93  
 &  68.34  
 &  80.81  
 &  99.80  
 &  99.95   \\
\hspace{\subtab} Shared Classi 
 &  \textbf{95.28}
 &  95.49    
 &  96.10  
 &  98.60  
 &  99.06  
 &  96.09  
 &  90.09  
 &  92.35  
 &  96.18  
 &  93.57   \\
Shared Combi $s_2$
 &  \textbf{95.28}
 &  97.26    
 &  98.66  
 &  \best{100.00}  
 &  99.93  
 &  99.94  
 &  89.71  
 &  92.84  
 &  99.72  
 &  99.88   \\
Shared Combi $s_3$
 &  \textbf{95.28}
 &  97.26    
 &  98.62  
 &  \best{100.00}  
 &  99.93  
 &  99.94  
 &  89.75  
 &  \best{92.85}  
 &  99.71  
 &  99.88   \\ \toprule

\multicolumn{10}{c}{in-distribution: CIFAR-100} \\
\midrule
       & &  Mean &     SVHN    &   LSUN    &   Uni &   Smooth     &   C-10 & 80M  &    &   OpenIm      \\ 
Model    & Acc. &  AUC &    AUC    &   AUC    &   AUC &  AUC    &   AUC & AUC &    &   AUC      \\ \midrule
Plain Classi
 &  77.16
 &  82.13    
 &  82.33  
 &  79.13  
 &  96.03  
 &  81.36  
 &  \best{76.14}  
 &  77.80  
 &  \phantom{0000}  
 &  75.80   \\
\midrule
Mahalanobis
 &
 & 85.88
 & 88.69
 & 87.56
 & 90.08
 & \best{99.91}
 & 70.78
 & \best{78.26}
 &
 & 77.51
\\
\midrule
Energy  ($s_1$)
 & 73.47
 & 88.67
 & 92.55
 & 99.96
 & 99.40
 & 98.71
 & 71.08
 & 70.33
 & 
 & 99.44
\\ \midrule
OE      ($s_3$)
 &  77.19
 &  90.37    
 &  89.54  
 &  99.98  
 &  99.03  
 &  99.68  
 &  75.95  
 &  78.03  
 &  \phantom{0000}  
 &  99.67   \\
\midrule
  \hspace{\subtab} \backgroundSone
 &   
 &  88.41    
 &  97.38  
 &  \best{99.99}  
 &  99.70  
 &  99.79  
 &  60.51  
 &  73.11  
 &  \phantom{0000}  
 &  \best{99.93}   \\
  \hspace{\subtab} \backgroundStwo
 &  \textbf{77.61}
 &  90.47    
 &  90.50  
 &  \best{99.99}  
 &  99.87  
 &  99.75  
 &  74.88  
 &  77.82  
 &  \phantom{0000}  
 &  99.64   \\
\backgroundSthree
 &  \textbf{77.61}
 &  90.46    
 &  90.46  
 &  \best{99.99}  
 &   \best{99.88}
 &  99.74  
 &  74.88  
 &  77.82  
 &  \phantom{0000}  
 &  99.64   \\
 \midrule
\hspace{\subtab} Shared BinDisc ($s_1$)
 &   
 &  84.62    
 &  \best{97.44}  
 &  \best{99.99}  
 &  99.70  
 &  99.68  
 &  47.82  
 &  63.13  
 &  \phantom{0000}  
 &  \best{99.93}   \\
\hspace{\subtab} Shared Classi 
 &  77.35
 &  82.06    
 &  82.72  
 &  99.05  
 &  72.73  
 &  84.14  
 &  75.76  
 &  77.99  
 &  \phantom{0000}  
 &  93.54   \\
Shared Combi $s_2$
 &  77.35
 &  \textbf{90.74}    
 &  91.74  
 &  \best{99.99}  
 &  99.59  
 &  99.54  
 &  75.50  
 &  78.10  
 &  \phantom{0000}  
 &  99.57   \\
Shared Combi $s_3$
 &  77.35
 &  90.73    
 &  91.69  
 &  \best{99.99}  
 &  99.57  
 &  99.53  
 &  75.50  
 &  78.10  
 &  \phantom{0000}  
 &  99.57   \\
 \bottomrule
\end{tabularx}
}
\end{sc}
\end{small}
\end{center}
\vskip -0.1in
\end{table*}

\begin{table*}[!h]
\caption{
\textbf{AUC} evaluation of the models trained with shared and separate representations from Table~\ref{table:separate_OI_FPR}.
\label{table:sep_OI_AUC}
}
\setlength\tabcolsep{.5pt} 
\vskip 0.15in
\begin{center}
\begin{small}
\begin{sc}
\makebox[\textwidth][c]{
\begin{tabularx}{1.0\textwidth}{lC|C|CCCCCCC|C}
\multicolumn{10}{c}{in-distribution: CIFAR-10} \\
\midrule
       & &  Mean &     SVHN    &   LSUN    &   Uni &   Smooth     &   C-100 &  80M &   CelA  &    OpenIm     \\ 
Model    & Acc. &  AUC &    AUC    &   AUC    &   AUC &  AUC    &   AUC & AUC &   AUC &   AUC      \\ \midrule
 \hspace{\subtab} Plain Classi
 &  95.16
 &  91.85    
 &  93.52  
 &  92.94  
 &  97.04  
 &  92.84  
 &  89.61  
 &  91.30  
 &  85.70  
 &  84.81   \\
  \hspace{\subtab} Separate BinDisc  ($s_1$)
 &   
 &  89.03    
 &  96.42  
 &  \best{100.00}  
 &  99.97  
 &  \best{99.99}  
 &  58.60  
 &  72.36  
 &  95.87  
 &  \best{99.99}   \\
  Separate Combi $s_3$
 & 95.16
 & 96.40
 &  98.16  
 &  \best{100.00}  
 &  \best{99.98}  
 &  \best{99.99}  
 &  89.64  
 &  91.35  
 &  95.70  
 &  99.94   \\
\midrule
\hspace{\subtab} Shared BinDisc ($s_1$)
 &  
 &  92.51    
 &  \best{98.77}  
 &  \best{100.00}  
 &  99.89  
 &  99.93  
 &  68.34  
 &  80.81  
 &  \best{99.80}  
 &  99.95   \\
\hspace{\subtab} Shared Classi 
 &  \textbf{95.28}
 &  95.49    
 &  96.10  
 &  98.60  
 &  99.06  
 &  96.09  
 &  \best{90.09}  
 &  92.35  
 &  96.18  
 &  93.57   \\
Shared Combi $s_3$
 &  \textbf{95.28}
 &  \textbf{97.26}    
 &  98.62  
 &  \best{100.00}  
 &  99.93  
 &  99.94  
 &  89.75  
 &  92.85  
 &  99.71  
 &  99.88   \\ \midrule
  Plain $\otimes$ Sha Disc $s_3$
 & 95.16
 & \textbf{97.26}
 & 98.67
 & \best{100.00}
 & 99.91
 & 99.93
 & 89.78
 & \best{92.95}
 & 99.58
 & 99.87 \\
\toprule \\
\multicolumn{10}{c}{in-distribution: CIFAR-100} \\
\midrule
Model    &  Acc.   &   Mean  &  SVHN    &   LSUN      &   Uni &   Smooth      &   C-10 & OpenIm &   \textcolor{white}{CelA} &   80M      \\ \midrule
 \hspace{\subtab} Plain Classi
 &  77.16
 &  82.13    
 &  82.33  
 &  79.13  
 &  96.03  
 &  81.36  
 &  \best{76.14}  
 &  77.80  
 &  \phantom{0000}
 &  75.80   \\
  \hspace{\subtab} Separate BinDisc ($s_1$)
 &   
 &  84.30    
 &  94.68  
 &  \best{100.00}  
 &  99.81  
 &  99.64  
 &  50.06  
 &  61.62  
 &   
 &  \best{99.98}   \\
  Separate Combi $s_3$
 &  77.16
 &  88.95
 &  84.49  
 &  99.99  
 &  \best{99.88}  
 &  96.06  
 &  75.83  
 &  77.45  
 &    
 &  99.82   \\ \midrule
\hspace{\subtab} Shared BinDisc ($s_1$)
 &   
 &  84.62    
 &  \best{97.44}  
 &  99.99  
 &  99.70  
 &  \best{99.68}  
 &  47.82  
 &  63.13  
 &  \phantom{0000}
 &  99.93   \\
\hspace{\subtab} Shared Classi 
 &  \textbf{77.35}
 &  82.06    
 &  82.72  
 &  99.05  
 &  72.73  
 &  84.14  
 &  75.76  
 &  77.99  
 &  \phantom{0000}
 &  93.54   \\
Shared Combi $s_3$
 &  \textbf{77.35}
 &  90.73    
 &  91.69  
 &  99.99  
 &  99.57  
 &  99.53  
 &  75.50  
 &  78.10  
 &  \phantom{0000}
 &  99.57   \\ \midrule
  Plain $\otimes$ Sha Disc $s_3$
 & 77.16
 & \textbf{90.79}
 & 91.99
 & 99.99
 & 99.80
 & 98.85
 & 75.87
 & \best{78.23}
 &
 & 99.50 \\
 \bottomrule
\end{tabularx}
}
\end{sc}
\end{small}
\end{center}
\vskip -0.1in
\end{table*}

\clearpage

\section{Discussion of the Choice of OpenImages and  Experimental Results with 80 Million Tiny Images as Training Out-Distribution}\label{section:80M}
Since the 80 Million Tiny Images~\citep{torralba200880} dataset was retracted by the authors -- their statement can be read at \url{http://groups.csail.mit.edu/vision/TinyImages/}\footnote{Archived statement: \url{https://web.archive.org/web/20210415160225/http://groups.csail.mit.edu/vision/TinyImages/}} --
as a reaction to~\citet{Birhane_2021_WACV} exposing the presence of offensive and prejudicial images in the dataset, a good "surrogate surrogate" training out-distribution for the CIFAR-10/CIFAR-100 in-distribution has to our knowledge not yet been established.

Our experience confirms the assessment which the authors of \cite{HenMazDie2019} make in their discussion section 5: the surrogate training out-distribution should consist mainly of natural images, should have a high semantic diversity, and the number of samples in the dataset should be large.
We also observe that it is vital that there are no easy to detect details that separate the training in- and out-distributions from each other, as for example using a different resizing interpolation method would lead to 'overfitting' on such features with 100\% train accuracy of the Binary Discriminator and near zero loss (apart from that of the in-distribution classifier) for OE and BGC, with no generalization to the test OOD datasets.
OpenImages fulfills the mentioned criteria, and in our judgement does not contain ethically problematic images.
It is, however, to be noted that CIFAR was sourced as a subset of 80M Tiny Images, see \citet{krizhevsky2009learning} and \url{https://www.cs.toronto.edu/~kriz/cifar.html}, which explains the somewhat better results with 80M as training OOD dataset that we observe below in this section.
Our results on the theoretical relations between the different methods and scoring functions do not depend on the choice of the training OOD dataset , and give reason to expect that our experimental confirmations of these relations will also hold for even better suited surrogate OOD datasets that we hope will be found in future works.

In the main paper, we employ OpenImages~\citep{OpenImages2} as a replacement, and for completeness and comparison to the originals of OE, Energy-Based OOD detection and NTOM/ATOM, below we show and discuss the results obtained with the retracted 80M dataset which is commonly used in OOD detection literature.

The training procedure of our methods again follows that of OE, and for the 80M experiments we do not add AutoAugment, in order to stay as close as possible to the original.
For the established methods we compare to, we use the original weights of their published models; the Plain and Outlier Exposure (OE) models were retrieved from the repository of the authors of OE (\url{https://github.com/hendrycks/outlier-exposure}), and for NTOM and ATOM~\citep{chen2020informative-outlier-matters} we use their code \url{https://github.com/jfc43/informative-outlier-mining} to evaluate their DenseNet models (their best models). We finetune the Mahalanobis detector on 80M with the same procedure as described above; the optimal input noise level for 80M is 0.0005 for both CIFAR in-distributions.
We evaluate the Energy-Based models fine-tuned on 80M which the authors of~\citet{liu2020energy} provide at \url{https://github.com/wetliu/energy_ood} with their evaluation code.

For CIFAR-10, as already seen with OpenImages as training out-distribution, the OOD detection performance of \textsc{Shared Classi} 
is much better than that of the plain classifier.
.In fact \textsc{Shared BinDisc} has already very good OOD performance with a mean FPR of 7.56 and mean AUC of 97.90  which is only improved by considering scoring function $s_2/s_3$ in the combination of  \textsc{Shared BinDisc} and \textsc{Shared Classi}
which yields the best performance in classification accuracy and mean AUC.

The classifier with 80M as background class (\textsc{BGC}) works very well for all scoring functions and reaches SOTA performance similar to/better than OE. Both \textsc{BGC}/\textsc{Shared Combi} with $s_2/s_3$ perform particularly well on the challenging close out-distributions CIFAR-100 and OpenImages (which are the data sets where NTOM/ATOM perform significantly worse).
However, as already observed with OpenImages as training out-distribution the differences of the methods are minor both in terms of classification accuracy.

The results for CIFAR-100 are again qualitatively similar to those for CIFAR-10. NTOM/ATOM now show worse mean AUC results which are mainly due to worse results for the close out-distributions CIFAR-10 and OpenImages, but better mean FPR@95\%TPR, which is explained by their excellent detection of LSUN Classroom, where the other methods work quite well in terms of AUC but still make quite many errors at the 95\%TPR threshold. OE again achieves comparable OOD results to the other evaluated methods. 
As the theoretical considerations we presented in Appendix~\ref{sec:energy} suggest, the Energy-Based OOD detector achieves good performance that is comparable that to the other methods.
Our \textsc{Shared Combi} $s_2/s_3$ performs best in terms of OOD performance and test accuracy but again differences are minor.

The conclusions are similar to those drawn for OpenImages in the main paper.
Comparing the results, it is clear that 80M still works somewhat better than OpenImages, so for the CIFAR in-distributions, the search for an ethically acceptable replacement of 80M that allows for equal or improved results continues. 
The consistent similarities between the examined methods over the different datasets suggest that the methods and scoring functions would be similarly viable with such an alternative training OOD dataset.

\begin{table*}[!htbp]
\caption{Accuracy on the in-distribution (CIFAR-10/CIFAR-100) and 
\textbf{FPR@95\%TPR} for various test out-distributions of different OOD methods with \textbf{80 Million Tiny Images} as training out-distribution (shown results for test set of 80M are not used for computing the mean FPR). Lower false positive rate is better. CelebA makes no sense as test out-distribution for CIFAR-100 as it contains man/woman as classes. \textsc{Plain}, \textsc{OE}, \textsc{BGC} and \textsc{Shared} have been trained using the same architecture and training parameters/schedule. $s_1,s_2,s_3$ are the scoring functions introduced in Section \ref{section:OE}. Our binary discriminator (\textsc{BinDisc}) resp. the combination with the shared classifier (\textsc{Shared Combi}) performs similar/better than Outlier Exposure \citep{HenMazDie2019}.
}\label{table:fpr_80M}
\setlength\tabcolsep{.5pt} 
\vskip 0.15in
\begin{center}
\begin{small}
\begin{sc}
\makebox[\textwidth][c]{
\begin{tabularx}{1.0\textwidth}{lC|C|CCCCCCC|C}
\multicolumn{10}{c}{in-distribution: CIFAR-10} \\
\midrule
       & &  Mean &     SVHN    &   LSUN    &   Uni &   Smooth     &   C-100 & OpenIm  &   CelA  &   80M      \\ 
Model    & Acc. &  FPR &    FPR    &   FPR    &   FPR &  FPR    &   FPR & FPR &   FPR &   FPR      \\ \midrule
Plain Classi
 &  94.84
 &  64.62  
 &  48.33  
 &  52.67  
 &  75.69  
 &  62.58  
 &  62.91  
 &  66.38  
 &  83.79  
 &  60.53 \\ \midrule
 Mahalanobis
 &
 & 41.15
 & 40.19
 & 50.00
 & \best{\phantom{0}0.00}
 & 0.17
 & 58.66
 & 58.36 
 & 80.69
 & 53.42
 \\ \midrule
 Energy ($s_1$)
 & 95.22
 & 9.01
 & 1.58
 & 2.00
 & \best{\phantom{0}0.00}
 & \best{\phantom{0}0.00}
 & 30.03
 & 28.26
 & 1.19
 & 8.21
 \\ \midrule
NTOM
 &  95.42 
 &  \phantom{0}8.21 
 &  \phantom{0}1.06 
 &  \best{\phantom{0}0.33 }
 &  \best{\phantom{0}0.00}  
 &  \best{\phantom{0}0.00}  
 &  29.61   
 &  26.15
 &  \phantom{0}0.30  
 &  \phantom{0}4.90   \\
 ATOM
 &  95.20 
 &  \phantom{0}7.76  
 &  \phantom{0}0.69
 &  \best{\phantom{0}0.33 }
 &  \best{\phantom{0}0.00}  
 &  \best{\phantom{0}0.00}  
 &  27.80  
 &  25.26  
 &  \phantom{0}0.25  
 &  \phantom{0}4.44    \\ \midrule
 OE  ($s_3$)
 &  95.74 
 &  \phantom{0}8.27  
 &  \phantom{0}1.96  
 &  \phantom{0}2.00  
 &  \best{\phantom{0}0.00}  
 &  \phantom{0}0.06  
 &  26.12  
 &  27.07  
 &  \phantom{0}0.71  
 &  \phantom{0}5.96   \\
 \midrule
  \hspace{\subtab} \backgroundSone
 &   
 &  \phantom{0}7.47    
 &  \phantom{0}0.83  
 &  \phantom{0}1.33  
 &  \best{\phantom{0}0.00}  
 &  \best{\phantom{0}0.00}  
 &  24.75  
 &  \best{25.19}  
 &  \best{\phantom{0}0.19}  
 &  \best{\phantom{0}4.43}   \\
  \hspace{\subtab} \backgroundStwo
 &  95.63 
 &  \phantom{0}\textbf{7.42} 
 &  \phantom{0}0.98  
 &  \phantom{0}1.33  
 &  \best{\phantom{0}0.00}  
 &  \best{\phantom{0}0.00}  
 &  24.13  
 &  25.33  
 &  \phantom{0}0.20  
 &  \phantom{0}4.95   \\
\backgroundSthree
 &  95.63 
 &  \phantom{0}7.49    
 &  \phantom{0}1.05  
 &  \phantom{0}1.33  
 &  \best{\phantom{0}0.00}  
 &  \best{\phantom{0}0.00}  
 &  24.26  
 &  25.57  
 &  \phantom{0}0.21  
 &  \phantom{0}4.82   \\
 \midrule
\hspace{\subtab} Shared BinDisc ($s_1$)
 &   
 &  \phantom{0}7.56  
 &  \phantom{0}\best{0.67}  
 &  \phantom{0}1.67  
 &  \best{\phantom{0}0.00}  
 &  \best{\phantom{0}0.00}  
 &  24.70  
 &  25.58  
 &  \phantom{0}0.31  
 &  \phantom{0}4.57   \\
\hspace{\subtab} Shared Classi
 &  \textbf{96.08}
 &  15.71  
 &  \phantom{0}6.29  
 &  13.00  
 &  \phantom{0}0.07  
 &  \phantom{0}0.13  
 &  37.07  
 &  40.48  
 &  12.95  
 &  19.47   \\
Shared Combi $s_2$
 &  \textbf{96.08} 
 &  \phantom{0}7.47  
 &  \phantom{0}0.71  
 &  \phantom{0}1.33  
 &  \best{\phantom{0}0.00}  
 &  \best{\phantom{0}0.00}  
 &  24.15  
 &  25.72  
 &  \phantom{0}0.35  
 &  \phantom{0}4.79   \\
Shared Combi $s_3$
 &  \textbf{96.08 }
 &  \phantom{0}7.44  
 &  \phantom{0}0.73  
 &  \phantom{0}1.33  
 &  \best{\phantom{0}0.00}  
 &  \best{\phantom{0}0.00}  
 &  \best{23.95}  
 &  25.70  
 &  \phantom{0}0.35  
 &  \phantom{0}4.84   \\
\toprule
\\
\multicolumn{10}{c}{in-distribution: CIFAR-100} \\
\midrule
       & &  Mean &     SVHN    &   LSUN    &   Uni &   Smooth     &   C-10 & OpenIm  &    &   80M      \\ 
Model    & Acc. &  FPR &    FPR    &   FPR    &   FPR &  FPR    &   FPR & FPR &    &   FPR      \\ 
\midrule
Plain Classi
 &  75.96
 &  82.26  
 &  84.33  
 &  80.00  
 &  98.99  
 &  65.81  
 &  81.97  
 &  82.47  
 &   
 &  80.17   \\ \midrule
 Mahalanobis
 &
 & 47.89
 & 64.58
 & 63.67
 & \best{\phantom{0}0.00}
 & 2.77
 & 81.39
 & 74.93
 &
 & 69.79
 \\ \midrule
 Energy ($s_1$)
 & 75.70
 & 32.95
 & \best{20.61}
 & 16.67
 & 4.23
 & 2.90
 & 84.27
 & 69.00
 & 
 & 42.18
 \\ \midrule
 NTOM
 &  74.88 
 &  \textbf{32.63}
 &  24.67
 &  10.00 
 &  \best{\phantom{0}0.00}  
 &  \best{\phantom{0}0.00} 
 &  90.58 
 &  70.52
 &    
 &   40.78  \\
 ATOM
 &  75.06
 &  34.60 
 &  37.78  
 &  \best{8.67}  
 &  \best{\phantom{0}0.00}  
 &  \phantom{0}0.30  
 &  89.80  
 &  71.02  
 &       
 &  \best{40.29}   \\ \midrule
 OE  ($s_3$)
 &  \textbf{76.73 }
 &  34.89  
 &  34.41  
 &  24.00  
 &  \phantom{0}1.10  
 &  \phantom{0}4.96  
 &  79.77  
 &  \best{65.09}  
 &   
 &  45.59   \\
 \midrule
  \hspace{\subtab} \backgroundSone
 &   
 &  34.79    
 &  35.73  
 &  23.00  
 &  \best{\phantom{0}0.00}  
 &  \phantom{0}0.07  
 &  81.61  
 &  68.31  
 &  \phantom{0000}
 &  45.76   \\
  \hspace{\subtab} \backgroundStwo
 &  75.82 
 &  35.86  
 &  40.36  
 &  26.67  
 &  \best{\phantom{0}0.00}  
 &  \phantom{0}0.45  
 &  \best{79.50}  
 &  68.21  
 &   
 &  47.72   \\
\backgroundSthree
 &  75.82
 &  35.91    
 &  40.53  
 &  26.67  
 &  \best{\phantom{0}0.00}  
 &  \phantom{0}0.42  
 &  79.54  
 &  68.28  
 &  \phantom{0000} 
 &  47.46   \\
 \midrule
\hspace{\subtab} Shared BinDisc ($s_1$)
 &  
 &  32.74  
 &  25.16  
 &  22.00  
 &  \best{\phantom{0}0.00}  
 &  \best{\phantom{0}0.00}  
 &  82.47  
 &  66.83  
 &   
 &  44.50   \\
\hspace{\subtab} Shared Classi
 &  76.52
 &  44.86  
 &  56.96  
 &  57.00  
 &  \phantom{0}0.06  
 &  \phantom{0}0.22  
 &  79.55  
 &  75.37  
 &   
 &  65.50   \\
Shared Combi $s_2$
 &  76.52 
 &  32.71  
 &  25.70  
 &  23.67  
 &  \best{\phantom{0}0.00}  
 &  \best{\phantom{0}0.00}  
 &  80.95  
 &  65.93  
 &   
 &  44.87   \\
Shared Combi $s_3$
 &  76.52 
 &  32.72  
 &  25.79  
 &  23.67  
 &  \best{\phantom{0}0.00}  
 &  \best{\phantom{0}0.00}  
 &  80.90  
 &  65.98  
 &   
 &  44.96   \\
 \bottomrule
\end{tabularx}
}
\end{sc}
\end{small}
\end{center}
\vskip -0.1in
\end{table*}

\begin{table*}[!htbp]
\caption{Accuracy on the in-distribution (CIFAR-10/CIFAR-100) and \textbf{AUC} (AUROC) for various test out-distributions of different OOD methods with \textbf{80 Million Tiny Images}  as training out-distribution (shown results for test set of 80M are not used for computing the mean AUC).
The relative performance of the different methods measured in AUC is similar to what we observed with the FPR@95\%TPR measure in Table~\ref{table:fpr_80M}.
}\label{table:80M_sota}
\setlength\tabcolsep{.5pt} 
\vskip 0.15in
\begin{center}
\begin{small}
\begin{sc}
\makebox[\textwidth][c]{
\begin{tabularx}{1.0\textwidth}{lC|C|CCCCCCC|C}
\multicolumn{10}{c}{in-distribution: CIFAR-10} \\
\midrule
       & &  Mean &     SVHN    &   LSUN    &   Uni &   Smooth     &   C-100 & OpenIm  &   CelA  &   80M      \\ 
Model    & Acc. &  AUC &    AUC    &   AUC    &   AUC &  AUC    &   AUC & AUC &   AUC &   AUC      \\ \midrule
 Plain Classi
 &  94.84
 &  85.75
 &  91.91  
 &  91.63  
 &  87.69  
 &  78.27  
 &  87.83  
 &  83.23  
 &  79.43  
 &  88.01 \\ \midrule
 Mahalanobis
 &
 & 91.12
 & 94.34 
 & 91.98
 & \best{100.00}
 & 99.51
 & 88.08
 & 84.92
 & 79.17
 & 89.65 \\ \midrule
 Energy ($s_1$)
 & 95.22
 & 97.32
 & 99.26
 & 99.49
 & 99.00
 & 99.40
 & 93.81
 & 90.73
 & 99.57
 & 97.71
 \\ \midrule
 NTOM
 &  95.42
 &  97.32
 &  99.59  
 &  \best{99.79}  
 &  99.97  
 &  99.84  
 &  92.19  
 &  89.96  
 &  99.89  
 &  98.72
 \\
 ATOM
 & 95.20
 & 97.42
 & 99.63   
 & 99.76   
 & 99.93   
 & 99.60   
 & 92.89   
 & 90.30   
 & 99.85   
 & 98.55
 \\ \midrule
 OE  ($s_3$)
 &  95.74
 &  97.64
 &  99.48  
 &  99.48  
 &  99.46  
 &  99.64  
 &  94.80  
 &  90.91  
 &  99.71  
 &  98.50\\
\midrule
  \hspace{\subtab} \backgroundSone
 &  
 &  97.94    
 &  99.64  
 &  99.58  
 &  99.96  
 &  \best{99.98}  
 &  94.84  
 &  \best{91.65}  
 &  \best{99.92}  
 &  \best{98.78}   \\
  \hspace{\subtab} \backgroundStwo
 &  95.63
 &  97.95  
 &  99.60  
 &  99.52  
 &  99.97  
 &  \best{99.98}  
 &  95.03  
 &  \best{91.65}  
 &  \best{99.92}  
 &  98.65   \\
\backgroundSthree
 &  95.63
 &  97.95    
 &  99.58  
 &  99.52  
 &  99.97  
 &  \best{99.98}  
 &  95.04  
 &  91.64  
 &  \best{99.92}  
 &  98.70   \\
 \midrule
 \hspace{\subtab} Shared BinDisc ($s_1$)
 &  
 &  97.90  
 &  \best{99.74}  
 &  99.60  
 &  99.94  
 &  99.96  
 &  94.75  
 &  91.42  
 &  99.87  
 &  98.77   \\
\hspace{\subtab} Shared Classi
 &  \textbf{96.08}
 &  96.57  
 &  98.77  
 &  97.78  
 &  99.87  
 &  99.68  
 &  93.40  
 &  88.70  
 &  97.80  
 &  96.42   \\
Shared \combiStwo
 &  \textbf{96.08} 
 &  \textbf{97.96 } 
 &  99.73  
 &  99.58  
 &  99.95  
 &  99.96  
 &  95.13  
 &  91.48  
 &  99.85  
 &  98.73   \\
Shared \combiSthree
 &  \textbf{96.08} 
 &  \textbf{97.96}  
 &  99.73  
 &  99.58  
 &  99.95  
 &  99.96  
 &  \best{95.14}  
 &  91.48  
 &  99.85  
 &  98.73   \\
\toprule
\\
\multicolumn{10}{c}{in-distribution: CIFAR-100} \\
\midrule
       & &  Mean &     SVHN    &   LSUN    &   Uni &   Smooth     &   C-10 & OpenIm  &    &   80M      \\ 
Model    & Acc. &  AUC &    AUC    &   AUC    &   AUC &  AUC    &   AUC & AUC &    &   AUC      \\ 
\midrule
Plain Classi
 &  75.96
 &  77.48
 &  71.38  
 &  76.89  
 &  78.14  
 &  88.36  
 &  75.33  
 &  74.60 
 &  \textcolor{white}{65.86} 
 &  75.92 
 \\ \midrule
 Mahalanobis
 &
 & 87.75
 & 85.96
 & 87.20
 & \best{100.00}
 & 99.16
 & 75.45
 & 78.74
 &  
 & 79.76 \\ \midrule
 Energy ($s_1$)
 & 75.70
 & 91.67
 & \best{96.54}
 & 96.69
 & 97.91
 & 98.92
 & 77.39
 & 82.57
 & 
 & \best{91.16}
 \\ \midrule
NTOM
 & 74.88
 & 88.49
 & 96.20   
 & 97.31   
 & 99.79   
 & 99.94   
 & 62.44   
 & 75.24
 & \textcolor{white}{57.64}   
 & 88.41
 \\
 ATOM
 & 75.06
 & 88.02
 & 93.68   
 & \best{97.51}   
 & 99.98   
 & 98.46   
 & 63.47   
 & 75.02
 & \textcolor{white}{58.83}   
 & 88.44
 \\ \midrule
 OE  ($s_3$)
 &  \textbf{76.73}
 &  91.72
 &  94.06  
 &  95.58  
 &  99.06  
 &  98.84  
 &  \best{79.53}  
 &  83.31 
 &  \textcolor{white}{71.45} 
 &  88.43 \\
\midrule
  \hspace{\subtab} \backgroundSone
 &  
 &  \textbf{92.04}    
 &  94.42  
 &  95.47  
 &  99.99  
 &  99.73  
 &  79.15  
 & \best{ 83.46}  
 &  \phantom{0000}
 &  89.19   \\
\hspace{\subtab} \backgroundStwo
 &  75.82
 &  91.54  
 &  93.32  
 &  94.64  
 &  99.95  
 &  99.63  
 &  79.29  
 &  82.41  
 &   
 &  88.11   \\
\backgroundSthree
 &  75.82
 &  91.53    
 &  93.30  
 &  94.62  
 &  99.94  
 &  99.62  
 &  79.29  
 &  82.40  
 &  \phantom{0000}  
 &  88.23   \\
 \midrule
 \hspace{\subtab} Shared BinDisc ($s_1$)
 & 
 &  91.84  
 &  95.90  
 &  95.69  
 &  99.79  
 &  99.94  
 &  76.56  
 &  83.19 
 &    
 &  89.25   \\
\hspace{\subtab} Shared Classi
 &  76.52
 &  88.16  
 &  86.28  
 &  87.61  
 &  99.97  
 &  99.90  
 &  77.00  
 &  78.23 
 &    
 &  81.72   \\
Shared \combiStwo
 &  76.52 
 &  92.03  
 &  95.50  
 &  95.43  
 &  99.96  
 &  \best{99.98}  
 &  78.46  
 &  82.86 
 &    
 &  88.69   \\
Shared \combiSthree
 &  76.52 
 &  92.03  
 &  95.49  
 &  95.42  
 &  99.97  
 &  \best{99.98 } 
 &  78.46  
 &  82.85 
 &    
 &  88.67   \\

 \bottomrule
\end{tabularx}
}
\end{sc}
\end{small}
\end{center}
\vskip -0.1in
\end{table*}

\clearpage
\section{Experiments with Restricted ImageNet as In-Distribution}\label{sec:RImageNet}
In addition to the results for CIFAR-10 and CIFAR-100 shown in the main paper, here we provide results for Restricted ImageNet. Restricted ImageNet, introduced by \citet{tsipras2018robustness}, consists of 9 classes, where each individual class is a union of multiple ImageNet~\citep{imagenet_cvpr09} classes, for example the Restricted ImageNet class 'dog' contains all dog breeds from ImageNet. As Restricted ImageNet only contains animal classes, the union over all its classes does not cover the entire ILSVRC2012 dataset~\citep{ILSVRC15},  which allows us to use the remaining ILSVRC2012 classes as training out-distribution. Like we did for the CIFAR experiments, we train a plain classifier, an Outlier Exposure model, a background class model and a shared discriminator/classifier and evaluate them with the different scoring functions. The model is a ResNet50 and we use random cropping and flipping as data augmentation during training. The results in terms of FPR@95\%TPR and AUC can be found in Table~\ref{table:rImgNet}.

Once again, we see that \textsc{Shared Classi} has relatively good OOD detection performance and clearly beats the plain classifier from standard training.
Again, we see that Outlier Exposure~\cite{HenMazDie2019}, training with background class and shared training of classifier and binary discriminator perform similarly. 
In terms of accuracy, the shared model benefits most from the added unlabelled data compared to plain training.
At 95\% TPR, the sharedly trained binary discriminator and its combinations \textsc{Shared \combiStwo} and   \textsc{Shared \combiStwo} detect flower images significantly better than the other approaches which results in the best mean FPR@95\%TPR compared to the other methods, while in terms of AUC, OE has a slight advantage.

\begin{table*}[!htbp]
\caption{\textcolor{update}{Out-of-distribution detection evaluation for various ResNet50 models trained on \textbf{Restricted ImageNet} in terms of AUC and FPR@95\%TPR. The last column (\textsc{NotRIN}) refers to the remaining classes from the ILSVRC2012 validation split that are not part of Restricted ImageNet and that were used as the training out-distribution;it does not contribute to the mean test FPR/AUC. As all models use the train split of NotRIN as training-out distribution.} }\label{table:rImgNet}
\setlength\tabcolsep { .30pt }
\vskip 0.15in
\begin{center}
\begin{small}
\begin{sc}
\makebox[\textwidth][c]{
\begin{tabularx}{1.0\textwidth}{lC|C| *{ 5 }{C} |C}
\multicolumn{ 9 }{c}{in-distribution: Restricted ImageNet}\\[3mm]
\multicolumn{ 9 }{c}{FPR@95\%TPR}\\
\midrule
	& & Mean & Flowers & FGVC & Cars & Smooth & Uniform & NotRIN\\
Model & Acc. & FPR & FPR & FPR & FPR & FPR & FPR & FPR \\ \midrule
Plain Classi        &  96.34 & 36.71 &   60.11 &  50.23 &  73.20 &   \best{0.00} &     \best{0.00}&      50.37 \\ \midrule
OE ($s_3$)           &  97.10 & 4.26 &   21.06 &   \best{0.18} &   0.04 &   \best{0.00} &     \best{0.00}&       6.91 \\ \midrule
\hspace{\subtab} \backgroundSone           &   & 4.22
 &   20.74 &   0.33 &   \best{0.01} &   \best{0.00} &     \best{0.00}&       6.46 \\ 
\hspace{\subtab} \backgroundStwo        &  97.50 & 4.77 &   23.43 &   0.39 &   0.02 &   \best{0.00} &     \best{0.00}&       6.02 \\
\backgroundSthree         &  97.50 & 4.73 &   23.26 &   0.36 &   \best{0.01} &   \best{0.00} &     \best{0.00}&       6.13 \\ \midrule
\hspace{\subtab} Shared BinDisc ($s_1$)  &   & 2.73 &   \best{10.10} &   0.24 &   0.04 &   3.26 &     \best{0.00}&       5.93 \\
\hspace{\subtab} Shared Classi &  \textbf{97.59} & 17.51 &   50.54 &  17.79 &  19.21 &   \best{0.00} &     \best{0.00}&      21.99 \\
Shared \combiStwo    &  \textbf{97.59} & \textbf{2.55} &   12.28 &   0.45 &   0.04 &   \best{0.00} &     \best{0.00}&       5.58 \\
Shared \combiSthree    &  \textbf{97.59} & 2.62 &   12.60 &   0.45 &   0.04 &   \best{0.00} &     \best{0.00}&       \best{5.63} \\
\toprule
\\
\multicolumn{ 9 }{c}{AUC}\\
\midrule
	& & Mean & Flowers & FGVC & Cars & Smooth & Uniform & NotRIN\\
Model & Acc. & AUC & AUC & AUC & AUC & AUC & AUC & AUC \\
\midrule
Plain Classi        &  96.34 & 94.96 &  91.65 &  92.67 &  92.46 &  98.74 &   99.26 &      92.38 \\ \midrule
OE ($s_3$)           &  97.10 & \textbf{98.76} &  96.65 &  99.75 &  \best{99.85} &  97.95 &   \best{99.58} &      98.46 \\ \midrule
\hspace{\subtab} \backgroundSone           &        & 98.61 &  96.64 &  \best{99.86} &  \best{99.97} &  97.77 &   98.80 &      98.67 \\ 
\hspace{\subtab} \backgroundStwo       &  97.50 & 98.66 &  96.39 &  99.83 &  99.96 &  98.18 &   98.94 &      98.69 \\
    \backgroundSthree        &  97.50 & 98.66 &  96.43 &  99.83 &  99.96 &  98.14 &   98.93 &      98.68 \\ \midrule
\hspace{\subtab} Shared BinDisc ($s_1$)  &        & 98.26 &  \best{97.62} &  99.83 &  99.94 &  96.13 &   97.78 &      98.71 \\
\hspace{\subtab} Shared Classi &  \textbf{97.59} & 96.93 &  93.40 &  96.58 &  96.53 &  \best{99.48} &   98.66 &      96.10 \\
Shared \combiStwo    &  \textbf{97.59} & 98.54 &  97.41 &  99.80 &  99.93 &  97.37 &   98.18 &      \best{98.72} \\
Shared \combiSthree   &  \textbf{97.59} & 98.58 &  97.36 &  99.79 &  99.92 &  97.61 &   98.22 &      98.71 \\
\bottomrule
\end{tabularx}
}
\end{sc}
\end{small}
\end{center}
\vskip - 0.1in
\end{table*}

\clearpage
\section{Experiments with SVHN as Training Out-Distribution}\label{section:SVHN}
In order to examine the effect of a training out-distribution that relatively far away from the in-distribution, we show experiments with SVHN as out-distribution in Tables~\ref{table:svhn_fpr} and~\ref{table:svhn_auc}. The OOD detection performance of these methods is much worse than with the closer OpenImages and 80M training out-distributions.
In most cases, combinations $s_2$ and $s_3$/OE which implicitly or explicitly use the classifier confidence lead to better OOD detection than using the binary discriminator/$s_1$.
The inconsistent behaviour over the different test out.distributions of the methods that use SVHN as training out-distribution can be explained by the easiness of the discrimination task, which manifests itself in the fact that \textsc{BGC}$s_1$ and \textsc{Shared BinDisc} reach perfect FPR and AUC metrics.
This indicates a form of overfitting to this specific out-distribution, without consistent generalization to unseen distributions which do not have characteristic features that are similar to those appearing in SVHN images.

We also do not observe a beneficial effect on test accuracy for the methods that use SVHN compared to plain, which is to be expected as the representations learned from SVHN are hardly useful for the in-distribution task.

\begin{table*}[!htbp]
\caption{\textbf{ FPR@95\%TPR} for CIFAR-10/CIFAR-100 as in-distribution with \textbf{SVHN} as training out-distribution.
}\label{table:svhn_fpr}
\setlength\tabcolsep{.5pt} 
\vskip 0.15in
\begin{center}
\begin{small}
\begin{sc}
\makebox[\textwidth][c]{
\begin{tabularx}{1.0\textwidth}{lC|C|CCCCCCC|C}
\multicolumn{10}{c}{in-distribution: CIFAR-10} \\
\midrule
       & &  Mean    &   LSUN    &   Uni &   Smooth     &   C-100 &  80M & OpenIm &  CelA  &  SVHN       \\ 
Model    & Acc. &  FPR &    FPR    &   FPR    &   FPR &  FPR    &   FPR & FPR &   FPR &   FPR      \\ \midrule
Plain Classi
 &  \textbf{94.84}
 &  64.62  
 &  \best{48.33}
 &  52.67  
 &  75.69  
 &  62.58  
 &  62.91  
 &  66.38  
 &  83.79  
 &  60.53 \\ \midrule
OE  ($s_3$)
 &  94.80
 &  63.73    
 &  54.33  
 &  92.01  
 &  41.27  
 &  60.55  
 &  55.91  
 &  65.52  
 &  76.51  
 &  \phantom{0}0.03   \\ \midrule
\hspace{\subtab} \backgroundSone
 &  
 &  54.31    
 &  53.00  
 &  100.00  
 &  \best{\phantom{0}0.51}
 &  \best{55.16}  
 &  \best{39.28}  
 &  \best{64.36}  
 &  \best{67.87}  
 &  \best{\phantom{0}0.00}  \\
  \hspace{\subtab} \backgroundStwo
 &  94.51
 &  62.98    
 &  58.33  
 &  98.53  
 &  18.73  
 &  62.27  
 &  57.47  
 &  67.35  
 &  78.16  
 &  \phantom{0}0.02   \\
\backgroundSthree
 &  94.51
 &  63.08    
 &  58.33  
 &  98.52  
 &  19.18  
 &  62.34  
 &  57.63  
 &  67.39  
 &  78.16  
 &  \phantom{0}0.02   \\ \midrule
\hspace{\subtab} Shared BinDisc
 &  
 &  72.95    
 &  95.67  
 &  100.00  
 &  \phantom{0}5.81  
 &  77.04  
 &  66.75  
 &  85.17  
 &  80.22  
 &  \best{\phantom{0}0.00}  \\
\hspace{\subtab} Shared Classi 
 &  94.71
 &  48.39    
 &  54.33  
 &  \best{\phantom{0}0.00} 
 &  17.99  
 &  62.34  
 &  57.37  
 &  67.05  
 &  79.64  
 &  \phantom{0}0.57   \\
Shared Combi $s_2$
 &  94.71
 &  \textbf{46.27}    
 &  55.00  
 &  \best{\phantom{0}0.00} 
 &  \phantom{0}6.74  
 &  60.84  
 &  55.15  
 &  66.50  
 &  79.65  
 &  \phantom{0}0.02   \\
Shared Combi $s_3$
 &  94.71
 &  46.31    
 &  55.00  
 &  \best{\phantom{0}0.00} 
 &  \phantom{0}6.84  
 &  60.90  
 &  55.26  
 &  66.53  
 &  79.65  
 &  \phantom{0}0.02   \\
\toprule
\\
\multicolumn{10}{c}{in-distribution: CIFAR-100} \\
\midrule
       & &  Mean &   LSUN    &   Uni &   Smooth     &   C-10 & 80M  &  OpenIm    & & SVHN     \\ 
Model    & Acc. &  FPR &    FPR    &   FPR    &   FPR &  FPR    &   FPR & FPR &    &   FPR      \\ 
\midrule
Plain Classi
 &  75.96
 &  82.26  
 &  84.33  
 &  \best{80.00} 
 &  98.99  
 &  65.81  
 &  81.97  
 &  82.47  
 &   
 &  80.17   \\ \midrule
OE  ($s_3$)
 &  75.78
 &  \textbf{73.89}    
 &  81.33  
 &  99.47  
 &  15.93  
 &  83.32  
 &  80.30  
 &  \best{83.01}  
 &   \phantom{0000}  
 &  \phantom{0}0.04   \\ \midrule
  \hspace{\subtab} \backgroundSone
 &  
 &  77.16    
 &  96.33  
 &  100.00  
 &  \best{\phantom{0}3.35}  
 &  91.05  
 &  80.55  
 &  91.68  
 &   \phantom{0000}  
 &  \best{\phantom{0}0.00}  \\
  \hspace{\subtab} \backgroundStwo
 &  75.21
 &  74.53    
 &  80.00  
 &  98.49  
 &  21.49  
 &  83.14  
 &  80.52  
 &  83.54  
 &   \phantom{0000}  
 &  \phantom{0}0.07   \\
\backgroundSthree
 &  75.21
 &  74.53    
 &  80.00  
 &  98.49  
 &  21.50  
 &  83.14  
 &  80.53  
 &  83.54  
 &   \phantom{0000}  
 &  \phantom{0}0.07   \\  \midrule
\hspace{\subtab} Shared BinDisc
 &  
 &  82.08    
 &  99.33  
 &  100.00  
 &  10.04  
 &  95.00  
 &  91.21  
 &  96.89  
 &   \phantom{0000}  
 &  \best{\phantom{0}0.00}  \\
\hspace{\subtab} Shared Classi 
 &  \textbf{75.97}
 &  83.46    
 &  \best{78.67 } 
 &  100.00  
 &  76.67  
 &  \best{82.36}  
 &  79.99  
 &  83.08  
 &   \phantom{0000}  
 &  \phantom{0}0.44   \\
Shared Combi $s_2$
 &  \textbf{75.97}
 &  78.66    
 &  \best{78.67}  
 &  100.00  
 &  47.73  
 &  82.44  
 &  79.93  
 &  83.20  
 &   \phantom{0000}  
 &  \phantom{0}0.07   \\
Shared Combi $s_3$
 &  \textbf{75.97}
 &  78.66    
 &  \best{78.67}
 &  100.00  
 &  47.78  
 &  82.43  
 &  \best{79.92}  
 &  83.17  
 &   \phantom{0000}  
 &  \phantom{0}0.07   \\
 \bottomrule
\end{tabularx}
}
\end{sc}
\end{small}
\end{center}
\vskip -0.1in
\end{table*}

\begin{table*}[!htbp]
\caption{\textbf{AUC} for CIFAR-10/CIFAR-100 as in-distribution with \textbf{SVHN} as training out-distribution.
}\label{table:svhn_auc}
\setlength\tabcolsep{.5pt} 
\vskip 0.15in
\begin{center}
\begin{small}
\begin{sc}
\makebox[\textwidth][c]{
\begin{tabularx}{1.0\textwidth}{lC|C|CCCCCCC|C}
\multicolumn{10}{c}{in-distribution: CIFAR-10} \\
\midrule
       & &  Mean    &   LSUN    &   Uni &   Smooth     &   C-100 &  80M & OpenIm &  CelA  &  SVHN       \\ 
Model    & Acc. &  AUC &    AUC    &   AUC    &   AUC &  AUC    &   AUC & AUC &   AUC &   AUC      \\ \midrule
 Plain Classi
 &  \textbf{94.84}
 &  85.75
 &  \best{91.91}  
 &  91.63  
 &  87.69  
 &  78.27  
 &  87.83  
 &  83.23  
 &  79.43  
 &  88.01 \\ \midrule
 OE ($s_3$)
 &  94.80
 &  88.37    
 &  91.72  
 &  91.33  
 &  91.20  
 &  88.26  
 &  89.55  
 &  \best{82.97}  
 &  83.59  
 &  \best{100.00}   \\ \midrule
  \hspace{\subtab} \backgroundSone
 &  
 &  83.58    
 &  89.43  
 &  58.34  
 &  \best{99.80}  
 &  85.44  
 &  90.28  
 &  78.01  
 &  83.76  
 &  \best{100.00}   \\
  \hspace{\subtab} \backgroundStwo
 &  94.51
 &  89.62    
 &  91.50  
 &  91.14  
 &  97.60  
 &  \best{88.72}  
 &  \best{89.98}  
 &  82.96  
 &  \best{85.40}  
 &  \best{100.00}   \\
\backgroundSthree
 &  94.51
 &  89.60    
 &  91.50  
 &  91.14  
 &  97.54  
 &  88.71  
 &  89.96  
 &  82.96  
 &  \best{85.40}  
 &  \best{100.00}   \\ \midrule
\hspace{\subtab} Shared BinDisc
 &  
 &  58.47    
 &  37.27  
 &  38.61  
 &  98.85  
 &  55.19  
 &  65.66  
 &  48.61  
 &  65.12  
 &  \best{100.00}   \\
\hspace{\subtab} Shared Classi 
 &  94.71
 &  90.14    
 &  90.20  
 &  \best{99.70}  
 &  97.44  
 &  88.50  
 &  89.68  
 &  82.97  
 &  82.51  
 &  99.80   \\
Shared Combi $s_2$
 &  94.71
 &  90.18    
 &  89.74  
 &  99.67  
 &  99.00  
 &  88.49  
 &  89.87  
 &  82.43  
 &  82.09  
 &  \best{100.00}   \\
Shared Combi $s_3$
 &  94.71
 &  \textbf{90.19}    
 &  89.76  
 &  99.68  
 &  98.97  
 &  88.50  
 &  89.87  
 &  82.46  
 &  82.12  
 &  \best{100.00}   \\
\toprule
\\
\multicolumn{10}{c}{in-distribution: CIFAR-100} \\
\midrule
       & &  Mean    &   LSUN    &   Uni &   Smooth     &   C-10 & 80M  &  OpenIm  &  &  SVHN     \\ 
Model    & Acc. &  AUC &    AUC    &   AUC    &   AUC &  AUC    &   AUC & AUC &    &   AUC      \\
\midrule
Plain Classi
 &  75.96
 &  77.48
 &  71.38  
 &  \best{76.89}  
 &  78.14  
 &  \best{88.36}  
 &  75.33  
 &  74.60 
 &  \textcolor{white}{65.86} 
 &  75.92 
 \\ \midrule
OE ($s_3$)
 &  75.78
 &  \best{78.86}    
 &  75.93  
 &  73.53  
 &  97.02  
 &  75.37  
 &  76.57  
 &  74.74  
 &  \phantom{0000}
 &  99.99   \\ \midrule
  \hspace{\subtab} \backgroundSone
 &  
 &  66.60    
 &  67.45  
 &  24.51  
 &  \best{99.37}  
 &  68.50  
 &  73.82  
 &  65.97  
 &  \phantom{0000}
 &  \best{100.00}   \\
  \hspace{\subtab} \backgroundStwo
 &  75.21
 &  77.38    
 &  76.04  
 &  67.25  
 &  96.28  
 &  74.83  
 &  75.86  
 &  74.01  
 &  \phantom{0000}
 &  99.98   \\
\backgroundSthree
 &  75.21
 &  77.38    
 &  76.04  
 &  67.25  
 &  96.27  
 &  74.83  
 &  75.86  
 &  74.01  
 &  \phantom{0000}
 &  99.98   \\ \midrule
 \hspace{\subtab} Shared BinDisc
 & 
 &  55.60    
 &  34.81  
 &  53.23  
 &  98.03  
 &  49.00  
 &  54.57  
 &  43.93  
 &  \phantom{0000}
 &  \best{100.00}   \\
\hspace{\subtab} Shared Classi 
 &  \textbf{75.97}
 &  72.04    
 &  \best{80.31}
 &  41.57  
 &  83.32  
 &  75.64  
 &  \best{76.58}  
 &  \best{74.81}  
 &  \phantom{0000}
 &  99.90   \\
Shared Combi $s_2$
 &  \textbf{75.97}
 &  72.95    
 &  80.23  
 &  41.34  
 &  89.21  
 &  75.59  
 &  76.57  
 &  74.75  
 &  \phantom{0000}
 &  99.99   \\
Shared Combi $s_3$
 &  \textbf{75.97}
 &  72.94    
 &  80.23  
 &  41.34  
 &  89.18  
 &  75.59  
 &  76.57  
 &  74.75  
 &  \phantom{0000}
 &  99.99   \\
\bottomrule
\end{tabularx}
}
\end{sc}
\end{small}
\end{center}
\vskip -0.1in
\end{table*}

\clearpage

\section{Experiments with few In-Distribution training labels}\label{section:partially_labelled}
One practical advantage of binary discriminators between in- and out-distribution is that they do not necessitate labelled in-distribution data.
This means they are in principle employable in situations where we have a training dataset which is known to be sampled from the in-distribution, but would be infeasible to label, or even for tasks where no classes exist.
However when possible, as we observed in \ref{table:separate_OI_FPR}, it is often beneficial to use class information for shared training of the binary discriminator as well as combining it with classifier confidences.

Here, we regard a situation that represents a middle ground where only some of the in-distribution training samples come with labels.
Concretely, we use as the available in-distribution training dataset CIFAR-10 or CIFAR-100 where only every tenth image is labelled, while the remaining 90\% of samples are flagged as in-distribution but do not carry a class label.
In order to obtain the same number of training steps for OE and the classifier, we train for 1000 epochs over the labelled in-distribution data, which means that every unlabelled image is forwarded ~111 times and the training time is comparable to standard training with 100 epochs.

We compare standard OE which only uses the labelled part of the in-distribution training set with shared training of a binary discriminator trained on all data together with a classifier which is being trained on the labelled data only.

The results are shown in Tables~\ref{table:partially_labelled_OI} and ~\ref{table:partially_labelled_80M}.
We observe that, as expected, the binary discriminator heavily profits from the additional training data compared to the classifier part or OE.
Interestingly, as manifested in its improved accuracy and OOD detection performance compared to OE, \textsc{SharedClassi} again strongly benefits from sharing representations with the binary discriminator which sees vastly more data samples.

\clearpage

\begin{table*}[!htbp]
\caption{
\textbf{ FPR@95\%TPR} for CIFAR-10/CIFAR-100 with \textbf{90\% of the training samples being unlabelled} as in-distribution and with \textbf{OpenImages} as training out distribution 
\label{table:partially_labelled_OI}
}
\setlength\tabcolsep{.5pt} 
\vskip 0.15in
\begin{center}
\begin{small}
\begin{sc}
\makebox[\textwidth][c]{
\begin{tabularx}{1.0\textwidth}{lC|C|CCCCCCC|C}
\multicolumn{11}{c}{in-distribution: CIFAR10 (10\% labelled)  } \\ 
\midrule
Model & Acc. & Mean  &  SVHN  &  LSUN  &  Uni  &  Smooth  &  C-100  &  80M  &  CelA  &  OpenIm  \\ 
 \midrule
 OE 
  &  72.53
  &  86.41    
  &  95.72  
  &  63.67  
  &  100.00  
  &  99.86  
  &  86.52  
  &  84.76  
  &  74.34  
  &  76.00  
  \\ 
 \midrule
\hspace{\subtab} Shared BinDisc 
  &   
  & \textbf{28.43}    
  &  \best{33.91  }
  &  \best{\phantom{0}0.00}
  &  \best{\phantom{0}0.11}
  &  \best{27.29}  
  &  73.04  
  &  64.53  
  &  \best{\phantom{0}0.16}
  &  \best{\phantom{0}2.97}
  \\ 
\hspace{\subtab} Shared Classi 
  &  \textbf{86.65}
  &  71.33    
  &  79.20  
  &  35.00  
  &  99.86  
  &  77.94  
  &  78.24  
  &  74.91  
  &  54.17  
  &  59.96  
  \\ 
Shared Combi $s_2$
  &  \textbf{86.65}
  &  29.70    
  &  41.94  
  &  \best{\phantom{0}0.00}
  &  \phantom{0}0.80
  &  31.22  
  &  \best{70.89}
  &  \best{62.76} 
  &  \phantom{0}0.32  
  &  \phantom{0}5.10  
  \\ 
Shared Combi $s_3$
  &  \textbf{86.65}
  &  30.33    
  &  43.73  
  &  \best{\phantom{0}0.00}
  &  \phantom{0}0.98  
  &  32.67  
  &  71.40  
  &  63.16  
  &  \phantom{0}0.34  
  &  \phantom{0}5.54  
  \\ 
  \toprule
\multicolumn{11}{c}{in-distribution: CIFAR100 (10\% labelled) } \\ 
\midrule
Model & Acc. & Mean  &  SVHN  &  LSUN  &  Uni  &  Smooth  &  C-10  &  80M  &    &  OpenIm  \\ 
 \midrule
 OE
  &  38.09
  &  80.40    
  &  96.07  
  &  24.33  
  &  100.00  
  &  79.18  
  &  92.85  
  &  89.99  
  &  \phantom{0000}
  &  56.41  
  \\ 

 \midrule
\hspace{\subtab} Shared BinDisc 
  &  
  &  \textbf{54.92}    
  &  \best{73.39}  
  &  \best{\phantom{0}0.33}  
  &  \best{\phantom{0}6.18}  
  &  \best{67.27}  
  &  93.12  
  &  89.22  
  &  \phantom{0000}  
  &  \best{\phantom{0}3.52}  
  \\ 
\hspace{\subtab} Shared Classi 
  &  \textbf{51.74}
  &  77.28    
  &  81.53  
  &  12.33  
  &  99.66  
  &  93.93  
  &  \best{87.60}  
  &  88.64  
  &  \phantom{0000}  
  &  42.94  
  \\ 
Shared Combi $s_2$
  &  \textbf{51.74}
  &  70.77    
  &  76.61  
  &  \best{\phantom{0}0.33}  
  &  85.54  
  &  85.57  
  &  88.59  
  &  87.97  
  &  \phantom{0000}  
  &  \phantom{0}8.88  
  \\ 
Shared Combi $s_3$
  &  \textbf{51.74}
  &  70.96    
  &  76.70  
  &  \best{\phantom{0}0.33}  
  &  86.48  
  &  85.79  
  &  88.56  
  &  \best{87.90}  
  &  \phantom{0000}  
  &  \phantom{0}9.00  
  \\ 
\bottomrule
\end{tabularx}
}
\end{sc}
\end{small}
\end{center}
\vskip -0.1in
\end{table*}

\begin{table*}[!htbp]
\caption{
\textbf{ FPR@95\%TPR} for CIFAR-10/CIFAR-100 with \textbf{90\% of the training samples being unlabelled} as in-distribution and with \textbf{80 Million Tiny Images} as training out distribution 
\label{table:partially_labelled_80M}
}
\setlength\tabcolsep{.5pt} 
\vskip 0.15in
\begin{center}
\begin{small}
\begin{sc}
\makebox[\textwidth][c]{
\begin{tabularx}{1.0\textwidth}{lC|C|CCCCCCC|C}
\multicolumn{11}{c}{in-distribution: CIFAR10 (10\% labelled)  } \\ 
\midrule
Model & Acc. & Mean  &  SVHN  &  LSUN  &  Uni  &  Smooth  &  C-100  &  OpenIm  &  CelA  &  80M  \\ 

 \midrule
 OE 
  &  41.18
  &  88.54    
  &  54.00  
  &  95.33  
  &  100.00  
  &  82.74  
  &  95.16  
  &  94.52  
  &  98.06  
  &  94.51  
  \\ 

 \midrule
\hspace{\subtab} Shared BinDisc 
  &  
  &  \textbf{16.54}    
  &  \best{\phantom{0}7.67}  
  &  \best{\phantom{0}6.67}  
  &  \best{\phantom{0}0.13} 
  &  \best{\phantom{0}0.20} 
  &  52.15  
  &  \best{46.23} 
  &  \best{\phantom{0}2.74}  
  &  \best{21.82}  
  \\ 
\hspace{\subtab} Shared Classi 
  &  \textbf{82.74}
  &  54.20    
  &  41.69  
  &  44.00  
  &  93.86  
  &  40.87  
  &  66.81  
  &  66.32  
  &  25.84  
  &  48.57  
  \\ 
Shared Combi $s_2$
  &  \textbf{82.74}
  &  17.43    
  &  \phantom{0}9.34  
  &  \phantom{0}9.00  
  &  \phantom{0}0.73  
  &  \phantom{0}0.37  
  &  \best{51.59} 
  &  47.63  
  &  \phantom{0}3.33  
  &  24.08  
  \\ 
Shared Combi $s_3$
  &  \textbf{82.74}
  &  17.77    
  &  \phantom{0}9.86  
  &  \phantom{0}9.33  
  &  \phantom{0}0.97  
  &  \phantom{0}0.42  
  &  52.04  
  &  48.27  
  &  \phantom{0}3.52  
  &  24.67  
  \\ 
  \toprule
\multicolumn{11}{c}{in-distribution: CIFAR100 (10\% labelled) } \\
\midrule
Model & Acc. & Mean  &  SVHN  &  LSUN  &  Uni  &  Smooth  &  C-10  &  OpenIm  &     &  80M  \\ 

 \midrule
 OE 
  &  26.87
  &  91.50    
  &  92.03  
  &  74.00  
  &  100.00  
  &  98.76  
  &  93.70  
  &  90.52  
  &  \phantom{0000}  
  &  85.23  
  \\ 

 \midrule
\hspace{\subtab} Shared BinDisc 
  &  
  &  \textbf{49.79}    
  &  75.75  
  &  \best{35.33}
  &  \best{\phantom{0}0.03}  
  &  \best{17.06}  
  &  90.72  
  &  79.83  
  &  \phantom{0000}  
  &  \best{63.33}  
  \\ 
\hspace{\subtab} Shared Classi 
  &  \textbf{43.29}
  &  72.74    
  &  82.74  
  &  63.33  
  &  53.60  
  &  61.60  
  &  \best{89.67}  
  &  85.50  
  &  \phantom{0000}  
  &  76.33  
  \\ 
Shared Combi $s_2$
  &  \textbf{43.29}
  &  50.21    
  &  \best{74.88}  
  &  37.00  
  &  \phantom{0}0.10  
  &  19.91  
  &  89.75  
  &  \best{79.63}
  &  \phantom{0000}  
  &  63.42  
  \\ 
Shared Combi $s_3$
  &  \textbf{43.29}
  &  50.29    
  &  74.93  
  &  37.00  
  &  \phantom{0}0.11  
  &  20.27  
  &  89.71  
  &  79.73  
  &  \phantom{0000} 
  &  63.54  
  \\ 
\midrule
\bottomrule
\end{tabularx}
}
\end{sc}
\end{small}
\end{center}
\vskip -0.1in
\end{table*}

\clearpage

\section{The Effect of Varying $\lambda$}\label{sec:lambda}
We investigate the effect of choosing the training hyperparameter $\lambda$, which is the factor of the respective loss on the out-of-distribution samples and represents $\frac{p(o)}{p(i)}$ during training.
We evaluate models trained with Outlier Exposure~\cite{HenMazDie2019}, background class and shared training of binary discriminator and classifier, all scored with $s_3$ (the implicit scoring function of OE). Note that in Section \ref{section:experiments}, the  \textsc{OE},  \textsc{BGC} and \textsc{Shared} models trained with $\lambda = 1.0$.

In Tables \ref{table:lambda_OI_FPR} and \ref{table:lambda_OI_AUC} we see that for CIFAR-10 the differences between different choices of $\lambda$ are minor with no clear favorite, but setting $\lambda=2.0$ tends to be too high. For CIFAR-100, the differences are much larger. Here, choosing $\lambda$ too small can have a severe negative effect on the detection of the otherwise relatively  to detect far out-distributions SVHN, Uniform Noise and Smooth Noise.
Regarding the numbers for both in-distribution datasets, using $\lambda = 1$ is a considerate default choice.

\begin{table*}[!h]
\caption{Effect of varying $\lambda$ during training for OE~\citep{HenMazDie2019}, the $s_3$ scoring function for models with background class and \textsc{Shared Combi $s_3$}. Shown are test accuracy and \textbf{FPR@95\%TPR} with \textbf{OpenImages} as training out-distribution.
}\label{table:lambda_OI_FPR}
\setlength\tabcolsep{.5pt} 
\vskip 0.15in
\begin{center}
\begin{small}
\begin{sc}
\makebox[\textwidth][c]{
\begin{tabularx}{1.0\textwidth}{lC|C|CCCCCCC|C}
\multicolumn{10}{c}{in-distribution: CIFAR-10} \\
\midrule
       & &  Mean &     SVHN    &   LSUN    &   Uni &   Smooth     &   C-100 &  80M &   CelA  &    OpenIm     \\ 
Model    & Acc. &  FPR &    FPR    &   FPR    &   FPR &  FPR    &   FPR & FPR &   FPR &   FPR      \\
\midrule
OE $\lambda$=0.1
 &  95.36
 &  14.34    
 &  13.31  
 &  \phantom{0}0.00  
 &  \phantom{0}0.25  
 &  \phantom{0}0.03  
 &  48.37  
 &  36.65  
 &  \phantom{0}1.77  
 &  12.60   \\
 OE $\lambda$=0.25
 &  95.31
 &  16.17    
 &  19.64  
 &  \phantom{0}0.00  
 &  \phantom{0}0.00  
 &  \phantom{0}0.05  
 &  51.77  
 &  40.21  
 &  \phantom{0}1.55  
 &  \phantom{0}7.74   \\
OE    $\lambda$=0.5
 &  95.27
 &  15.98    
 &  16.23  
 &  \phantom{0}0.00  
 &  \phantom{0}0.07  
 &  \phantom{0}0.01  
 &  53.06  
 &  41.72  
 &  \phantom{0}0.79  
 &  \phantom{0}4.85   \\
OE $\lambda$=1.0
 &  95.06
 &  15.20    
 &  \phantom{0}9.58  
 &  \phantom{0}0.00  
 &  \phantom{0}0.00  
 &  \phantom{0}0.00  
 &  54.05  
 &  42.33  
 &  \phantom{0}0.45  
 &  \phantom{0}3.46   \\
OE $\lambda$=2.0
 &  95.19
 &  15.98    
 &  12.93  
 &  \phantom{0}0.00  
 &  \phantom{0}0.00  
 &  \phantom{0}0.00  
 &  55.99  
 &  42.68  
 &  \phantom{0}0.26  
 &  \phantom{0}1.67   \\
\midrule
\backgroundSthree \  $\lambda$=0.1
 &  95.38
 &  16.21    
 &  18.93  
 &  \phantom{0}0.00  
 &  \phantom{0}0.12  
 &  \phantom{0}1.05  
 &  52.43  
 &  39.42  
 &  \phantom{0}1.53  
 &  \phantom{0}9.56   \\
\backgroundSthree \  $\lambda$=0.25
 &  95.27
 &  15.83    
 &  12.76  
 &  \phantom{0}0.00  
 &  \phantom{0}0.04  
 &  \phantom{0}0.03  
 &  54.06  
 &  43.26  
 &  \phantom{0}0.66  
 &  \phantom{0}5.43   \\
\backgroundSthree \  $\lambda$=0.5
 &  95.23
 &  15.14    
 &  \phantom{0}9.33  
 &  \phantom{0}0.00  
 &  \phantom{0}0.09  
 &  \phantom{0}0.04  
 &  52.97  
 &  42.79  
 &  \phantom{0}0.74  
 &  \phantom{0}3.19   \\
\backgroundSthree  $\lambda$=1.0
 &  95.21
 &  16.63    
 &  \phantom{0}7.69  
 &  \phantom{0}0.00  
 &  \phantom{0}0.07  
 &  \phantom{0}2.36  
 &  55.19  
 &  44.67  
 &  \phantom{0}6.41  
 &  \phantom{0}1.74   \\
\backgroundSthree \  $\lambda$=2.0
 &  95.28
 &  17.45    
 &  \phantom{0}8.72  
 &  \phantom{0}0.00  
 &  \phantom{0}0.00  
 &  \phantom{0}0.00  
 &  59.02  
 &  48.09  
 &  \phantom{0}6.31  
 &  \phantom{0}0.71   \\
 \midrule
SC $s_3$ \ $\lambda$=0.1
 &  95.20
 &  16.49    
 &  15.22  
 &  \phantom{0}0.00  
 &  \phantom{0}0.27  
 &  \phantom{0}0.00  
 &  54.30  
 &  43.97  
 &  \phantom{0}1.69  
 &  10.46   \\
SC $s_3$ \ $\lambda$=0.25
 &  95.21
 &  16.17    
 &  15.32  
 &  \phantom{0}0.00  
 &  \phantom{0}0.15  
 &  \phantom{0}0.04  
 &  54.60  
 &  42.59  
 &  \phantom{0}0.53  
 &  \phantom{0}5.19   \\
SC $s_3$ \ $\lambda$=0.5
 &  95.25
 &  16.02    
 &  \phantom{0}9.21  
 &  \phantom{0}0.00  
 &  \phantom{0}0.09  
 &  \phantom{0}0.00  
 &  55.72  
 &  45.55  
 &  \phantom{0}1.56  
 &  \phantom{0}2.75   \\
SC $s_3$ \ $\lambda$=1.0
 &  95.28
 &  16.06    
 &  \phantom{0}9.00  
 &  \phantom{0}0.00  
 &  \phantom{0}0.00  
 &  \phantom{0}0.00  
 &  58.68  
 &  42.85  
 &  \phantom{0}1.91  
 &  \phantom{0}0.66   \\
SC $s_3$ \ $\lambda$=2.0
 &  95.26
 &  17.04    
 &  \phantom{0}9.33  
 &  \phantom{0}0.00  
 &  \phantom{0}0.00  
 &  \phantom{0}0.00  
 &  58.20  
 &  45.81  
 &  \phantom{0}5.92  
 &  \phantom{0}0.68   \\

\toprule
\multicolumn{10}{c}{in-distribution: CIFAR-100} \\
\midrule
       & &  Mean &     SVHN    &   LSUN    &   Uni &   Smooth     &   C-10 & 80M  &    &   OpenIm      \\ 
Model    & Acc. &  FPR &    FPR    &   FPR    &   FPR &  FPR    &   FPR & FPR &    &   FPR      \\ \midrule
 OE  $\lambda$=0.1
 &  77.28
 &  52.35    
 &  70.29  
 &  \phantom{0}1.00  
 &  28.58  
 &  56.15  
 &  82.44  
 &  75.64  
 &  \phantom{0000}  
 &  17.94   \\
 OE  $\lambda$=0.25
 &  76.96
 &  46.40    
 &  60.70  
 &  \phantom{0}0.00  
 &  \phantom{0}9.24  
 &  48.89  
 &  83.01  
 &  76.59  
 &  \phantom{0000}  
 &  \phantom{0}6.01   \\
 OE     $\lambda$=0.5
 &  77.22
 &  36.43    
 &  53.12  
 &  \phantom{0}0.00  
 &  \phantom{0}0.15  
 &  \phantom{0}4.62  
 &  82.86  
 &  77.84  
 &  \phantom{0000}  
 &  \phantom{0}4.23   \\
OE  $\lambda$=1.0
 &  77.19
 &  35.03    
 &  47.36  
 &  \phantom{0}0.00  
 &  \phantom{0}0.67  
 &  \phantom{0}0.08  
 &  84.64  
 &  77.42  
 &  \phantom{0000}  
 &  \phantom{0}1.28   \\
 OE  $\lambda$=2.0
 &  76.95
 &  33.25    
 &  37.04  
 &  \phantom{0}0.00  
 &  \phantom{0}0.00  
 &  \phantom{0}1.70  
 &  83.28  
 &  77.48  
 &  \phantom{0000}  
 &  \phantom{0}0.86   \\
\midrule
\backgroundSthree \  $\lambda$=0.1
 &  76.87
 &  41.23    
 &  67.39  
 &  \phantom{0}1.00  
 &  \phantom{0}5.97  
 &  13.81  
 &  82.44  
 &  76.76  
 &  \phantom{0000}  
 &  15.70   \\
\backgroundSthree \  $\lambda$=0.25
 &  77.05
 &  42.05    
 &  49.94  
 &  \phantom{0}0.00  
 &  17.07  
 &  25.41  
 &  82.34  
 &  77.55  
 &  \phantom{0000}  
 &  \phantom{0}7.47   \\
\backgroundSthree \  $\lambda$=0.5
 &  77.17
 &  36.14    
 &  43.90  
 &  \phantom{0}0.00  
 &  10.54  
 &  \phantom{0}3.70  
 &  82.86  
 &  75.86  
 &  \phantom{0000}  
 &  \phantom{0}3.76   \\
\backgroundSthree $\lambda$=1.0
 &  77.61
 &  33.36    
 &  37.27  
 &  \phantom{0}0.00  
 &  \phantom{0}0.00  
 &  \phantom{0}0.20  
 &  84.51  
 &  78.19  
 &  \phantom{0000}  
 &  \phantom{0}1.27   \\
\backgroundSthree \  $\lambda$=2.0
 &  77.26
 &  32.60    
 &  35.34  
 &  \phantom{0}0.00  
 &  \phantom{0}0.00  
 &  \phantom{0}0.00  
 &  82.69  
 &  77.57  
 &  \phantom{0000}  
 &  \phantom{0}1.20   \\
 \midrule
SC $s_3$ \ $\lambda$=0.1
 &  77.35
 &  58.54    
 &  62.07  
 &  \phantom{0}0.33  
 &  92.62  
 &  37.44  
 &  82.05  
 &  76.74  
 &  \phantom{0000}  
 &  14.17   \\
SC $s_3$ \ $\lambda$=0.25
 &  76.96
 &  41.34    
 &  63.29  
 &  \phantom{0}0.00  
 &  10.24  
 &  14.34  
 &  82.68  
 &  77.50  
 &  \phantom{0000}  
 &  \phantom{0}7.22   \\
SC $s_3$ \ $\lambda$=0.5
 &  76.64
 &  49.63    
 &  49.55  
 &  \phantom{0}0.00  
 &  44.78  
 &  43.85  
 &  82.30  
 &  77.31  
 &  \phantom{0000}  
 &  \phantom{0}1.97   \\
SC $s_3$ \ $\lambda$=1.0
 &  77.35
 &  33.06    
 &  37.57  
 &  \phantom{0}0.00  
 &  \phantom{0}0.00  
 &  \phantom{0}1.13  
 &  82.68  
 &  77.01  
 &  \phantom{0000}  
 &  \phantom{0}1.85   \\
SC $s_3$ \ $\lambda$=2.0
 &  76.63
 &  34.09    
 &  39.56  
 &  \phantom{0}0.00  
 &  \phantom{0}0.00  
 &  \phantom{0}5.21  
 &  82.57  
 &  77.18  
 &  \phantom{0000}  
 &  \phantom{0}1.04   \\
\bottomrule
\end{tabularx}
}
\end{sc}
\end{small}
\end{center}
\vskip -0.1in
\end{table*}

\begin{table*}[!h]
\caption{Effect of varying $\lambda$ during training for OE~\citep{HenMazDie2019}, the $s_3$ scoring function for models with background class and \textsc{Shared Combi $s_3$}. Shown are test accuracy and \textbf{AUC} (AUROC) with \textbf{OpenImages} as training out-distribution.
}\label{table:lambda_OI_AUC}
\setlength\tabcolsep{.5pt} 
\vskip 0.15in
\begin{center}
\begin{small}
\begin{sc}
\makebox[\textwidth][c]{
\begin{tabularx}{1.0\textwidth}{lC|C|CCCCCCC|C}
\multicolumn{10}{c}{in-distribution: CIFAR-10} \\
\midrule
       & &  Mean &     SVHN    &   LSUN    &   Uni &   Smooth     &   C-100 &  80M &   CelA  &    OpenIm     \\ 
Model    & Acc. &  AUC &    AUC    &   AUC    &   AUC &  AUC    &   AUC & AUC &   AUC &   AUC      \\
\midrule
  OE  $\lambda$=0.1
 &  95.36
 &  97.49    
 &  97.96  
 &  99.96  
 &  99.84  
 &  99.93  
 &  91.25  
 &  93.81  
 &  99.70  
 &  97.86   \\
   OE  $\lambda$=0.25
 &  95.31
 &  97.15    
 &  96.93  
 &  99.98  
 &  99.98  
 &  99.93  
 &  90.59  
 &  92.91  
 &  99.74  
 &  98.71   \\
 OE   $\lambda$=0.5
 &  95.27
 &  97.19    
 &  97.53  
 &  99.99  
 &  99.95  
 &  99.99  
 &  90.29  
 &  92.73  
 &  99.87  
 &  99.20   \\
OE  $\lambda$=1.0
 &  95.06
 &  97.28    
 &  98.49  
 &  99.99  
 &  99.99  
 &  99.99  
 &  90.03  
 &  92.53  
 &  99.91  
 &  99.43   \\
   OE  $\lambda$=2.0
 &  95.19
 &  97.19    
 &  97.99  
 &  99.99  
 &  99.99  
 &  100.00  
 &  89.89  
 &  92.56  
 &  99.93  
 &  99.67   \\
 \midrule
\backgroundSthree \   $\lambda$=0.1
 &  95.38
 &  97.37    
 &  97.32  
 &  99.99  
 &  99.93  
 &  99.63  
 &  91.18  
 &  93.74  
 &  99.78  
 &  98.39   \\
\backgroundSthree \   $\lambda$=0.25
 &  95.27
 &  97.43    
 &  98.25  
 &  100.00  
 &  99.96  
 &  99.98  
 &  90.85  
 &  93.09  
 &  99.91  
 &  99.13   \\
\backgroundSthree \   $\lambda$=0.5
 &  95.23
 &  97.44    
 &  98.62  
 &  100.00  
 &  99.93  
 &  99.98  
 &  90.80  
 &  92.84  
 &  99.89  
 &  99.48   \\
\backgroundSthree \   $\lambda$=1.0
 &  95.21
 &  97.21    
 &  98.87  
 &  100.00  
 &  99.98  
 &  99.62  
 &  90.47  
 &  92.41  
 &  99.08  
 &  99.71   \\
\backgroundSthree \   $\lambda$=2.0
 &  95.28
 &  97.10    
 &  98.71  
 &  100.00  
 &  99.98  
 &  100.00  
 &  89.93  
 &  91.94  
 &  99.11  
 &  99.86   \\
 \midrule
SC $s_3$ \ $\lambda$=0.1
 &  95.20
 &  97.29    
 &  97.87  
 &  99.99  
 &  99.91  
 &  99.98  
 &  90.69  
 &  92.82  
 &  99.77  
 &  98.32   \\
SC $s_3$ \ $\lambda$=0.25
 &  95.21
 &  97.33    
 &  97.90  
 &  99.99  
 &  99.91  
 &  99.97  
 &  90.57  
 &  93.08  
 &  99.92  
 &  99.17   \\
SC $s_3$ \ $\lambda$=0.5
 &  95.25
 &  97.34    
 &  98.68  
 &  100.00  
 &  99.95  
 &  100.00  
 &  90.39  
 &  92.56  
 &  99.78  
 &  99.57   \\
SC $s_3$ \ $\lambda$=1.0
 &  95.28
 &  97.26    
 &  98.62  
 &  100.00  
 &  99.93  
 &  99.94  
 &  89.75  
 &  92.85  
 &  99.71  
 &  99.88   \\
SC $s_3$ \ $\lambda$=2.0
 &  95.26
 &  97.13    
 &  98.59  
 &  100.00  
 &  100.00  
 &  99.99  
 &  89.95  
 &  92.25  
 &  99.13  
 &  99.88   \\
\midrule
\toprule
\multicolumn{10}{c}{in-distribution: CIFAR-100} \\
\midrule
       & &  Mean &     SVHN    &   LSUN    &   Uni &   Smooth     &   C-10 & 80M  &    &   OpenIm      \\ 
Model    & Acc. &  AUC &    AUC    &   AUC    &   AUC &  AUC    &   AUC & AUC &    &   AUC      \\ \midrule
OE   $\lambda$=0.1
 &  77.28
 &  87.38    
 &  83.69  
 &  99.79  
 &  94.79  
 &  91.50  
 &  76.08  
 &  78.43  
 &  \phantom{0000} 
 &  95.73   \\
OE   $\lambda$=0.25
 &  76.96
 &  88.31    
 &  86.94  
 &  99.98  
 &  97.74  
 &  90.74  
 &  76.11  
 &  78.35  
 &  \phantom{0000} 
 &  98.51   \\
OE    $\lambda$=0.5
 &  77.22
 &  89.82    
 &  87.22  
 &  99.98  
 &  99.74  
 &  98.95  
 &  75.39  
 &  77.64  
 &  \phantom{0000} 
 &  98.94   \\
OE   $\lambda$=1.0    
 &  77.19
 &  90.37    
 &  89.54  
 &  99.98  
 &  99.03  
 &  99.68  
 &  75.95  
 &  78.03  
 &  \phantom{0000}  
 &  99.67   \\
OE   $\lambda$=2.0
 &  76.95
 &  90.56    
 &  91.42  
 &  99.99  
 &  99.78  
 &  99.08  
 &  75.51  
 &  77.58  
 &  \phantom{0000} 
 &  99.76   \\
 \midrule
\backgroundSthree \  $\lambda$=0.1
 &  76.87
 &  88.86    
 &  83.41  
 &  99.87  
 &  98.35  
 &  97.37  
 &  75.97  
 &  78.21  
 &  \phantom{0000} 
 &  95.99   \\
\backgroundSthree \  $\lambda$=0.25
 &  77.05
 &  89.29    
 &  89.10  
 &  99.99  
 &  97.13  
 &  95.91  
 &  75.96  
 &  77.68  
 &  \phantom{0000} 
 &  98.08   \\
\backgroundSthree \  $\lambda$=0.5
 &  77.17
 &  90.17    
 &  90.07  
 &  99.99  
 &  97.99  
 &  99.07  
 &  75.49  
 &  78.42  
 &  \phantom{0000} 
 &  99.04   \\
\backgroundSthree \   $\lambda$=1.0
 &  77.61
 &  90.46    
 &  90.46  
 &  99.99  
 &  99.88  
 &  99.74  
 &  74.88  
 &  77.82  
 &  \phantom{0000}  
 &  99.64   \\
\backgroundSthree \  $\lambda$=2.0
 &  77.26
 &  90.83    
 &  91.74  
 &  99.99  
 &  99.72  
 &  99.80  
 &  75.72  
 &  78.02  
 &  \phantom{0000} 
 &  99.67   \\
\midrule
SC $s_3$ \  $\lambda$=0.1
 &  77.35
 &  86.08    
 &  85.56  
 &  99.95  
 &  84.55  
 &  93.24  
 &  75.27  
 &  77.89  
 &  \phantom{0000} 
 &  96.74   \\
SC $s_3$ \  $\lambda$=0.25
 &  76.96
 &  88.76    
 &  85.08  
 &  99.98  
 &  97.73  
 &  97.17  
 &  75.11  
 &  77.51  
 &  \phantom{0000} 
 &  98.31   \\
SC $s_3$ \  $\lambda$=0.5
 &  76.64
 &  88.42    
 &  88.80  
 &  99.99  
 &  94.61  
 &  93.94  
 &  75.63  
 &  77.52  
 &  \phantom{0000} 
 &  99.51   \\
SC $s_3$ \ $\lambda$=1.0
 &  77.35
 &  90.73    
 &  91.69  
 &  99.99  
 &  99.57  
 &  99.53  
 &  75.50  
 &  78.10  
 &  \phantom{0000}  
 &  99.57   \\
SC $s_3$ \  $\lambda$=2.0
 &  76.63
 &  90.54    
 &  91.24  
 &  99.99  
 &  99.86  
 &  98.93  
 &  75.30  
 &  77.91  
 &  \phantom{0000} 
 &  99.74   \\
\bottomrule
\end{tabularx}
}
\end{sc}
\end{small}
\end{center}
\vskip -0.1in
\end{table*}

\clearpage
\section{Statistics Over Five Runs with Different Seeds}

\begin{table*}[!htbp]
\caption{\textbf{Mean} $\mu$ and \textbf{standard deviation} $\sigma$ of the \textbf{FPR@95\%TPR} measure for different methods and scoring functions over five runs each for models with \textbf{OpenImages} as training out-distribution. The training details are the same as for the results shown in Table~\ref{table:OI_AA_FPR}.
}\label{table:stats_OI_fpr}
\setlength\tabcolsep{.5pt} 
\vskip 0.15in
\begin{center}
\begin{small}
\begin{sc}
\makebox[\textwidth][c]{
\begin{tabularx}{1.0\textwidth}{lC|C|CCCCCCC|C}
\multicolumn{10}{c}{in-distribution: CIFAR-10} \\
\midrule
       & &  Mean &     SVHN    &   LSUN    &   Uni &   Smooth     &   C-100 &  80M &   CelA  &    OpenIm     \\ 
Model    & Acc. &  FPR &    FPR    &   FPR    &   FPR &  FPR    &   FPR & FPR &   FPR &   FPR      \\ \midrule
OE  $\mu$
 &  95.11
 &  \textbf{15.49}    
 &  11.42  
 &  \best{\phantom{0}0.00  }
 &  \best{\phantom{0}0.00  }
 &  \phantom{0}0.09 
 &  \best{54.35  }
 &  \best{42.14}  
 &  \best{\phantom{0}0.40}  
 &  \phantom{0}2.75   \\
\rowcolor[gray]{.85}OE  $\sigma$
 &  0.04
 &  \phantom{0}0.32    
 &  \phantom{0}1.75  
 &  \phantom{0}0.00  
 &  \phantom{0}0.00  
 &  \phantom{0}0.18  
 &  \phantom{0}0.61  
 &  \phantom{0}0.55  
 &  \phantom{0}0.11  
 &  \phantom{0}0.49   \\ \midrule
 \hspace{\subtab}  \backgroundSone \ $\mu$
 &  
 &  19.18    
 &  \best{\phantom{0}3.86}
 &  \best{\phantom{0}0.00  }
 &  \best{\phantom{0}0.00  }
 &  \best{\phantom{0}0.00  }
 &  73.03  
 &  56.76  
 &  \phantom{0}0.60  
 &  \best{\phantom{0}0.05}   \\
\rowcolor[gray]{.85} \hspace{\subtab}  \backgroundSone \ $\sigma$
 &  
 &  \phantom{0}0.32    
 &  \phantom{0}0.97  
 &  \phantom{0}0.00  
 &  \phantom{0}0.00  
 &  \phantom{0}0.00  
 &  \phantom{0}1.17  
 &  \phantom{0}1.19  
 &  \phantom{0}0.32  
 &  \phantom{0}0.01   \\
 \backgroundStwo\ $\mu$
 &  \textbf{95.29}
 &  15.96    
 &  \phantom{0}8.48  
 &  \best{\phantom{0}0.00  }
 &  \phantom{0}0.04  
 &  \phantom{0}0.42  
 &  55.41  
 &  44.52  
 &  \phantom{0}2.87  
 &  \phantom{0}1.09   \\
\rowcolor[gray]{.85} \backgroundStwo\ $\sigma$
 &  0.10
 &  \phantom{0}0.63    
 &  \phantom{0}0.88  
 &  \phantom{0}0.00  
 &  \phantom{0}0.05  
 &  \phantom{0}0.84  
 &  \phantom{0}1.57  
 &  \phantom{0}1.55  
 &  \phantom{0}1.94  
 &  \phantom{0}0.34   \\
 \backgroundSthree\ $\mu$
 &  \textbf{95.29}
 &  15.99    
 &  \phantom{0}8.69  
 &  \best{\phantom{0}0.00}  
 &  \phantom{0}0.04  
 &  \phantom{0}0.47  
 &  55.29  
 &  44.49  
 &  \phantom{0}2.94  
 &  \phantom{0}1.16   \\
\rowcolor[gray]{.85} \backgroundSthree\ $\sigma$
 &  0.10
 &  \phantom{0}0.63    
 &  \phantom{0}0.91  
 &  \phantom{0}0.00  
 &  \phantom{0}0.06  
 &  \phantom{0}0.95  
 &  \phantom{0}1.53  
 &  \phantom{0}1.53  
 &  \phantom{0}1.99  
 &  \phantom{0}0.35   \\ \midrule
\hspace{\subtab} Shared BinDisc  $\mu$ 
 &  
 &  20.45    
 &  \phantom{0}5.30  
 &  \best{\phantom{0}0.00  }
 &  \best{\phantom{0}0.00  }
 &  \best{\phantom{0}0.00  }
 &  77.19  
 &  59.63  
 &  \phantom{0}1.06  
 &  \best{\phantom{0}0.05}   \\
\rowcolor[gray]{.85}\hspace{\subtab} Shared BinDisc  $\sigma$
 & 
 &  \phantom{0}0.49    
 &  \phantom{0}1.02  
 &  \phantom{0}0.00  
 &  \phantom{0}0.00  
 &  \phantom{0}0.00  
 &  \phantom{0}1.09  
 &  \phantom{0}3.00  
 &  \phantom{0}0.46  
 &  \phantom{0}0.02   \\
\hspace{\subtab} Shared Classi  $\mu$ 
 &  95.21
 &  33.21    
 &  26.37  
 &  \phantom{0}9.33  
 &  49.73  
 &  11.31  
 &  57.34  
 &  48.37  
 &  29.98  
 &  36.73   \\
\rowcolor[gray]{.85}\hspace{\subtab} Shared Classi  $\sigma$
 &  0.07
 &  \phantom{0}3.81    
 &  \phantom{0}1.24  
 &  \phantom{0}1.35  
 &  29.63  
 &  12.22  
 &  \phantom{0}1.18  
 &  \phantom{0}1.19  
 &  \phantom{0}3.85  
 &  \phantom{0}1.67   \\
Shared Combi $s_2$  $\mu$ 
 &  95.21
 &  16.28    
 &  \phantom{0}8.77  
 &  \best{\phantom{0}0.00}  
 &  \phantom{0}0.20  
 &  \best{\phantom{0}0.00  }
 &  57.14  
 &  44.92  
 &  \phantom{0}2.92  
 &  \phantom{0}0.95   \\
\rowcolor[gray]{.85}Shared Combi $s_2$  $\sigma$
 &  0.07
 &  \phantom{0}0.39    
 &  \phantom{0}1.11  
 &  \phantom{0}0.00  
 &  \phantom{0}0.38  
 &  \phantom{0}0.00  
 &  \phantom{0}1.39  
 &  \phantom{0}1.37  
 &  \phantom{0}1.42  
 &  \phantom{0}0.23   \\
Shared Combi $s_3$  $\mu$ 
 &  95.21
 &  16.35    
 &  \phantom{0}9.06  
 &  \best{\phantom{0}0.00  }
 &  \phantom{0}0.24  
 &  \best{\phantom{0}0.00}  
 &  57.13  
 &  45.02  
 &  \phantom{0}3.03  
 &  \phantom{0}1.03   \\
\rowcolor[gray]{.85}Shared Combi $s_3$  $\sigma$
 &  0.07
 &  \phantom{0}0.41    
 &  \phantom{0}1.13  
 &  \phantom{0}0.00
 &  \phantom{0}0.46  
 &  \phantom{0}0.00  
 &  \phantom{0}1.38  
 &  \phantom{0}1.38  
 &  \phantom{0}1.47  
 &  \phantom{0}0.24   \\
 \bottomrule
 \\
\multicolumn{10}{c}{in-distribution: CIFAR-100} \\
\midrule
       & &  Mean &     SVHN    &   LSUN    &   Uni &   Smooth     &   C-10 & 80M  &    &   OpenIm      \\ 
Model    & Acc. &  FPR &    FPR    &   FPR    &   FPR &  FPR    &   FPR & FPR &    &   FPR      \\ \midrule
OE $\mu$
 &  77.13
 &  35.16    
 &  45.27  
 &  \best{\phantom{0}0.00}  
 &  \phantom{0}0.13  
 &  \phantom{0}5.15  
 &  83.47  
 &  76.93  
 &  \phantom{0000}   
 &  \phantom{0}1.66   \\
\rowcolor[gray]{.85}OE $\sigma$
 &  0.23
 &  \phantom{0}0.80    
 &  \phantom{0}5.35  
 &  \phantom{0}0.00  
 &  \phantom{0}0.27  
 &  \phantom{0}5.35  
 &  \phantom{0}0.89  
 &  \phantom{0}0.59  
 &  \phantom{0000}   
 &  \phantom{0}0.54   \\ \midrule
 \hspace{\subtab}  \backgroundSone \ $\mu$
 &  
 &  \textbf{31.12}    
 &  \best{11.88}  
 &  \best{\phantom{0}0.00}  
 &  \best{\phantom{0}0.00}  
 &  \phantom{0}0.12  
 &  93.85  
 &  80.85  
 &  \phantom{0000}   
 &  \phantom{0}0.08   \\
\rowcolor[gray]{.85} \hspace{\subtab}  \backgroundSone \ $\sigma$
 &  
 &  \phantom{0}0.41    
 &  \phantom{0}2.20  
 &  \phantom{0}0.00  
 &  \phantom{0}0.00  
 &  \phantom{0}0.24  
 &  \phantom{0}0.15  
 &  \phantom{0}0.98  
 &  \phantom{0000}   
 &  \phantom{0}0.03   \\ 
 \backgroundStwo\ $\mu$
 &  \textbf{77.30}
 &  35.54    
 &  40.11  
 &  \best{\phantom{0}0.00}  
 &  \phantom{0}7.58  
 &  \phantom{0}4.90  
 &  83.44  
 &  77.20  
 &  \phantom{0000}   
 &  \phantom{0}1.67   \\
\rowcolor[gray]{.85} \backgroundStwo\ $\sigma$
 &  0.35
 &  \phantom{0}2.12    
 &  \phantom{0}5.67  
 &  \phantom{0}0.00  
 &  10.26  
 &  \phantom{0}5.35  
 &  \phantom{0}0.80  
 &  \phantom{0}0.73  
 &  \phantom{0000}   
 &  \phantom{0}0.55   \\
 \backgroundSthree\ $\mu$
 &  \textbf{77.30}
 &  35.64    
 &  40.33  
 &  \best{\phantom{0}0.00}  
 &  \phantom{0}7.89  
 &  \phantom{0}4.98  
 &  83.43  
 &  77.21  
 &  \phantom{0000}   
 &  \phantom{0}1.69   \\
\rowcolor[gray]{.85} \backgroundSthree\ $\sigma$
 &  0.35
 &  \phantom{0}2.16    
 &  \phantom{0}5.68  
 &  \phantom{0}0.00  
 &  10.68  
 &  \phantom{0}5.40  
 &  \phantom{0}0.81  
 &  \phantom{0}0.73  
 &  \phantom{0000}   
 &  \phantom{0}0.56   \\ \midrule
\hspace{\subtab} Shared BinDisc  $\mu$
 &  
 &  32.16    
 &  13.43  
 &  \best{\phantom{0}0.00}  
 &  \best{\phantom{0}0.00}  
 &  \best{\phantom{0}0.00}  
 &  95.13  
 &  84.39  
 &  \phantom{0000}   
 &  \best{\phantom{0}0.05}   \\
\rowcolor[gray]{.85}\hspace{\subtab} Shared BinDisc  $\sigma$
 &  
 &  \phantom{0}0.22    
 &  \phantom{0}2.10  
 &  \phantom{0}0.00  
 &  \phantom{0}0.00  
 &  \phantom{0}0.00  
 &  \phantom{0}0.19  
 &  \phantom{0}0.83  
 &  \phantom{0000}   
 &  \phantom{0}0.02   \\
\hspace{\subtab} Shared Classi  $\mu$
 &  77.11
 &  56.52    
 &  64.95  
 &  \phantom{0}2.73  
 &  81.18  
 &  31.85  
 &  \best{81.61}  
 &  76.80  
 &  \phantom{0000}   
 &  22.97   \\
\rowcolor[gray]{.85}\hspace{\subtab} Shared Classi  $\sigma$
 &  0.19
 &  \phantom{0}6.86    
 &  \phantom{0}5.79  
 &  \phantom{0}1.24  
 &  31.58  
 &  20.90  
 &  \phantom{0}0.53  
 &  \phantom{0}0.90  
 &  \phantom{0000}   
 &  \phantom{0}2.95   \\
Shared Combi $s_2$  $\mu$
 &  77.11
 &  33.13    
 &  35.93  
 &  \best{\phantom{0}0.00}  
 &  \best{\phantom{0}0.00}  
 &  \phantom{0}3.17  
 &  83.26  
 &  \best{76.43}  
 &  \phantom{0000}   
 &  \phantom{0}1.13   \\
\rowcolor[gray]{.85}Shared Combi $s_2$  $\sigma$
 &  0.19
 &  \phantom{0}0.59    
 &  \phantom{0}4.68  
 &  \phantom{0}0.00  
 &  \phantom{0}0.00  
 &  \phantom{0}4.35  
 &  \phantom{0}0.51  
 &  \phantom{0}0.82  
 &  \phantom{0000}   
 &  \phantom{0}0.36   \\
Shared Combi $s_3$  $\mu$
 &  77.11
 &  33.18    
 &  36.17  
 &  \best{\phantom{0}0.00}  
 &  \best{\phantom{0}0.00}  
 &  \phantom{0}3.23  
 &  83.24  
 &  \best{76.43}  
 &  \phantom{0000}   
 &  \phantom{0}1.14   \\
\rowcolor[gray]{.85}Shared Combi $s_3$  $\sigma$
 &  0.19
 &  \phantom{0}0.59    
 &  \phantom{0}4.70  
 &  \phantom{0}0.00  
 &  \phantom{0}0.00  
 &  \phantom{0}4.40  
 &  \phantom{0}0.51  
 &  \phantom{0}0.81  
 &  \phantom{0000}   
 &  \phantom{0}0.38   \\
 \bottomrule
\end{tabularx}
}
\end{sc}
\end{small}
\end{center}
\vskip -0.1in
\end{table*}

\begin{table*}[!htbp]
\caption{\textbf{Mean} $\mu$ and \textbf{standard deviation} $\sigma$ of the \textbf{AUROC} measure for different methods and scoring functions over five runs each for models with \textbf{OpenImages} as training out-distribution. The training details are the same as for the results shown in Table~\ref{table:OI_auc}.
}\label{table:stats_OI_auc}
\setlength\tabcolsep{.5pt} 
\vskip 0.15in
\begin{center}
\begin{small}
\begin{sc}
\makebox[\textwidth][c]{
\begin{tabularx}{1.0\textwidth}{lC|C|CCCCCCC|C}
\multicolumn{10}{c}{in-distribution: CIFAR-10} \\
\midrule
       & &  Mean &     SVHN    &   LSUN    &   Uni &   Smooth     &   C-100 &  80M &   CelA  &    OpenIm     \\ 
Model    & Acc. &  AUC &    AUC    &   AUC    &   AUC &  AUC    &   AUC & AUC &   AUC &   AUC      \\ \midrule
OE  $\mu$
 &  95.11
 &  97.25    
 &  98.23  
 &  99.99  
 &  99.96  
 &  99.93  
 &  90.13  
 &  \best{92.58}  
 &  \best{99.91} 
 &  99.52   \\
\rowcolor[gray]{.85}OE  $\sigma$
 &  0.04
 &  \phantom{0}0.05    
 &  \phantom{0}0.26  
 &  \phantom{0}0.01  
 &  \phantom{0}0.04  
 &  \phantom{0}0.07  
 &  \phantom{0}0.08  
 &  \phantom{0}0.15  
 &  \phantom{0}0.01  
 &  \phantom{0}0.06   \\ \midrule
 \hspace{\subtab}  \backgroundSone \ $\mu$
 &  
 &  94.59    
 &  \best{99.16}  
 &   \best{100.00}  
 &  \best{99.98}  
 &  99.97  
 &  78.02  
 &  85.15  
 &  99.85  
 &  \best{99.96}   \\
\rowcolor[gray]{.85} \hspace{\subtab}  \backgroundSone \ $\sigma$
 &  
 &  \phantom{0}0.27    
 &  \phantom{0}0.21  
 &  \phantom{0}0.00  
 &  \phantom{0}0.02  
 &  \phantom{0}0.02  
 &  \phantom{0}1.02  
 &  \phantom{0}0.85  
 &  \phantom{0}0.07  
 &  \phantom{0}0.01   \\
 \backgroundStwo\ $\mu$
 &  \textbf{95.29}
 &  \textbf{97.33}    
 &  98.74  
 &   \best{100.00}  
 &  \best{99.98}  
 &  99.92  
 &  90.48  
 &  \best{92.58}  
 &  99.59  
 &  99.81   \\
\rowcolor[gray]{.85} \backgroundStwo\ $\sigma$
 &  0.10
 &  \phantom{0}0.08    
 &  \phantom{0}0.12  
 &  \phantom{0}0.00  
 &  \phantom{0}0.03  
 &  \phantom{0}0.13  
 &  \phantom{0}0.18  
 &  \phantom{0}0.22  
 &  \phantom{0}0.27  
 &  \phantom{0}0.05   \\
 \backgroundSthree\ $\mu$
 &  \textbf{95.29}
 &  97.32    
 &  98.71  
 &   \best{100.00}  
 &  99.97  
 &  99.91  
 &  \best{90.49}  
 &  \best{92.58}  
 &  99.58  
 &  99.80   \\
\rowcolor[gray]{.85} \backgroundSthree\ $\sigma$
 &  0.10
 &  \phantom{0}0.08    
 &  \phantom{0}0.12  
 &  \phantom{0}0.00  
 &  \phantom{0}0.04  
 &  \phantom{0}0.15  
 &  \phantom{0}0.18  
 &  \phantom{0}0.22  
 &  \phantom{0}0.28  
 &  \phantom{0}0.05   \\  \midrule
\hspace{\subtab} Shared BinDisc $\mu$
 &  
 &  92.10    
 &  98.68  
 &   \best{100.00}  
 &  99.94  
 &  99.98  
 &  67.81  
 &  78.51  
 &  99.75  
 &  99.95   \\
\rowcolor[gray]{.85}\hspace{\subtab} Shared BinDisc $\sigma$
 & 
 &  \phantom{0}0.42    
 &  \phantom{0}0.23  
 &  \phantom{0}0.00  
 &  \phantom{0}0.04  
 &  \phantom{0}0.02  
 &  \phantom{0}1.30  
 &  \phantom{0}1.85  
 &  \phantom{0}0.10  
 &  \phantom{0}0.01   \\
\hspace{\subtab}Shared Classi $\mu$
 &  95.21
 &  95.10    
 &  96.36  
 &  98.45  
 &  94.26  
 &  98.15  
 &  90.29  
 &  92.18  
 &  95.98  
 &  93.27   \\
\rowcolor[gray]{.85}\hspace{\subtab}Shared Classi $\sigma$
 &  0.07
 &  \phantom{0}0.39    
 &  \phantom{0}0.18  
 &  \phantom{0}0.15  
 &  \phantom{0}3.07  
 &  \phantom{0}1.21  
 &  \phantom{0}0.18  
 &  \phantom{0}0.21  
 &  \phantom{0}0.46  
 &  \phantom{0}0.34   \\
Shared Combi $s_2$ $\mu$
 &  95.21
 &  97.24    
 &  98.67  
 &   \best{100.00}  
 &  99.92  
 &  \best{99.98}  
 &  90.07  
 &  92.47  
 &  99.58  
 &  99.84   \\
\rowcolor[gray]{.85}Shared Combi $s_2$ $\sigma$
 &  0.07
 &  \phantom{0}0.03    
 &  \phantom{0}0.17  
 &  \phantom{0}0.00  
 &  \phantom{0}0.08  
 &  \phantom{0}0.02  
 &  \phantom{0}0.23  
 &  \phantom{0}0.23  
 &  \phantom{0}0.20  
 &  \phantom{0}0.03   \\
Shared Combi $s_3$ $\mu$
 &  95.21
 &  97.24    
 &  98.64  
 &   \best{100.00}  
 &  99.91  
 &  \best{99.98}  
 &  90.09  
 &  92.48  
 &  99.56  
 &  99.83   \\
\rowcolor[gray]{.85}Shared Combi $s_3$ $\sigma$
 &  0.07
 &  \phantom{0}0.04    
 &  \phantom{0}0.17  
 &  \phantom{0}0.00  
 &  \phantom{0}0.10  
 &  \phantom{0}0.02  
 &  \phantom{0}0.22  
 &  \phantom{0}0.23  
 &  \phantom{0}0.20  
 &  \phantom{0}0.03   \\
 \bottomrule
 \\
\multicolumn{10}{c}{in-distribution: CIFAR-100} \\
\midrule
       & &  Mean &     SVHN    &   LSUN    &   Uni &   Smooth     &   C-10 & 80M  &    &   OpenIm      \\ 
Model    & Acc. &  AUC &    AUC    &   AUC    &   AUC &  AUC    &   AUC & AUC &    &   AUC      \\ \midrule
OE $\mu$
 &  77.13
 &  90.24    
 &  89.58  
 &  99.98  
 &  \best{99.66} 
 &  98.76  
 &  75.50  
 &  77.96  
 &  \phantom{0000}   
 &  99.54   \\
\rowcolor[gray]{.85}OE $\sigma$
 &  0.23
 &  \phantom{0}0.16    
 &  \phantom{0}1.10  
 &  \phantom{0}0.00  
 &  \phantom{0}0.32  
 &  \phantom{0}0.77  
 &  \phantom{0}0.42  
 &  \phantom{0}0.28  
 &  \phantom{0000}   
 &  \phantom{0}0.13   \\ \midrule
 \hspace{\subtab}  \backgroundSone \ $\mu$
 & 
 &  88.48    
 &  \best{97.17}  
 &   \best{99.99}  
 &  99.61  
 &  \best{99.57}  
 &  60.96  
 &  73.57  
 &  \phantom{0000}   
 &  99.92   \\
\rowcolor[gray]{.85} \hspace{\subtab}  \backgroundSone \ $\sigma$
 &  
 &  \phantom{0}0.17    
 &  \phantom{0}0.38  
 &  \phantom{0}0.00  
 &  \phantom{0}0.14  
 &  \phantom{0}0.21  
 &  \phantom{0}0.51  
 &  \phantom{0}0.38  
 &  \phantom{0000}   
 &  \phantom{0}0.01   \\
 \backgroundStwo\ $\mu$
 &  \textbf{77.30}
 &  90.23    
 &  90.45  
 &   \best{99.99}  
 &  98.60  
 &  98.90  
 &  75.34  
 &  78.12  
 &  \phantom{0000}   
 &  99.55   \\
\rowcolor[gray]{.85} \backgroundStwo\ $\sigma$
 &  0.35
 &  \phantom{0}0.27    
 &  \phantom{0}1.10  
 &  \phantom{0}0.00  
 &  \phantom{0}1.50  
 &  \phantom{0}1.04  
 &  \phantom{0}0.32  
 &  \phantom{0}0.31  
 &  \phantom{0000}   
 &  \phantom{0}0.13   \\
 \backgroundSthree\ $\mu$
 &  \textbf{77.30}
 &  90.22    
 &  90.41  
 &   \best{99.99}  
 &  98.56  
 &  98.89  
 &  75.34  
 &  78.12  
 &  \phantom{0000}   
 &  99.55   \\
\rowcolor[gray]{.85} \backgroundSthree\ $\sigma$
 &  0.35
 &  \phantom{0}0.28    
 &  \phantom{0}1.10  
 &  \phantom{0}0.00  
 &  \phantom{0}1.55  
 &  \phantom{0}1.05  
 &  \phantom{0}0.32  
 &  \phantom{0}0.31  
 &  \phantom{0000}   
 &  \phantom{0}0.13   \\ \midrule
\hspace{\subtab} Shared BinDisc  $\mu$
 &  
 &  84.80    
 &  96.66  
 &   \best{99.99}  
 &  99.64  
 &  99.54  
 &  48.89  
 &  64.09  
 &  \phantom{0000}   
 &  \best{99.93}   \\
\rowcolor[gray]{.85}\hspace{\subtab} Shared BinDisc  $\sigma$
 &  
 &  \phantom{0}0.24    
 &  \phantom{0}0.55  
 &  \phantom{0}0.00  
 &  \phantom{0}0.12  
 &  \phantom{0}0.13  
 &  \phantom{0}1.01  
 &  \phantom{0}1.00  
 &  \phantom{0000}   
 &  \phantom{0}0.01   \\
\hspace{\subtab} Shared Classi  $\mu$
 &  77.11
 &  84.42    
 &  84.18  
 &  99.40  
 &  75.39  
 &  93.54  
 &  \best{75.82}  
 &  78.16  
 &  \phantom{0000}   
 &  94.82   \\
\rowcolor[gray]{.85}\hspace{\subtab} Shared Classi  $\sigma$
 &  0.19
 &  \phantom{0}2.29    
 &  \phantom{0}2.09  
 &  \phantom{0}0.20  
 &  12.57  
 &  \phantom{0}5.07  
 &  \phantom{0}0.14  
 &  \phantom{0}0.25  
 &  \phantom{0000}   
 &  \phantom{0}0.68   \\
Shared Combi $s_2$  $\mu$
 &  77.11
 &  \textbf{90.71}    
 &  91.89  
 &   \best{99.99}  
 &  99.60  
 &  99.23  
 &  75.31  
 &  \best{78.24}  
 &  \phantom{0000}   
 &  99.72   \\
\rowcolor[gray]{.85}Shared Combi $s_2$  $\sigma$
 &  0.19
 &  \phantom{0}0.10    
 &  \phantom{0}1.11  
 &  \phantom{0}0.00  
 &  \phantom{0}0.07  
 &  \phantom{0}0.68  
 &  \phantom{0}0.20  
 &  \phantom{0}0.23  
 &  \phantom{0000}   
 &  \phantom{0}0.08   \\
Shared Combi $s_3$  $\mu$
 &  77.11
 &  90.70    
 &  91.85  
 &   \best{99.99}  
 &  99.59  
 &  99.22  
 &  75.32  
 &  \best{78.24}  
 &  \phantom{0000}   
 &  99.72   \\
\rowcolor[gray]{.85}Shared Combi $s_3$  $\sigma$
 &  0.19
 &  \phantom{0}0.10    
 &  \phantom{0}1.11  
 &  \phantom{0}0.00  
 &  \phantom{0}0.08  
 &  \phantom{0}0.69  
 &  \phantom{0}0.20  
 &  \phantom{0}0.23  
 &  \phantom{0000}   
 &  \phantom{0}0.08   \\
 \bottomrule
\end{tabularx}
}
\end{sc}
\end{small}
\end{center}
\vskip -0.1in
\end{table*}

\end{document}